\pdfoutput=1
\RequirePackage{fix-cm}
\documentclass[letterpaper]{article}
\usepackage{iclr2027_conference,times}
\usepackage[T1]{fontenc}
\usepackage[utf8]{inputenc}
\usepackage{amsmath,amssymb,amsthm}
\usepackage{graphicx,booktabs,multirow,array,longtable}
\usepackage{microtype}
\usepackage[table]{xcolor}
\usepackage{enumitem}
\usepackage{capt-of}
\usepackage{wrapfig}
\usepackage{needspace}
\usepackage{algorithm}
\usepackage{algpseudocode}
\usepackage{xurl}
\usepackage{hyperref}
\usepackage{etoc}
\makeatletter
\g@addto@macro\normalsize{%
  \setlength{\abovedisplayskip}{5pt plus 1pt minus 2pt}%
  \setlength{\belowdisplayskip}{5pt plus 1pt minus 2pt}%
  \setlength{\abovedisplayshortskip}{2pt plus 1pt}%
  \setlength{\belowdisplayshortskip}{3pt plus 1pt minus 1pt}}
\makeatother

\hypersetup{pdftitle={CoViST: Visual Token Compression via Composable States},
  pdfauthor={Qi Zhang, Xiandong Meng, Ronggang Wang, Siwei Ma},pdfsubject={Preprint},
  colorlinks=true,linkcolor=blue,citecolor=blue,urlcolor=blue}
\newcommand{\method}{CoViST}
\definecolor{covistrowcolor}{HTML}{F8E6DA}
\newcommand{\covistrow}{\rowcolor{covistrowcolor}}
\makeatletter
\newcommand{\covistrowfull}{%
  \rlap{\color{covistrowcolor}%
    \rule[-\dp\@arstrutbox]{\linewidth}%
         {\dimexpr\ht\@arstrutbox+\dp\@arstrutbox\relax}}%
}
\makeatother
\newlength{\budgetlabw}
\newcommand{\tsup}[1]{\ensuremath{^{\scalebox{0.8}{$\scriptstyle\mathrm{#1}$}}}}
\newcommand{\budgetlabel}[3][0pt]{\multirow{-#2}{*}[#1]{\rotatebox[origin=c]{90}{\itshape #3}}}
\newcommand{\budgetlabelii}[4][0pt]{\multirow{-#2}{*}[#1]{\rotatebox[origin=c]{90}{\itshape\thickmuskip=1mu\shortstack{#3\\#4}}}}
\makeatletter
\newcommand{\covistrowskip}{%
  \hspace*{\dimexpr\budgetlabw+3pt\relax}%
  \rlap{\color{covistrowcolor}%
    \rule[-\dp\@arstrutbox]{\dimexpr\linewidth-\budgetlabw-3pt\relax}%
         {\dimexpr\ht\@arstrutbox+\dp\@arstrutbox\relax}}%
  \hspace*{-\dimexpr\budgetlabw+3pt\relax}}
\makeatother
\newcommand{\fixed}{\method-Fixed}
\newcommand{\pro}{\method-Pro}
\newcommand{\methodvenue}[2]{#1~{\normalfont\fontsize{7}{7.5}\selectfont(#2)}}
\newcommand{\longmethodvenue}[2]{\begin{tabular}[c]{@{}l@{}}#1\\[-1pt]{\normalfont\fontsize{7}{7.5}\selectfont(#2)}\end{tabular}}
\newcommand{\R}{\mathbb{R}}
\newcommand{\E}{\mathbb{E}}
\newcommand{\norm}[1]{\left\lVert #1\right\rVert}

\newcommand{\clip}{\operatorname{clip}}
\newcommand{\softmax}{\operatorname{softmax}}

\newcommand{\argmax}{\operatorname*{arg\,max}}

\newcommand{\MassNorm}{\operatorname{MassNorm}}
\newtheorem{proposition}{Proposition}
\newtheorem{lemma}{Lemma}

\title{CoViST: Visual Token Compression\\via Composable States}
\iclrfinalcopy
\author{\begin{tabular}[t]{@{}l@{\hspace{0.9in}}l@{}}
Qi Zhang & Xiandong Meng\\
Peng Cheng Laboratory & Peng Cheng Laboratory\\
Shenzhen, China & Shenzhen, China\\
\texttt{qizhang@alumni.pku.edu.cn} & \texttt{mengxd@pcl.ac.cn}\\[8pt]
Ronggang Wang\footnotemark[1] & Siwei Ma\thanks{Corresponding authors: Siwei Ma and Ronggang Wang.}\\
Peking University Shenzhen Graduate School & Peking University\\
Peng Cheng Laboratory & Beijing, China\\
Shenzhen, China & \texttt{swma@pku.edu.cn}\\
\texttt{rgwang@pkusz.edu.cn} &
\end{tabular}}
\begin{document}
\maketitle
\lhead{Preprint. Under review.}
\etocdepthtag.toc{mainmatter}
\begin{abstract}
Visual token compression lowers the inference cost of vision--language models by representing images with fewer tokens. However, most existing methods compress visual tokens to a reduced set, leaving the amount of visual evidence represented by each token and its original spatial context implicit. Therefore, the compressed representation does not explicitly encode how much visual information each representative carries or where it lies in the original image. This limitation arises even after a single reduction and becomes more pronounced when compression is repeated across decoder layers. To address this issue, we propose CoViST, a training-free framework that represents a compressed image as a composable visual state. Specifically, the state combines representative features with original positions, effective contribution weights, and reusable selection metadata. CoViST constructs this state through coverage-guided selection and conservation-based contribution composition, and explicitly incorporates its contribution and positional information into decoder attention. Each component of the state retains its interpretation under successive reductions, enabling the same formulation to support both fixed compression before prefill and progressive compression within the decoder. Experimental results on seven LLaVA-1.5-7B benchmarks show that CoViST-Fixed retains 99.9\%, 99.5\%, and 98.1\% of uncompressed performance at 192, 128, and 64 tokens, respectively, and CoViST-Pro retains 99.8\%, 99.9\%, and 99.1\% at the corresponding layer-average budgets, outperforming state-of-the-art methods under their respective budget settings. Code will be released publicly.
\end{abstract}

\section{Introduction}
\label{sec:intro}

Vision--language models (VLMs) encode an image into many visual tokens, which every decoder layer processes and the key--value (KV) cache retains throughout generation \citep{llava2024,llavanext2024,qwen2025vl}, so the visual sequence accounts for a large share of the inference cost. Visual token compression, which represents an image with fewer tokens, is therefore widely studied. Existing training-free methods select or aggregate visual tokens according to encoder attention, feature diversity, instruction relevance, or coverage \citep{yang2025visionzip,alvar2025divprune,zheng2025cdpruner,mmtok2026}, and apply the reduction either once, before or within the decoder, or repeatedly at several decoder layers \citep{chen2024fastv,xing2025pyramiddrop,visiontrim2026}.

However, despite the diversity of their selection criteria, most existing methods only compress the visual tokens to a reduced set. Each retained token is read by the decoder as one patch at one position, and methods that merge or reweight tokens fix the resulting count or weight at the moment of compression \citep{bolya2023tome,cho2026restore}. Therefore, the amount of visual evidence represented by each retained token and its original spatial context are left implicit, and the decoder processes the compressed representation without accounting for how the visual information has been redistributed. This limitation is present after a single compression and becomes more pronounced under progressive compression, where a later stage operates on representations that already aggregate evidence from multiple original tokens without any record of what the earlier stages have established.

\begin{figure}[t]
  \centering
  \includegraphics[width=\linewidth]{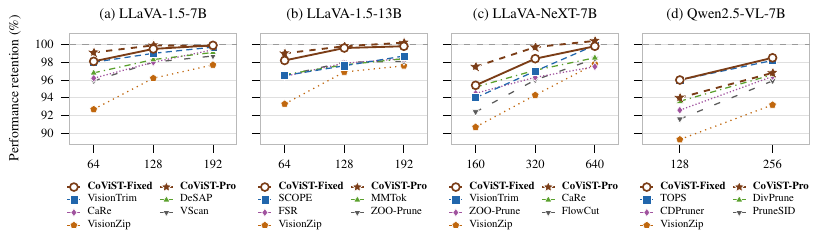}
  \caption{\method\ achieves state-of-the-art visual token compression performance on four VLMs.}
  \label{fig:retention}
\end{figure}

In this paper, we make this redistribution explicit by representing an image as a \emph{composable visual state}. Besides the representative features, the state retains their original positions, effective contribution weights, and reusable selection metadata. A key property of this representation is \emph{closure under reduction}: applying a further reduction to an existing state produces another state of the same form, with each component preserving its interpretation. Therefore, reductions can be applied repeatedly at different decoder depths. Moreover, the state is explicitly integrated into decoder attention, where the contribution weights and original positions modulate the attention logits through a bias derived from the frozen model. The same composable-state formulation can therefore support fixed compression before prefill and progressive compression within the decoder.

Based on this formulation, we propose \method\ (\textbf{Co}mposable \textbf{Vi}sual \textbf{S}tates for \textbf{T}oken compression), a training-free visual token compression framework consisting of three modules, as illustrated in Figure~\ref{fig:overview}. Specifically, \emph{Coverage-Guided State Initialization} (CSI) first selects representative tokens to cover visual and instruction-relevant evidence. Then, \emph{Coupled State Composition} (CSC) transfers the contribution of removed tokens to their representatives under a per-image conservation constraint, inherits the resulting weights across reductions, and applies a confidence-gated update to their features. Finally, \emph{State-Conditioned Attention} (SCA) incorporates the contribution weights and original positions into decoder attention. The same framework supports two execution strategies: \fixed\ compresses the visual state to exactly $K$ tokens before prefill, while \pro\ progressively compresses the state within the decoder. Our contributions can be summarized as follows:

\begin{itemize}[leftmargin=*,itemsep=1pt,topsep=2pt,parsep=0pt]
  \item We formulate visual token compression as the construction of \emph{composable visual states}, which carry representative content, original positions, effective contribution weights, and selection metadata across reductions and provide a unified basis for fixed and progressive compression.
  \item We propose \method, a training-free framework that constructs and maintains such states through coverage-guided initialization, conservation-based contribution transfer with cross-stage inheritance, and state-conditioned attention.
  \item Extensive experiments show that \fixed\ and \pro\ attain the highest retention among the compared methods on LLaVA-1.5-7B under their respective budget conventions, accelerate VLM inference, and generalize to LLaVA-1.5-13B, LLaVA-NeXT-7B, and Qwen2.5-VL-7B.
\end{itemize}

\begin{figure}[t]
  \centering
  \includegraphics[width=\linewidth]{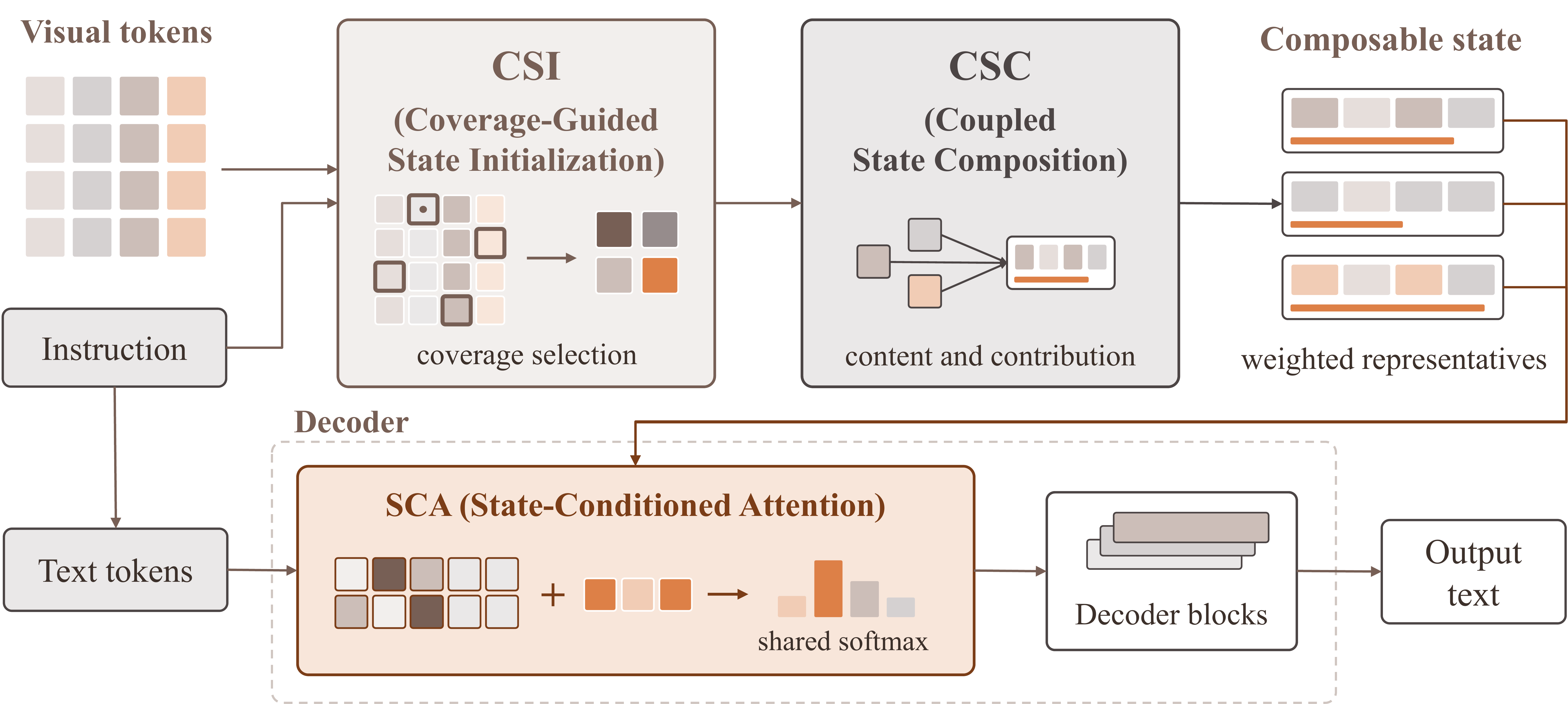}
  \caption{Overview of \method. The composable visual state is initialized by coverage-guided selection, reduced by a form-preserving composition rule, and incorporated into decoder attention through a state-conditioned bias. Fixed and progressive compression share this rule.}
  \label{fig:overview}
\end{figure}

\section{Related Work}
\label{sec:related}

\textbf{Visual token compression.} Existing training-free methods can be categorized by where the reduction is applied. Vision-side methods compress the tokens before the language model and measure the importance of a token by the attention of the [CLS] token \citep{yang2025visionzip,zhang2025vispruner}, by its similarity to the instruction \citep{wang2026eadp,oh2026anchorprune}, or by its contribution to the diversity or coverage of the retained set \citep{alvar2025divprune,zheng2025cdpruner,scope2025,mmtok2026}. LLM-side methods rely on the text-to-visual attention inside the decoder and discard tokens at one layer \citep{chen2024fastv} or progressively across several layers \citep{zhang2025sparsevlm,xing2025pyramiddrop}. Hybrid methods reduce the tokens at both locations \citep{visiontrim2026,zhang2026vscan}, and training-based methods learn compact visual representations \citep{cai2025matryoshka,zhang2025llavamini,hidrop2026}. However, these methods pass a reduced set of tokens to the decoder without recording how many original tokens each retained token represents, and each later stage reselects from the current features. \method\ also initializes its state with a coverage criterion of this family. Different from these methods, it records what each retained token represents and carries this record across reductions.

\textbf{Compensation for removed tokens.} A second line of work restores the information of the removed tokens in three forms. The first is the content: token merging folds the features of removed tokens into their representatives \citep{bolya2023tome,shang2025prumerge,yang2025visionzip}, refined by confidence-controlled correction \citep{li2026care} or statistical calibration of the averaged features \citep{zhou2026calibrated}. The second is the attention contribution: ToMe scales the attention logits by the number of merged tokens \citep{bolya2023tome}, ERA replaces this count by a saliency-weighted one \citep{wang2026era}, and Ada-KV formulates the loss of evicting a token as the change of the attention output \citep{feng2025adakv}. The third is the position: RESTORE retains the original positions and calibrates the attention with a distance factor \citep{cho2026restore}, and HiDrop keeps the original position identifiers \citep{hidrop2026}. However, the three quantities are treated separately, the weight attached to a representative counts members or saliency and is shared across all heads, and no weight is carried into a later reduction within the decoder. In contrast, \method\ treats the three quantities as components of one state, whose evidence-conserved contribution weights enter a per-head attention bias and are inherited across reductions (Appendix~\ref{app:design}).

\section{Method}
\label{sec:method}

\subsection{Composable visual states}
\label{sec:formulation}

Let $X\in\R^{N\times d}$ denote the projected visual embeddings of an image, which are processed by an $L$-layer decoder together with the instruction embeddings. At depth $\ell$, we represent the compressed visual component as a state
\begin{equation}
  S_\ell=(H_\ell,P_\ell,m_\ell,G_\ell),
  \qquad H_\ell\in\R^{K_\ell\times d},\quad m_\ell\in\R_{+}^{K_\ell},
  \label{eq:state}
\end{equation}
where $H_\ell$ contains the current representative features, $P_\ell$ their original positions, $m_\ell$ their effective contribution weights, and $G_\ell$ the selection metadata (the pairwise similarities and the surviving anchors). A reduction to $K<K_\ell$ representatives consists of a selection and a composition. The selection produces a representative set $A\subset\{1,\ldots,K_\ell\}$ with $|A|=K$ and a non-negative evidence vector $e\in\R_{\geq0}^{K_\ell}$ over the current representatives. The selection may consult signals outside the state, such as the encoder signals, the instruction, and the instruction states at the current depth. The composition then maps the state to $S'=\mathcal C_K(S_\ell;A,e)$. We call such a representation \emph{composable} if it satisfies two conditions. (i) \emph{Closure under reduction}: for every such $A$ and $e$, $\mathcal C_K(S_\ell;A,e)$ is a state with $K$ representatives of the same form, with each component preserving its interpretation, and, except for $A$ and $e$, the composition consults $S_\ell$ alone, so that no discarded token is retrieved and a further reduction applies to the result for any $K'<K$. (ii) \emph{Explicit incorporation into attention}: the decoder attention incorporates $m_\ell$ and $P_\ell$ through an explicit rule derived from the frozen model. 

\method\ constructs and uses such a representation in three steps. First, CSI initializes the state from uncompressed embeddings (Section~\ref{sec:selection}). Then, CSC performs the composition at every reduction (Section~\ref{sec:update}). Finally, SCA incorporates the contribution weights and original positions into every decoder layer (Section~\ref{sec:attention}). Under the progressive strategy, internal reductions at intermediate depths select fewer representatives from the current state, and CSC composes the state again (Section~\ref{sec:policies}).

\subsection{Coverage-Guided State Initialization}
\label{sec:selection}

\textbf{Affinity and evidence.} Let $c_i$ be a normalized key of the CLIP encoder \citep{radford2021clip} and $x_i$ a normalized projected embedding. We measure the affinity between two tokens in both spaces as
\begin{equation}
  s_{ij}=\tfrac12\clip(c_i^\top c_j,0,1)+\tfrac12\clip(x_i^\top x_j,0,1),
  \qquad \kappa_{ij}=s_{ij}^{\,2},
  \label{eq:similarity}
\end{equation}
so that a representative is shared only by tokens that are similar in both the encoder space and the decoder input space; the square $\kappa_{ij}$ serves as the coverage kernel. A normalized visual evidence distribution $b$ combines the denoised encoder attention with local contrast and global feature distinctiveness (Appendix~\ref{app:implementation}). Meanwhile, the frozen CLIP text encoder yields a normalized instruction distribution $\pi$ from entropy-filtered text features. The instruction enters the evidence solely through its positive residual over the visual evidence:
\begin{equation}
  r=\MassNorm([\pi-b]_+),\qquad
  \lambda_q=\lambda_{\max}\min(\sigma/\sigma_0,1),\qquad
  e=\MassNorm\big((1-\lambda_q)b+\lambda_qr\big),
  \label{eq:query-evidence}
\end{equation}
where $\MassNorm$ normalizes a non-negative vector to unit sum and $\sigma=\tfrac12\norm{\pi-\mathbf1/N}_1$ measures the selectivity of the instruction. The residual thus contributes only the evidence that the image alone does not supply.

\textbf{Selecting the representatives.} To allocate representatives across both concentrated and diffuse evidence, we mix the sharpened evidence with a uniform component,
\begin{equation}
  w_j=(1-u)\,\frac{e_j^\beta}{\sum_t e_t^\beta}+\frac{u}{N},
  \label{eq:coverage-weights}
\end{equation}
where $u$ decreases with the budget and adapts to the concentration of the encoder attention. We then balance the weights over $N_{\mathrm r}$ regions formed from features and coordinates, which yields the coverage weights $\omega_j$. Starting from the protected attention peaks and non-redundant instruction anchors, CSI greedily increases the covered evidence
\begin{equation}
  F(A)=\sum_{j=1}^{N}\omega_j\max_{i\in A}\kappa_{ij},
  \qquad |A|=K.
  \label{eq:coverage}
\end{equation}
Finally, the last fraction $\rho_{\mathrm r}$ of the budget is reserved for the tokens with the largest remaining coverage gap $1-\max_{i\in A}\kappa_{ij}$.

\subsection{Coupled State Composition}
\label{sec:update}

CSC composes the representative content and the effective contribution for a representative set $A$, the removed set $D$, and the evidence $e$ of the stage, with the same rule at the initial reduction and at every later one. The evidence is the distribution of Equation~\ref{eq:query-evidence} at the initial stage, the internal selection score at the first internal stage, and the uniform vector at the final stage, which subsamples the state spatially (Appendix~\ref{app:implementation}). Every removed token is assigned to its most similar representative,
\begin{equation}
  a(j)=\argmax_{i\in A}s_{ij},\qquad j\in D.
  \label{eq:correspondence}
\end{equation}

\textbf{Composing contributions.} Each representative accumulates the similarity-weighted contributions of its assigned tokens, and the accumulation is scaled such that the evidence-weighted contribution is conserved before the upper bound is applied:
\begin{equation}
  b_i=\sum_{\substack{j\in D\\a(j)=i}}m_j s_{ij}^{\,\gamma},
  \qquad
  \lambda=\frac{\sum_{j\in D}m_je_j}{\sum_{i\in A}e_i b_i},
  \qquad
  m'_i=\clip(m_i+\lambda b_i,1,m_{\max}).
  \label{eq:mass}
\end{equation}
The scale $\lambda$ enforces the conservation requirement $\sum_{i\in A}e_i\lambda b_i=\sum_{j\in D}m_je_j$ for each image. The inherited weights $m_j$ enter both the accumulation and the conserved quantity: the first reduction starts from unit weights, and a later reduction transfers the contribution that a removed representative had itself accumulated. The upper bound $m_{\max}$ prevents a large homogeneous region from concentrating attention on one key, and Appendix~\ref{app:proofs} gives an exact identity for the contribution withheld by this bound.\looseness=-1

\textbf{Composing content.} Following the confidence-controlled correction of CaRe \citep{li2026care}, we compute a confidence $q_j$ for each removed token, which combines the semantic and spatial affinity with the concentration of its affinities over candidate representatives (Appendix~\ref{app:implementation}). Only sufficiently confident tokens contribute:
\begin{equation}
  \bar h_i=
  \frac{\sum_{j\in D_i}\tilde e_jq_jh_j}
       {\sum_{j\in D_i}\tilde e_jq_j},
  \qquad
  h'_i=\frac{\norm{h_i}}{\norm{h_i+\alpha\bar h_i}}
          (h_i+\alpha\bar h_i),
  \qquad
  D_i=\{j\in D:a(j)=i,\ q_j\geq\tau\},
  \label{eq:content}
\end{equation}
where $\tilde e_j$ is the min--max-normalized evidence. The update rotates the representative toward a weighted centroid while preserving its norm. Since $h_i$ enters the query, key, and value projections and the residual stream, the content composition is deliberately conservative. In progressive execution, it acts on the current decoder-transformed features. The composed state $S'=(H'_A,P_A,m'_A,G_A)$ keeps the original positions and restricts the similarities and anchors to $A$ (Algorithm~\ref{alg:state} in Appendix~\ref{app:implementation}).

\subsection{State-Conditioned Attention}
\label{sec:attention}

\begin{figure}[t]
  \centering
  \includegraphics[width=0.9\linewidth]{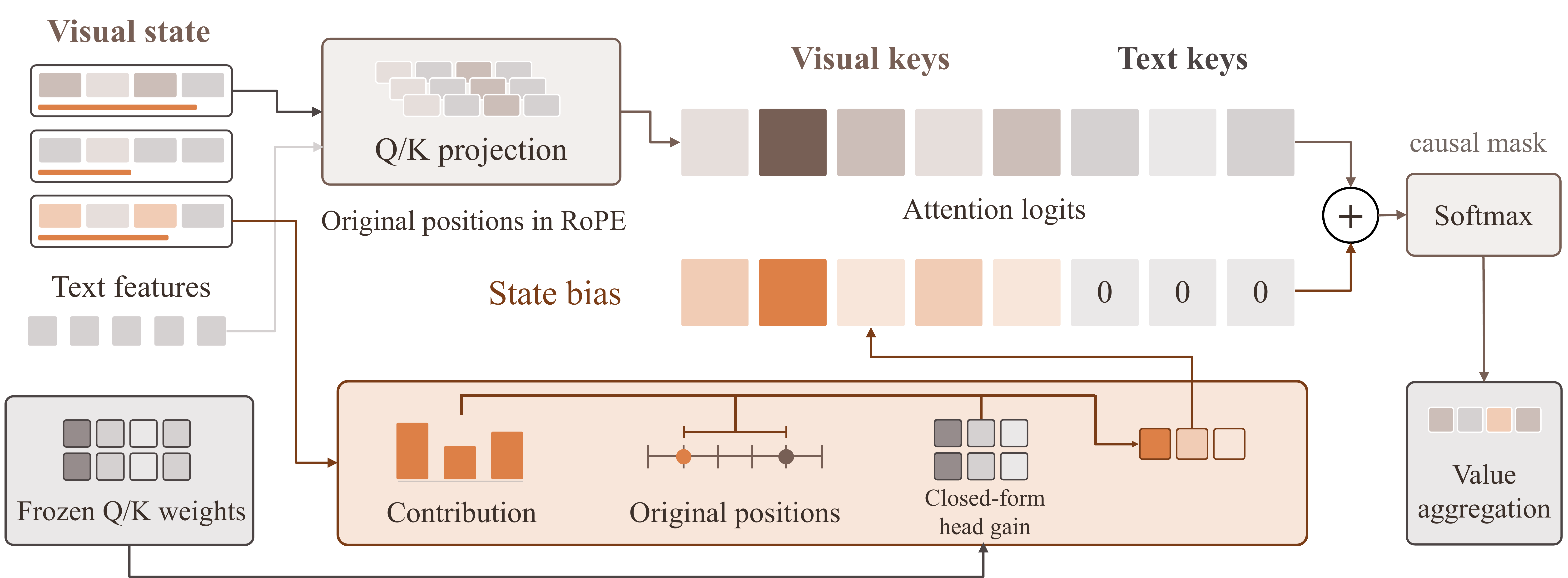}
  \caption{State-Conditioned Attention. The representatives enter the frozen query and key projections at their original rotary positions, while the effective contribution weights and the original positions form a separate bias with closed-form per-head gains, added to the visual logits before the softmax.}
  \label{fig:sca}
\end{figure}

SCA specifies how the decoder attention incorporates the state, which is illustrated in Figure~\ref{fig:sca}. The representatives keep their original indices in the rotary position embedding (RoPE) \citep{su2024roformer}. For a query $q$ and a visual representative $i$, the logit of head $h$ in layer $\ell$ receives the bias
\begin{equation}
  B_{\ell h}(q,i)
  =u_\ell\, g_{\ell h}\log\!\big[
       m_i\big(2-\mathcal D(|p_q-p_i|)\big)\big],
  \qquad
  \mathcal D(\Delta)=\frac{2}{d_h}
      \sum_{r=0}^{d_h/2-1}\cos(\Delta\theta_r),
  \label{eq:attention}
\end{equation}
where $\theta_r=10000^{-2r/d_h}$, $d_h$ is the head dimension, and $u_\ell$ is a strategy-dependent gain. The distance factor follows RESTORE \citep{cho2026restore} and the logarithmic weight follows proportional attention \citep{bolya2023tome}, with the weight being the evidence-conserved contribution produced by CSC, so the decoder attends to a representative in proportion to the evidence it currently represents.

\textbf{Head-dependent modulation.} The bias should be strongest for the heads that rely most on positional geometry. We measure this reliance directly on the frozen weights. For the RoPE frequency pair $r$, let $E_{\ell hr}$ be the product of the paired query- and key-row norms of head $h$ in layer $\ell$ (Appendix~\ref{app:implementation}). The gain is computed as
\begin{equation}
  \nu_{\ell h}=\frac{\sum_rE_{\ell hr}\theta_r}{\sum_rE_{\ell hr}},
  \qquad
  g_{\ell h}=g_{\min}+(g_{\max}-g_{\min})\frac{\nu_{\ell h}-\nu_{\min}}
                              {\nu_{\max}-\nu_{\min}},
  \label{eq:head-gain}
\end{equation}
with the extrema over all decoder layers and heads. The gains are a closed-form function of the frozen projections. Therefore, they require no annotation, calibration, or training, and transfer to another model by recomputation. Since the bias enters the common softmax over all causally visible keys, it modulates each representative relative to the other representatives and the text tokens (Appendix~\ref{app:implementation}).

\begin{table}[t]\centering
\caption{Performance comparison on LLaVA-1.5-7B. Bold and underline mark the best and second-best values in each benchmark.}\label{tab:external}
\fontsize{7.6}{8.5}\selectfont
\setlength{\tabcolsep}{2.6pt}\renewcommand{\arraystretch}{0.93}
\begin{tabular*}{\linewidth}{@{}w{l}{\budgetlabw}@{\hspace{3pt}}l@{\extracolsep{\fill}\hspace{2\tabcolsep}}ccccccccc@{}}
\toprule
 & Method & Budget & GQA & MMB & MME & POPE & SQA & VQA\tsup{T} & VQA\tsup{v2} & $R$ (\%)\\
\midrule
 & Vanilla, $K=576$ & -- & 62.0 & 64.0 & 1867 & 85.8 & 69.5 & 58.2 & 76.7 & 100.0\\
\midrule
 & \methodvenue{DivPrune}{CVPR 2025} & fixed & 58.9 & 63.1 & 1723 & 86.5 & 69.0 & 55.7 & 76.1 & 96.9\\
 & \methodvenue{PyramidDrop}{CVPR 2025} & avg. & 57.3 & 63.3 & 1797 & 84.8 & 69.2 & 56.5 & 76.4 & 97.1\\
 & \methodvenue{VisionZip}{CVPR 2025} & fixed & 59.3 & 63.0 & 1783 & 85.3 & 68.9 & 57.3 & 76.8 & 97.7\\
 & \methodvenue{MMTok}{ICLR 2026} & fixed & 60.1 & 63.4 & 1774 & 86.4 & 68.8 & 57.7 & 77.1 & 98.2\\
 & \methodvenue{ZOO-Prune}{CVPR 2026} & fixed & 60.0 & 62.9 & 1782 & \underline{87.2} & 69.2 & 57.3 & 77.3 & 98.3\\
 & \methodvenue{SCOPE}{NeurIPS 2025} & fixed & 60.1 & 63.6 & 1804 & 86.4 & 68.8 & 57.7 & -- & 98.5\\
 & \methodvenue{VScan}{TMLR 2026} & avg. & 60.6 & 63.9 & 1806 & 86.2 & 68.6 & 57.7 & \underline{77.8} & 98.7\\
 & \methodvenue{SpecFlow}{ICML 2026} & fixed & 58.3 & \textbf{65.8} & 1827 & 85.8 & 69.7 & 57.9 & 76.4 & 98.7\\
 & \methodvenue{RESTORE+HoloV}{ICML 2026} & fixed & 61.0 & 63.7 & 1793 & 86.6 & 69.6 & 57.2 & 77.6 & 98.8\\
 & \methodvenue{DeSAP}{ACM MM 2026} & fixed & 60.3 & 63.5 & 1848 & 87.1 & 69.9 & 57.2 & 77.5 & 99.1\\
 & \methodvenue{CaRe}{arXiv 2026} & fixed & 60.6 & 63.4 & 1809 & \textbf{88.9} & \underline{70.2} & \underline{58.0} & 77.4 & 99.4\\
 & \methodvenue{VisionTrim}{ICLR 2026} & avg.\,198 & 61.0 & 64.4 & 1798 & 86.8 & \textbf{70.8} & \textbf{58.4} & \textbf{78.4} & 99.7\\
\covistrowskip & \textbf{\pro} & avg. & \textbf{61.6} & \underline{64.5} & \underline{1855} & 86.2 & 69.3 & 57.9 & 76.5 & \underline{99.8}\\
\covistrowskip\budgetlabel{14}{$K{=}192$ ($\downarrow 66.7\%$)} & \textbf{\fixed} & fixed & \underline{61.3} & 64.3 & \textbf{1901} & 86.2 & 69.1 & 57.6 & 76.4 & \textbf{99.9}\\
\midrule
 & \methodvenue{PyramidDrop}{CVPR 2025} & avg. & 57.1 & 61.6 & 1761 & 82.6 & 68.4 & 56.6 & 76.0 & 95.8\\
 & \methodvenue{VisionZip}{CVPR 2025} & fixed & 57.6 & 62.0 & 1762 & 83.2 & 68.9 & 56.8 & 75.6 & 96.2\\
 & \methodvenue{DivPrune}{CVPR 2025} & fixed & 58.6 & 63.7 & 1702 & 86.5 & 68.9 & 55.2 & 75.2 & 96.5\\
 & \methodvenue{SpecFlow}{ICML 2026} & fixed & 57.6 & \underline{64.3} & 1794 & 84.9 & \textbf{69.9} & 56.8 & 75.3 & 97.4\\
 & \methodvenue{MMTok}{ICLR 2026} & fixed & 59.3 & 62.3 & 1779 & 86.3 & 68.8 & 57.0 & 76.4 & 97.5\\
 & \methodvenue{SCOPE}{NeurIPS 2025} & fixed & 59.7 & 62.5 & 1776 & 86.1 & 68.4 & 57.2 & -- & 97.6\\
 & \methodvenue{ZOO-Prune}{CVPR 2026} & fixed & 59.5 & 61.9 & 1752 & \underline{87.1} & 68.9 & 57.9 & 76.6 & 97.6\\
 & \methodvenue{CaRe}{arXiv 2026} & fixed & 60.2 & 62.1 & 1731 & \textbf{87.8} & \textbf{69.9} & \underline{58.0} & 76.7 & 98.0\\
 & \methodvenue{VScan}{TMLR 2026} & avg. & 59.8 & 63.0 & 1792 & 86.1 & 68.9 & 57.3 & \underline{77.1} & 98.0\\
 & \methodvenue{RESTORE+HoloV}{ICML 2026} & fixed & \underline{60.8} & 63.0 & 1807 & 86.0 & 68.7 & 56.6 & \underline{77.1} & 98.2\\
 & \methodvenue{DeSAP}{ACM MM 2026} & fixed & 59.5 & 62.8 & 1815 & 87.0 & \textbf{69.9} & 57.0 & 76.9 & 98.3\\
 & \methodvenue{VisionTrim}{ICLR 2026} & avg.\,138 & 60.3 & 64.0 & 1788 & 86.6 & \underline{69.7} & \textbf{58.2} & \textbf{78.2} & 99.0\\
\covistrowskip & \textbf{\fixed} & fixed & \textbf{61.3} & 64.1 & \underline{1867} & 86.3 & 69.0 & 57.1 & 76.2 & \underline{99.5}\\
\covistrowskip\budgetlabel{14}{$K{=}128$ ($\downarrow 77.8\%$)} & \textbf{\pro} & avg. & \textbf{61.3} & \textbf{64.7} & \textbf{1890} & 86.1 & 69.1 & 57.7 & 76.3 & \textbf{99.9}\\
\midrule
 & \methodvenue{PyramidDrop}{CVPR 2025} & avg. & 47.5 & 58.8 & 1561 & 76.2 & 69.0 & 50.6 & 73.3 & 88.5\\
 & \methodvenue{VisionZip}{CVPR 2025} & fixed & 55.1 & 60.1 & 1690 & 77.0 & 69.0 & 55.5 & 72.4 & 92.7\\
 & \methodvenue{DivPrune}{CVPR 2025} & fixed & 57.1 & 60.2 & 1653 & 85.3 & 68.3 & 54.5 & 73.3 & 94.1\\
 & \methodvenue{SpecFlow}{ICML 2026} & fixed & 55.3 & \underline{63.7} & 1713 & 80.5 & 69.7 & 54.9 & 73.7 & 94.6\\
 & \methodvenue{ZOO-Prune}{CVPR 2026} & fixed & 58.5 & 60.2 & 1676 & 85.9 & 68.3 & 55.4 & 75.0 & 95.2\\
 & \methodvenue{SCOPE}{NeurIPS 2025} & fixed & 58.3 & 61.7 & 1698 & 83.9 & 68.6 & 56.6 & -- & 95.7\\
 & \methodvenue{VScan}{TMLR 2026} & avg. & 58.3 & 62.1 & 1698 & 85.0 & 69.1 & 55.6 & 75.4 & 95.9\\
 & \methodvenue{MMTok}{ICLR 2026} & fixed & 58.3 & 61.2 & 1715 & 85.8 & 69.2 & 56.0 & 75.2 & 96.0\\
 & \methodvenue{CaRe}{arXiv 2026} & fixed & 59.1 & 60.6 & 1665 & \textbf{87.0} & \underline{69.8} & 56.3 & 75.8 & 96.2\\
 & \methodvenue{RESTORE+HoloV}{ICML 2026} & fixed & 59.0 & 61.9 & 1787 & 84.9 & 68.0 & 55.4 & 75.6 & 96.5\\
 & \methodvenue{DeSAP}{ACM MM 2026} & fixed & 58.3 & 62.3 & 1748 & 85.8 & 69.6 & 55.7 & \underline{76.7} & 96.8\\
 & \methodvenue{VisionTrim}{ICLR 2026} & avg.\,78 & 58.8 & 63.0 & 1780 & \underline{86.2} & \textbf{71.0} & \underline{56.8} & \textbf{76.8} & 98.0\\
\covistrowskip & \textbf{\fixed} & fixed & \underline{60.2} & 63.1 & \underline{1819} & 85.8 & 68.5 & 56.4 & 75.0 & \underline{98.1}\\
\covistrowskip\budgetlabel{14}{$K{=}64$ ($\downarrow 88.9\%$)} & \textbf{\pro} & avg. & \textbf{60.9} & \textbf{64.4} & \textbf{1863} & 86.0 & 68.4 & \textbf{57.0} & 75.6 & \textbf{99.1}\\
\bottomrule\end{tabular*}\end{table}

\subsection{Fixed and progressive execution}
\label{sec:policies}

\fixed\ applies CSI and CSC once before the first decoder layer and keeps $K$ representatives throughout. \pro\ initializes a larger state and reduces it again at two intermediate layers during prefill. Specifically, at the first of them, an internal selection scores the current representatives by the visual attention of the last instruction token and by their similarity to the instruction in the current value space, with protected anchors and spatial quotas (Appendix~\ref{app:implementation}). Then, CSC composes the state with this score as the evidence. Note that the selection uses unit weights. Therefore, a representative competes for retention on its relevance alone and, once retained, accounts for its tokens through the recorded weights. Both strategies apply SCA at every layer, with $u_\ell=1$ for \fixed\ and, for \pro, $u_\ell=0.8$ before the first internal reduction and $1$ thereafter.

\begin{table}[!t]
\begin{minipage}[t]{0.465\linewidth}\vspace{0pt}\centering
\caption{Comparison on LLaVA-1.5-13B.}\label{tab:cross-13b}
\fontsize{6.5}{7.5}\selectfont\setlength{\tabcolsep}{0.6pt}\renewcommand{\arraystretch}{0.9}\setlength{\aboverulesep}{1pt}\setlength{\belowrulesep}{1.6pt}
\begin{tabular*}{\linewidth}{@{}l@{\extracolsep{\fill}}cccccccc@{}}
\toprule
Method & GQA & MMB & MME & POPE & SQA & VQA\tsup{T} & VQA\tsup{v2} & $R$ (\%)\\\midrule
Vanilla & 63.3 & 68.7 & 1824 & 86.0 & 72.8 & 61.3 & 78.3 & 100.0\\\midrule
\multicolumn{9}{c}{\textit{$K{=}192$ ($\downarrow 66.7\%$)}}\\\midrule
VisionZip & 59.1 & 66.9 & 1754 & 85.1 & 73.5 & 59.5 & 78.1 & 97.6\\
DivPrune & 59.4 & 66.6 & 1782 & \underline{86.8} & 72.9 & 58.5 & 78.0 & 97.8\\
CDPruner & 60.4 & 67.2 & 1776 & 86.6 & 72.4 & 58.7 & 78.4 & 97.8\\
ZOO-Prune & 60.0 & 66.7 & 1762 & 86.7 & 73.1 & 59.1 & \textbf{78.7} & 98.1\\
FSR & 60.2 & 67.3 & 1805 & 86.4 & 73.3 & 59.5 & \underline{78.6} & 98.4\\
MMTok & 59.7 & 67.7 & 1784 & 86.2 & \underline{73.6} & 59.6 & 78.3 & 98.4\\
SCOPE & 59.7 & 67.6 & 1775 & 86.7 & \textbf{73.8} & 60.0 & -- & 98.7\\
\covistrowfull \textbf{\fixed} & \underline{62.9} & \underline{68.5} & \underline{1821} & \textbf{87.0} & 73.1 & \underline{60.6} & 77.9 & \underline{99.8}\\
\covistrowfull \textbf{\pro} & \textbf{63.3} & \textbf{68.6} & \textbf{1835} & \underline{86.8} & 73.2 & \textbf{60.9} & 78.1 & \textbf{100.2}\\
\midrule
\multicolumn{9}{c}{\textit{$K{=}128$ ($\downarrow 77.8\%$)}}\\\midrule
VisionZip & 57.9 & 66.7 & 1743 & 85.2 & \textbf{74.0} & 58.7 & 76.8 & 96.9\\
DivPrune & 58.9 & 66.1 & 1749 & 86.5 & 72.8 & 58.2 & 77.1 & 97.0\\
PruneSID & 58.9 & 65.5 & \underline{1811} & 85.9 & 73.1 & 57.5 & 76.7 & 97.1\\
SCOPE & 59.3 & 67.2 & 1735 & 85.9 & \underline{73.9} & 58.7 & -- & 97.6\\
CDPruner & 59.7 & 67.5 & 1778 & \textbf{87.3} & 72.5 & 58.4 & 77.7 & 97.7\\
MMTok & 59.0 & 67.2 & 1756 & 86.2 & 73.4 & 59.2 & 77.6 & 97.7\\
FSR & 59.6 & \underline{68.2} & 1768 & 86.3 & 73.8 & 58.8 & \textbf{78.0} & 97.9\\
ZOO-Prune & 58.9 & 67.0 & 1791 & 87.0 & 73.4 & 58.8 & 77.8 & 98.0\\
\covistrowfull \textbf{\fixed} & \underline{62.7} & 68.0 & \underline{1811} & \underline{87.1} & 73.7 & \underline{60.1} & 77.6 & \underline{99.6}\\
\covistrowfull \textbf{\pro} & \textbf{62.9} & \textbf{68.6} & \textbf{1822} & 86.9 & 72.9 & \textbf{60.6} & \underline{77.9} & \textbf{99.8}\\
\midrule
\multicolumn{9}{c}{\textit{$K{=}64$ ($\downarrow 88.9\%$)}}\\\midrule
VisionZip & 56.2 & 64.9 & 1676 & 76.0 & \textbf{74.4} & 57.4 & 73.7 & 93.3\\
PruneSID & 57.8 & 63.8 & 1711 & 82.0 & 71.8 & 56.3 & 75.2 & 94.2\\
DivPrune & 57.7 & 64.6 & \underline{1778} & 84.8 & 71.3 & 57.1 & 75.2 & 95.5\\
FSR & 58.6 & 66.3 & 1750 & 85.0 & 73.0 & 58.1 & \underline{76.8} & 96.4\\
SCOPE & 58.7 & 65.5 & 1762 & 83.0 & 73.2 & 58.3 & -- & 96.5\\
CDPruner & 59.4 & 65.5 & 1744 & \textbf{87.2} & 72.4 & 57.6 & 76.7 & 96.5\\
ZOO-Prune & 58.6 & 64.8 & \textbf{1780} & 85.3 & 72.1 & 58.6 & 76.4 & 96.5\\
MMTok & 58.4 & 65.7 & 1763 & 84.4 & 73.0 & 58.4 & 76.6 & 96.6\\
\covistrowfull \textbf{\fixed} & \underline{61.4} & \underline{67.1} & 1775 & \underline{86.3} & 73.0 & \underline{59.5} & 76.5 & \underline{98.2}\\
\covistrowfull \textbf{\pro} & \textbf{62.1} & \textbf{68.0} & 1775 & \textbf{87.2} & \underline{73.7} & \textbf{59.9} & \textbf{77.1} & \textbf{99.0}\\
\bottomrule
\end{tabular*}
\end{minipage}
\hfill
\begin{minipage}[t]{0.525\linewidth}\vspace{0pt}\centering
\caption{Comparison on Qwen2.5-VL-7B.}\label{tab:cross-qwen}
\fontsize{6.5}{7.5}\selectfont\setlength{\tabcolsep}{0.6pt}\renewcommand{\arraystretch}{0.9}\setlength{\aboverulesep}{1pt}\setlength{\belowrulesep}{1.6pt}
\begin{tabular*}{\linewidth}{@{}l@{\extracolsep{\fill}}ccccccccc@{}}
\toprule
Method & GQA & MMB & MME & POPE & SQA & VQA\tsup{T} & VQA\tsup{v2} & MMB\tsup{CN} & $R$ (\%)\\\midrule
Vanilla & 59.8 & 83.4 & 2320 & 86.6 & 88.1 & 76.7 & 81.6 & 80.1 & 100.0\\\midrule
\multicolumn{10}{c}{\textit{$K{=}512$ ($\downarrow 60.5\%$)}}\\\midrule
VisionZip & 56.6 & 78.9 & \textbf{2317} & 85.8 & 80.5 & -- & \underline{80.7} & -- & 95.8\\
PruneSID & \textbf{59.8} & 80.9 & 2218 & \underline{85.9} & \textbf{87.6} & -- & 80.4 & -- & 97.6\\
\covistrowfull \textbf{\pro} & 58.0 & \underline{82.1} & \underline{2315} & \textbf{86.4} & \underline{86.2} & \underline{73.1} & 80.1 & \underline{79.6} & \underline{98.2}\\
\covistrowfull \textbf{\fixed} & \underline{59.2} & \textbf{83.0} & 2308 & 85.5 & \textbf{87.6} & \textbf{76.7} & \textbf{80.9} & \textbf{80.2} & \textbf{99.4}\\
\midrule
\multicolumn{10}{c}{\textit{$K{=}256$ ($\downarrow 80.2\%$)}}\\\midrule
ZOO-Prune & 55.9 & 77.4 & 2139 & 81.7 & 82.4 & 68.8 & -- & -- & 92.4\\
VisionZip & 54.6 & 76.8 & 2224 & 83.4 & 80.4 & -- & 78.5 & -- & 93.2\\
CaRe & 58.1 & 77.9 & 2274 & 83.9 & 84.4 & \underline{70.9} & -- & -- & 95.4\\
PruneSID & \textbf{59.0} & 78.0 & 2169 & \underline{85.6} & \underline{86.9} & -- & \underline{78.7} & -- & 95.9\\
CDPruner & -- & 81.5 & 2232 & 83.5 & 84.1 & -- & -- & 80.1 & 96.4\\
DivPrune & -- & \underline{81.8} & 2167 & 85.3 & 84.8 & -- & -- & \textbf{80.9} & 96.6\\
\covistrowfull \textbf{\pro} & 57.0 & \textbf{82.2} & \textbf{2315} & 85.5 & 84.7 & 69.7 & 78.1 & 79.3 & 96.8\\
TOPS$^{\dagger}$ & -- & 81.4 & 2285 & \textbf{86.1} & \textbf{87.2} & -- & -- & \underline{80.5} & \underline{98.2}\\
\covistrowfull \textbf{\fixed} & \underline{58.9} & \textbf{82.2} & \underline{2308} & 84.6 & 86.3 & \textbf{75.7} & \textbf{79.9} & 79.0 & \textbf{98.5}\\
\midrule
\multicolumn{10}{c}{\textit{$K{=}128$ ($\downarrow 90.1\%$)}}\\\midrule
ZOO-Prune & 53.3 & 72.5 & 1962 & 78.5 & 79.9 & 62.0 & -- & -- & 86.9\\
CaRe & 54.5 & 73.4 & 2017 & 79.2 & 81.5 & 65.3 & -- & -- & 88.9\\
VisionZip & 53.2 & 75.8 & 2025 & 78.9 & 80.1 & -- & 73.8 & -- & 89.3\\
PruneSID & \underline{55.8} & 73.9 & 2076 & 80.2 & \textbf{86.5} & -- & 74.6 & -- & 91.6\\
CDPruner & -- & 78.6 & 2033 & 80.5 & 82.5 & -- & -- & \underline{78.7} & 92.6\\
DivPrune & -- & 79.2 & 2044 & \underline{83.9} & 82.8 & -- & -- & 78.6 & 93.6\\
\covistrowfull \textbf{\pro} & 55.4 & \textbf{81.0} & \underline{2231} & \textbf{84.9} & 83.2 & \underline{65.5} & \underline{74.9} & 77.0 & \underline{94.0}\\
TOPS$^{\dagger}$ & -- & 80.2 & 2217 & \underline{83.9} & \underline{85.3} & -- & -- & \textbf{78.8} & \textbf{96.0}\\
\covistrowfull \textbf{\fixed} & \textbf{57.9} & \underline{80.3} & \textbf{2287} & 82.2 & 84.1 & \textbf{71.6} & \textbf{78.1} & 77.6 & \textbf{96.0}\\
\bottomrule
\end{tabular*}
\par\vspace{1pt}{\fontsize{6.0}{7.0}\selectfont\raggedright \pro\ uses layer-average budgets in both tables, and TOPS also prunes within the decoder. VisionZip and PruneSID keep 33.3\%, 22.2\%, and 11.1\% of the tokens, and ZOO-Prune and CaRe keep 20\% and 10\%. $^{\dagger}$Layer average cannot be derived.\par}
\end{minipage}
\end{table}

\section{Experiments}
\label{sec:experiments}

\subsection{Experimental setup}
\label{sec:setup}

\textbf{1) Models and benchmarks:} We adopt LLaVA-1.5-7B \citep{llava2024} as the primary model, and further evaluate \method\ on LLaVA-1.5-13B, LLaVA-NeXT-7B \citep{llavanext2024}, and Qwen2.5-VL-7B \citep{qwen2025vl}. We conduct the evaluation on seven benchmarks, \textit{i.e.,} GQA \citep{hudson2019gqa}, MMBench (MMB) \citep{liu2024mmbench}, MME \citep{fu2023mme}, POPE \citep{li2023pope}, ScienceQA (SQA) \citep{lu2022scienceqa}, TextVQA (VQA$^{\mathrm{T}}$) \citep{singh2019textvqa}, and VQAv2 (VQA$^{\mathrm{v2}}$) \citep{goyal2017vqav2}. For Qwen2.5-VL-7B, we additionally use MMBench-CN (MMB$^{\mathrm{CN}}$).

\textbf{2) Evaluation protocol:} We report the performance retention $R=\frac{100}{B}\sum_{b=1}^{B}s_b/s_b^{\mathrm{full}}$, \textit{i.e.,} the average ratio of each task score $s_b$ to its uncompressed reference $s_b^{\mathrm{full}}$ over the $B$ tasks.  Note that when a compared method does not report a task, its retention is averaged over the tasks it does report. In the Budget column, \emph{fixed} denotes $K$ visual tokens at every decoder layer, and \emph{avg.}\ denotes a layer schedule within the decoder whose layer-average count is $K$. If the budget of a method is defined differently, we mark it with the layer-average count it actually uses, or with $^{\dagger}$ when this count cannot be derived. Results of compared methods are taken from published literature under their own budgets, protocols, and full-model references.

\subsection{Main results on LLaVA-1.5-7B}
\label{sec:results-7b}

Table~\ref{tab:external} compares both execution strategies with state-of-the-art methods on LLaVA-1.5-7B. It can be observed that the two strategies rank first and second at every budget. Specifically, \fixed\ retains $99.9\%$, $99.5\%$, and $98.1\%$ at 192, 128, and 64 tokens, respectively, and \pro\ retains $99.8\%$, $99.9\%$, and $99.1\%$ at the same budgets taken as layer averages. Moreover, \method\ can even exceed the vanilla model on MMBench, POPE, and MME after compression. Meanwhile, the two strategies differ in where their advantage lies. \fixed\ ranks first at 192 tokens, whereas the advantage of \pro\ over \fixed\ grows as the budget shrinks and reaches $0.4$ and $1.0$ points at 128 and 64 tokens, respectively. Particularly at 64 tokens, \pro\ attains the highest GQA, MMBench, MME, and TextVQA scores. The reason is that the later reductions of \pro\ act on decoder-transformed features while inheriting the weights and positions established before them (Table~\ref{tab:components}).

\begin{table}[!t]\centering\caption{Performance comparison on LLaVA-NeXT-7B. $K$ counts content tokens, and \method\ additionally keeps the 48 newline tokens. $^{\dagger}$Layer average cannot be derived.}\label{tab:cross-next}
\fontsize{7.9}{8.8}\selectfont\setlength{\tabcolsep}{2.6pt}\renewcommand{\arraystretch}{0.93}
\begin{tabular*}{\linewidth}{@{}w{l}{\budgetlabw}@{\hspace{3pt}}l@{\extracolsep{\fill}\hspace{2\tabcolsep}}ccccccccc@{}}
\toprule
 & Method & Budget & GQA & MMB & MME & POPE & SQA & VQA\tsup{T} & VQA\tsup{v2} & $R$ (\%)\\
\midrule
 & Vanilla & -- & 64.2 & 64.3 & 1817 & 86.9 & 68.0 & 61.1 & 79.9 & 100.0\\
\midrule
 & \methodvenue{DivPrune}{CVPR 2025} & fixed & 61.6 & 65.4 & 1773 & 85.5 & 67.8 & 55.4 & 78.9 & 96.1\\
 & \methodvenue{PyramidDrop}{CVPR 2025} & avg. & 62.9 & 66.5 & 1733 & 86.4 & \underline{69.4} & 58.3 & -- & 97.3\\
 & \methodvenue{ApET}{CVPR 2026} & avg.$^{\dagger}$ & 63.0 & 65.3 & 1815 & 87.2 & -- & 57.9 & 79.2 & 97.5\\
 & \methodvenue{ZOO-Prune}{CVPR 2026} & fixed & 62.2 & 65.2 & 1816 & 86.8 & 68.0 & 58.0 & \underline{79.6} & 97.5\\
 & \methodvenue{VisionZip}{CVPR 2025} & fixed & 61.3 & 66.3 & 1787 & 87.7 & 68.1 & 60.2 & -- & 97.8\\
 & \methodvenue{FlowCut}{NeurIPS 2025} & avg.\,680 & 61.9 & \underline{66.7} & 1837 & 86.1 & -- & \underline{60.6} & \textbf{79.8} & 98.4\\
 & \methodvenue{CaRe}{arXiv 2026} & fixed & \textbf{63.9} & 66.0 & 1831 & \underline{88.2} & \underline{69.4} & 58.4 & 78.0 & 98.5\\
\covistrowskip & \textbf{\fixed} & fixed & \underline{63.5} & 63.6 & \underline{1865} & 87.4 & 67.9 & 60.3 & 79.2 & 99.8\\
 & \methodvenue{VisionTrim}{ICLR 2026} & avg.\,690 & 63.2 & \textbf{67.2} & 1825 & \textbf{88.5} & \textbf{70.7} & \textbf{61.0} & -- & \underline{99.9}\\
\covistrowskip\budgetlabel{10}{$K{=}640$ ($\downarrow 77.8\%$)} & \textbf{\pro} & avg. & \underline{63.5} & 64.2 & \textbf{1881} & 87.4 & 68.1 & \textbf{61.0} & 79.5 & \textbf{100.4}\\
\midrule
 & \methodvenue{PyramidDrop}{CVPR 2025} & avg. & 58.5 & 63.2 & 1667 & 81.9 & 66.8 & 58.2 & -- & 93.3\\
 & \methodvenue{VisionZip}{CVPR 2025} & fixed & 59.3 & 63.1 & 1702 & 82.1 & 67.3 & 58.9 & 76.2 & 94.3\\
 & \methodvenue{ApET}{CVPR 2026} & avg.$^{\dagger}$ & 61.0 & 63.5 & 1783 & 85.6 & -- & 54.4 & 75.8 & 94.3\\
 & \methodvenue{DivPrune}{CVPR 2025} & fixed & 61.1 & \underline{65.1} & 1721 & 84.7 & 67.7 & 56.2 & 77.2 & 95.0\\
 & \methodvenue{FlowCut}{NeurIPS 2025} & avg.\,340 & 59.8 & \textbf{65.3} & 1791 & 83.4 & -- & \textbf{60.1} & 77.8 & 96.0\\
 & \methodvenue{ZOO-Prune}{CVPR 2026} & fixed & 61.0 & 64.9 & 1788 & 85.5 & 67.8 & 57.3 & \underline{78.1} & 96.3\\
 & \methodvenue{VisionTrim}{ICLR 2026} & avg.\,390 & 61.7 & 64.8 & 1795 & 83.6 & \textbf{69.6} & 59.6 & -- & 97.0\\
 & \methodvenue{CaRe}{arXiv 2026} & fixed & 62.0 & 64.9 & 1792 & \textbf{87.6} & \underline{68.9} & 57.7 & 77.8 & 97.1\\
\covistrowskip & \textbf{\fixed} & fixed & \underline{62.6} & 62.5 & \underline{1812} & 87.3 & 67.8 & 58.6 & \underline{78.1} & \underline{98.4}\\
\covistrowskip\budgetlabel{10}{$K{=}320$ ($\downarrow 88.9\%$)} & \textbf{\pro} & avg. & \textbf{63.2} & 63.7 & \textbf{1885} & \underline{87.4} & 67.8 & \underline{59.7} & \textbf{78.9} & \textbf{99.7}\\
\midrule
 & \methodvenue{PyramidDrop}{CVPR 2025} & avg. & 56.1 & 60.3 & 1545 & 78.0 & 67.4 & 54.7 & -- & 89.3\\
 & \methodvenue{ApET}{CVPR 2026} & avg.$^{\dagger}$ & 58.4 & 60.8 & 1680 & 82.6 & -- & 53.8 & 72.7 & 90.7\\
 & \methodvenue{VisionZip}{CVPR 2025} & fixed & 55.5 & 60.1 & 1630 & 79.4 & 68.3 & 56.2 & -- & 90.7\\
 & \methodvenue{DivPrune}{CVPR 2025} & fixed & 59.3 & 63.2 & 1614 & 80.0 & 67.2 & 54.1 & 75.0 & 91.7\\
 & \methodvenue{FlowCut}{NeurIPS 2025} & avg.\,170 & 57.6 & 62.8 & \underline{1746} & 79.9 & -- & 57.6 & 74.6 & 92.4\\
 & \methodvenue{VisionTrim}{ICLR 2026} & avg.\,240 & 57.2 & 63.3 & 1702 & 81.1 & \textbf{70.2} & \textbf{58.3} & -- & 94.0\\
 & \methodvenue{ZOO-Prune}{CVPR 2026} & fixed & 59.9 & \underline{64.2} & 1739 & 83.1 & 68.4 & 55.4 & \underline{76.1} & 94.5\\
 & \methodvenue{CaRe}{arXiv 2026} & fixed & 60.0 & \textbf{64.5} & 1707 & \underline{86.8} & \underline{69.7} & 55.8 & 75.8 & 95.2\\
\covistrowskip & \textbf{\fixed} & fixed & \underline{61.1} & 59.5 & 1734 & 86.7 & 66.7 & 55.8 & 76.0 & \underline{95.4}\\
\covistrowskip\budgetlabel{10}{$K{=}160$ ($\downarrow 94.4\%$)} & \textbf{\pro} & avg. & \textbf{62.0} & 60.4 & \textbf{1809} & \textbf{87.2} & 68.1 & \underline{58.0} & \textbf{77.1} & \textbf{97.5}\\
\bottomrule\end{tabular*}\end{table}

\subsection{Generalization to other models}
\label{sec:transfer}

To verify the generalizability of \method, we further apply it to three other models. Only the evidence extraction and the positional interface follow each architecture (Appendix~\ref{app:cross-model-implementation}), and the head gains are recomputed from the frozen projections. Meanwhile, the two reductions of \pro\ keep their relative depths (Appendix~\ref{app:schedule}).

\textbf{1) Larger decoder:} The results on LLaVA-1.5-13B are shown in Table~\ref{tab:cross-13b}, where \pro\ and \fixed\ rank first and second at every budget. Specifically, \pro\ retains $100.2\%$, $99.8\%$, and $99.0\%$ at 192, 128, and 64 tokens, respectively, and even exceeds the vanilla model at 192 tokens. It also attains the highest GQA, MMBench, and TextVQA scores at every budget. \fixed\ retains $99.8\%$, $99.6\%$, and $98.2\%$, which are at least $1.1$ points higher than those of the strongest compared method.

\textbf{2) Non-CLIP encoder:} Qwen2.5-VL-7B adopts a native visual merger and multidimensional rotary positions. At $1008\times1008$ resolution, we reduce its 1,296 tokens to 512, 256, and 128. The results in Table~\ref{tab:cross-qwen} show that \fixed\ ranks first at every budget with $99.4\%$, $98.5\%$, and $96.0\%$, respectively, and attains the highest TextVQA and VQAv2 scores throughout. Moreover, \pro\ outperforms all compared methods except TOPS, whose retention is averaged from only five of the eight tasks.

\textbf{3) Longer visual sequences:} LLaVA-NeXT-7B encodes a $672\times672$ input into 2,880 content tokens, which we compress to 640, 320, and 160. According to the results in Table~\ref{tab:cross-next}, \pro\ ranks first at all three budgets with $100.4\%$, $99.7\%$, and $97.5\%$, which exceed the strongest compared method by $0.5$, $2.6$, and $2.3$ points, respectively. At 640 tokens, \pro\ even outperforms the vanilla model. Meanwhile, \fixed\ ranks second at 320 and 160 tokens. Compared with LLaVA-1.5-7B, the advantage of \pro\ over \fixed\ is larger and reaches $2.1$ points at 160 tokens. The reason is that the longer visual sequence leaves more redundancy for the later reductions to remove.

\subsection{Inference efficiency}
\label{sec:efficiency}

We measure the efficiency of LLaVA-NeXT-7B on POPE with natural EOS stopping, batch size one, FP16, and scaled dot-product attention on one NVIDIA A800 GPU. The results are presented in Table~\ref{tab:efficiency}. It can be observed that \fixed\ accelerates the generation by $1.39\times$, $1.89\times$, and $2.22\times$ at 640, 320, and 160 content tokens, respectively, and \pro\ by $1.14\times$, $1.58\times$, and $1.92\times$. Specifically, at 160 tokens, the time to first token decreases from $263.9$\,ms to $102.7$ and $122.1$\,ms for the two strategies, respectively, and the prefill KV cache decreases from $1495$ to $135$\,MiB. Note that these speedups include the cost of constructing the state. At equal layer-average budgets, the two strategies are complementary: \pro\ attains the higher retention and \fixed\ the lower latency, because \pro\ performs a larger initial selection and two internal reductions.

\begin{table}[!htb]
\begin{minipage}[t]{0.475\linewidth}
  \vspace{0pt}
  \centering
  \caption{Efficiency of LLaVA-NeXT-7B on POPE. $K$ counts content tokens (a layer average for \pro), and \method\ also keeps 48 newline tokens. Parentheses give the generation speedup over Vanilla; the generation time includes visual encoding and compression. Table~\ref{tab:efficiency-next} gives the full measurements.}
\label{tab:efficiency}
\fontsize{7.5}{8.6}\selectfont
\setlength{\tabcolsep}{1.5pt}
\begin{tabular*}{\linewidth}{@{\extracolsep{\fill}}lccr@{\extracolsep{0pt}\,}l@{\extracolsep{\fill}\hspace{2\tabcolsep}}cc@{}}\toprule
Method & $K$ & \shortstack[c]{TTFT\\(ms)} & \multicolumn{2}{c}{\shortstack{Generation\\(ms)}} & \shortstack[c]{FLOPs\\(T)} & \shortstack[c]{KV\\(MiB)}\\\midrule
Vanilla & -- & 263.94 & 289.64 & & 43.408 & 1494.82\\\midrule
\fixed & 640 & 178.96 & 207.66 & ($1.39\times$) & 10.004 & 374.81\\
\fixed & 320 & 125.85 & 153.32 & ($1.89\times$) & 5.661 & 214.81\\
\fixed & 160 & 102.72 & 130.25 & ($2.22\times$) & 3.530 & 134.81\\\midrule
\pro & 640 & 225.41 & 253.89 & ($1.14\times$) & 10.058 & 374.81\\
\pro & 320 & 153.56 & 183.57 & ($1.58\times$) & 5.675 & 214.81\\
\pro & 160 & 122.12 & 150.84 & ($1.92\times$) & 3.534 & 134.81\\
\bottomrule\end{tabular*}

\end{minipage}\hfill
\begin{minipage}[t]{0.5\linewidth}
  \vspace{0pt}
  \centering
  \caption{Composing the state (retention, \%). Rows (a)--(f) build up \fixed, with a span of 288 at $K=64$ in (a)--(e); (g) and (h) run \pro\ without and with state inheritance.}
  \label{tab:components}
  \fontsize{7.5}{8.6}\selectfont
  \setlength{\tabcolsep}{2pt}
  \begin{tabular*}{\linewidth}{@{\extracolsep{\fill}}c>{\raggedright\arraybackslash}p{1.3in}ccc@{}}\toprule
   & Configuration & $K{=}192$ & $K{=}128$ & $K{=}64$\\\midrule
  (a) & CSI representatives at persistent positions, $m\equiv\mathbf{1}$ & 99.28 & 98.81 & 96.35\\
  (b) & + $m$ composed under the conservation constraint & 99.74 & 99.18 & 97.10\\
  (c) & + coverage-gap reserve & 99.84 & 99.27 & 97.05\\
  (d) & + closed-form head gains & 99.91 & 99.50 & 97.49\\
  (e) & + $H$ composed with confidence gating & 99.95 & 99.47 & 97.59\\
  (f) & + $P$ on the full original span (main) & 99.96 & 99.49 & 98.10\\\midrule
  (g) & \pro, no inheritance & 99.75 & 99.37 & 98.84\\
  (h) & + state inherited (main) & 99.83 & 100.00 & 99.22\\
  \bottomrule\end{tabular*}
\end{minipage}
\end{table}

\subsection{Ablation study}
\label{sec:components}

To verify the effectiveness of each component, we build up the state one step at a time under \fixed\ in Table~\ref{tab:components}, starting from representatives with unit contribution weights, which the decoder reads as single patches. Note that the performance retention is averaged over six tasks other than VQAv2. According to the results, composing the contribution weights under the conservation constraint in row (b) brings the largest single gain at every budget. Moreover, the closed-form head gains in row (d) add increments that grow as the budget shrinks. Conversely, removing the attention bias from the complete construction with the representatives and their positions held fixed lowers MME by $68$ points at 192 tokens (Appendix~\ref{app:ablation}). Restoring the original span of the positions in row (f) adds $0.5$ points at 64 tokens, and compressing the span around the complete state costs over $1.4$ points at 192 and 128 tokens (Appendix~\ref{app:components}). Therefore, the decoder relies on the original grid geometry once the contribution weights are in place. Row (c) refines the selection, and row (e) composes the content with confidence gating. In total, the composed state raises the retention at 64 tokens from $96.35\%$ to $98.10\%$. Furthermore, rows (g) and (h) verify the effectiveness of inheriting the state across the reductions of \pro. Without inheritance, \pro\ loses its entire margin over \fixed\ at 128 tokens. By inheriting the state, its retention increases from $99.37\%$ to $100.00\%$ at this budget and by $0.38$ points at 64 tokens (Appendix~\ref{app:inheritance}). Finally, we vary each constant one at a time, and the retention changes by at most $0.4$ points (Appendix~\ref{app:sensitivity}). Hence, one shared configuration serves all budgets and models.

\section{Conclusion}
\label{sec:conclusion}

In this paper, we propose \method, a training-free visual token compression framework that represents a compressed image as a composable visual state, which is closed under reduction and explicitly incorporated into decoder attention. Its three modules select the representatives by coverage, compose the contribution weights under a conservation constraint with cross-stage inheritance and the content with confidence gating, and incorporate the weights and original positions into every decoder layer with closed-form head gains. One state and one composition rule thereby support both fixed and progressive compression. \method\ attains the highest retention among the compared methods on LLaVA-1.5-7B at every budget under both strategies, generalizes to three other models, and accelerates LLaVA-NeXT-7B by up to $2.22\times$. We also observe that CoViST exceeds the vanilla models on several benchmarks, and we plan to explore the underlying principles to develop beyond-lossless visual token compression methods in the future.

\label{page:main-end}
\clearpage
\subsection*{AI use statement}

In this work, we used generative AI tools for tasks with required disclosure, including providing feedback on experiments, implementing methods, supporting qualitative and thematic data analysis, and interpreting results. We have not used generative AI tools for helping develop theoretical models or conceptual frameworks, proposing or refining hypotheses, and the rest of the required disclosure tasks are not applicable to this work. Additionally, we used generative AI tools for creating scientific figures, editing software code, drafting parts of the paper, summarizing existing literature, brainstorming, sourcing/searching for information, editing the paper to improve readability, and formatting references. All AI-assisted output was reviewed by the authors to ensure its correctness. We take responsibility for the final content of this work, including text, claims or artifacts produced with the aid of generative AI.

\subsection*{Reproducibility statement}
Section~\ref{sec:method} specifies the three modules and the two execution strategies of our method, and Algorithm~\ref{alg:state} summarizes their shared construction. Table~\ref{tab:constants} lists every constant of the main configuration. Appendix~\ref{app:proofs} gives the attention-output analysis and the conservation identity, Appendix~\ref{app:implementation} details the evidence computation, the confidence function, the head gains, and state inheritance, and Appendix~\ref{app:schedule} gives the exact schedules. Appendix~\ref{app:evaluation} describes the metrics, data protocols, cross-model adaptations, and protocols of the component studies, Appendix~\ref{app:external} lists the source scores and normalization conventions of every external comparison, and Appendix~\ref{app:runtime} details the runtime measurement protocol. Additionally, the code and the runtime environment details will be released to the public.

\subsection*{Ethics statement}
This work studies visual token compression for existing vision--language models on public benchmarks and introduces no new dataset or human-subject study. Compression can change fine-grained perception and the distribution of model answers even when average benchmark performance is preserved, so aggregate retention should not be read as a guarantee of reliability, fairness, or suitability for safety-critical use. The method inherits the limitations of the underlying models and evaluation benchmarks.

\bibliography{covist}

\begin{thebibliography}{81}
\providecommand{\natexlab}[1]{#1}
\providecommand{\url}[1]{\texttt{#1}}
\expandafter\ifx\csname urlstyle\endcsname\relax
  \providecommand{\doi}[1]{doi: #1}\else
  \providecommand{\doi}{doi: \begingroup \urlstyle{rm}\Url}\fi

\bibitem[Alvar et~al.(2025)Alvar, Singh, Akbari, and Zhang]{alvar2025divprune}
Saeed~Ranjbar Alvar, Gursimran Singh, Mohammad Akbari, and Yong Zhang.
\newblock {DivPrune: Diversity-based Visual Token Pruning for Large Multimodal
  Models}.
\newblock In \emph{IEEE/CVF Conference on Computer Vision and Pattern
  Recognition}, 2025.

\bibitem[Bai et~al.(2025)Bai, Chen, Liu, Wang, Ge, Song, Dang, Wang, Wang,
  Tang, Zhong, Zhu, Yang, Li, Wan, Wang, Ding, Fu, Xu, Ye, Zhang, Xie, Cheng,
  Zhang, Yang, Xu, and Lin]{qwen2025vl}
Shuai Bai, Keqin Chen, Xuejing Liu, Jialin Wang, Wenbin Ge, Sibo Song, Kai
  Dang, Peng Wang, Shijie Wang, Jun Tang, Humen Zhong, Yuanzhi Zhu, Mingkun
  Yang, Zhaohai Li, Jianqiang Wan, Pengfei Wang, Wei Ding, Zheren Fu, Yiheng
  Xu, Jiabo Ye, Xi~Zhang, Tianbao Xie, Zesen Cheng, Hang Zhang, Zhibo Yang,
  Haiyang Xu, and Junyang Lin.
\newblock {Qwen2.5-VL Technical Report}.
\newblock \emph{arXiv preprint arXiv:2502.13923}, 2025.

\bibitem[Bolya et~al.(2023)Bolya, Fu, Dai, Zhang, Feichtenhofer, and
  Hoffman]{bolya2023tome}
Daniel Bolya, Cheng-Yang Fu, Xiaoliang Dai, Peizhao Zhang, Christoph
  Feichtenhofer, and Judy Hoffman.
\newblock {Token Merging: Your ViT But Faster}.
\newblock In \emph{International Conference on Learning Representations}, 2023.

\bibitem[Cai et~al.(2025)Cai, Yang, Gao, and Lee]{cai2025matryoshka}
Mu~Cai, Jianwei Yang, Jianfeng Gao, and Yong~Jae Lee.
\newblock {Matryoshka Multimodal Models}.
\newblock In \emph{International Conference on Learning Representations}, 2025.

\bibitem[Chen et~al.(2026{\natexlab{a}})Chen, Cai, Luo, Zhang, Yin, and
  Chen]{clse2026}
Bin Chen, Yuxiang Cai, Yadan Luo, Yi~Zhang, Jianwei Yin, and Zhi Chen.
\newblock {Spectral Evolution-Guided Token Pruning in Multimodal Large Language
  Models}.
\newblock In \emph{European Conference on Computer Vision}, 2026{\natexlab{a}}.

\bibitem[Chen et~al.(2026{\natexlab{b}})Chen, Liu, Wen, Wang, Huang, and
  Chen]{v2drop2026}
Junjie Chen, Xuyang Liu, Zichen Wen, Yiyu Wang, Siteng Huang, and Honggang
  Chen.
\newblock {Variation-aware Vision Token Dropping for Faster Large
  Vision-Language Models}.
\newblock In \emph{IEEE/CVF Conference on Computer Vision and Pattern
  Recognition}, 2026{\natexlab{b}}.

\bibitem[Chen et~al.(2024)Chen, Zhao, Liu, Bai, Lin, Zhou, and
  Chang]{chen2024fastv}
Liang Chen, Haozhe Zhao, Tianyu Liu, Shuai Bai, Junyang Lin, Chang Zhou, and
  Baobao Chang.
\newblock {An Image is Worth 1/2 Tokens After Layer 2: Plug-and-Play Inference
  Acceleration for Large Vision-Language Models}.
\newblock In \emph{European Conference on Computer Vision}, 2024.

\bibitem[Chen et~al.(2026{\natexlab{c}})Chen, Cai, Guo, Cai, Yin, and
  Chen]{priortr2026}
Zengjie Chen, Yuxiang Cai, Jingcai Guo, Taotao Cai, Jianwei Yin, and Zhi Chen.
\newblock {Accelerating Multimodal Large Language Models with Prior-Corrected
  Token Reduction}.
\newblock In \emph{European Conference on Computer Vision}, 2026{\natexlab{c}}.

\bibitem[Cho et~al.(2026)Cho, Baek, Kim, and Ham]{cho2026restore}
Hyeonwoo Cho, Donghyeon Baek, Yewon Kim, and Bumsub Ham.
\newblock {Improving Visual Token Reduction via Rectifying Distortions for
  Efficient Multimodal LLM Inference}.
\newblock In \emph{International Conference on Machine Learning}, 2026.

\bibitem[Deng et~al.(2025)Deng, Li, Zhou, and He]{scope2025}
Jinhong Deng, Wen Li, Joey~Tianyi Zhou, and Yang He.
\newblock {SCOPE: Saliency-Coverage Oriented Token Pruning for Efficient
  Multimodel LLMs}.
\newblock In \emph{Advances in Neural Information Processing Systems}, 2025.

\bibitem[Ding et~al.(2026)Ding, Li, Liu, Zhang, Xiao, Kong, and
  Zhang]{etprune2026}
Zizhong Ding, Junxian Li, Kai Liu, Shaoqiu Zhang, Xiao Xiao, Linghe Kong, and
  Yulun Zhang.
\newblock {ET-Prune: Evidence-Aware Dynamic Budgeting for Visual Token Pruning
  in Text-Rich MLLMs}.
\newblock \emph{arXiv preprint arXiv:2608.01979}, 2026.

\bibitem[Dong et~al.(2026)Dong, Hu, Zhang, Yin, Fu, and Qian]{mmtok2026}
Sixun Dong, Juhua Hu, Mian Zhang, Ming Yin, Yanjie Fu, and Qi~Qian.
\newblock {MMTok: Multimodal Coverage Maximization for Efficient Inference of
  VLMs}.
\newblock In \emph{International Conference on Learning Representations}, 2026.

\bibitem[Du et~al.(2026)Du, Deng, Liu, Zhang, Chen, and Ren]{du2026coin}
Chenxi Du, Yongheng Deng, Jiani Liu, Yujia Zhang, Xi~Chen, and Ju~Ren.
\newblock {CoIn: Coverage and Informativeness-Guided Token Reduction for
  Efficient Large Multimodal Models}.
\newblock In \emph{IEEE/CVF Conference on Computer Vision and Pattern
  Recognition}, pp.\  10492--10501, 2026.

\bibitem[Fang et~al.(2026)Fang, Lyu, Zhang, Lu, Yu, and Pei]{prunesid2026}
Zhengyao Fang, Pengyuan Lyu, Chengquan Zhang, Guangming Lu, Jun Yu, and Wenjie
  Pei.
\newblock {Prune Redundancy, Preserve Essence: Vision Token Compression in VLMs
  via Synergistic Importance-Diversity}.
\newblock In \emph{International Conference on Learning Representations}, 2026.

\bibitem[Feng et~al.(2025)Feng, Lv, Cao, Xie, and Zhou]{feng2025adakv}
Yuan Feng, Junlin Lv, Yukun Cao, Xike Xie, and S.~Kevin Zhou.
\newblock {Ada-KV: Optimizing KV Cache Eviction by Adaptive Budget Allocation
  for Efficient LLM Inference}.
\newblock In \emph{Advances in Neural Information Processing Systems}, 2025.

\bibitem[Fu et~al.(2025)Fu, Chen, Shen, Qin, Zhang, Lin, Yang, Zheng, Li, Sun,
  Wu, Ji, Shan, and He]{fu2023mme}
Chaoyou Fu, Peixian Chen, Yunhang Shen, Yulei Qin, Mengdan Zhang, Xu~Lin,
  Jinrui Yang, Xiawu Zheng, Ke~Li, Xing Sun, Yunsheng Wu, Rongrong Ji, Caifeng
  Shan, and Ran He.
\newblock {MME: A Comprehensive Evaluation Benchmark for Multimodal Large
  Language Models}.
\newblock In \emph{Advances in Neural Information Processing Systems Datasets
  and Benchmarks Track}, 2025.

\bibitem[Goyal et~al.(2017)Goyal, Khot, Summers-Stay, Batra, and
  Parikh]{goyal2017vqav2}
Yash Goyal, Tejas Khot, Douglas Summers-Stay, Dhruv Batra, and Devi Parikh.
\newblock Making the {V} in {VQA} matter: Elevating the role of image
  understanding in visual question answering.
\newblock In \emph{IEEE/CVF Conference on Computer Vision and Pattern
  Recognition}, pp.\  6904--6913, 2017.

\bibitem[Guo et~al.(2026)Guo, Wang, Zhang, Meng, Zhang, Cao, Jiang, Wu, Wu,
  Chen, Luo, Cheng, Tang, Li, Li, Huang, Wang, and Zhang]{starpro2026}
Yichen Guo, Tinghao Wang, Qizhe Zhang, Lingbei Meng, Yuan Zhang, Jiajun Cao,
  Hao Jiang, Chenwei Wu, Jixian Wu, Sixiang Chen, Tao Luo, Hongyang Cheng, Kai
  Tang, Chenxi Li, Renyuan Li, Xiande Huang, Wenya Wang, and Shanghang Zhang.
\newblock {STAR-Pro: Stage-Wise Token Adaptive Reduction with Progressive
  Refinement for Efficient Large Vision-Language Models}.
\newblock \emph{arXiv preprint arXiv:2609.05916}, 2026.

\bibitem[He et~al.(2026)He, Young, and Xu]{stepprune2026}
Landi He, Shawn Young, and Lijian Xu.
\newblock {Stepwise Token Selection for Efficient Multimodal Large Language
  Models}.
\newblock \emph{arXiv preprint arXiv:2606.16067}, 2026.

\bibitem[Hudson \& Manning(2019)Hudson and Manning]{hudson2019gqa}
Drew~A. Hudson and Christopher~D. Manning.
\newblock {GQA: A New Dataset for Real-World Visual Reasoning and Compositional
  Question Answering}.
\newblock In \emph{IEEE/CVF Conference on Computer Vision and Pattern
  Recognition}, 2019.

\bibitem[Jia et~al.(2026)Jia, Tang, and Yang]{s2prune2026}
Yuanyuan Jia, Shunpu Tang, and Qianqian Yang.
\newblock {S$^2$Prune: Spatially Structured Visual Token Pruning for Multimodal
  Large Language Models}.
\newblock \emph{arXiv preprint arXiv:2609.01224}, 2026.

\bibitem[Kim et~al.(2026)Kim, Zhang, Liu, Jung, Lee, and Hong]{zooprune2026}
Youngeun Kim, Youjia Zhang, Huiling Liu, Aecheon Jung, Sunwoo Lee, and Sungeun
  Hong.
\newblock {ZOO-Prune: Training-Free Token Pruning via Zeroth-Order Gradient
  Estimation in Vision-Language Models}.
\newblock In \emph{IEEE/CVF Conference on Computer Vision and Pattern
  Recognition}, pp.\  39572--39582, 2026.

\bibitem[Lee et~al.(2026)Lee, Wen, and Choi]{spare2026}
Jaeyeon Lee, Shunjie Wen, and Dong-Wan Choi.
\newblock {Moving Beyond Diversity: Visual Token Pruning as Subspace
  Reconstruction for Efficient VLMs}.
\newblock \emph{arXiv preprint arXiv:2606.18681}, 2026.

\bibitem[Li et~al.(2025{\natexlab{a}})Li, Duan, Zhang, Ma, Xie, Carneiro,
  Yaqub, and Wang]{transprune2026}
Ao~Li, Yuxiang Duan, Jinghui Zhang, Congbo Ma, Yutong Xie, Gustavo Carneiro,
  Mohammad Yaqub, and Hu~Wang.
\newblock {TransPrune: Token Transition Pruning for Efficient Large
  Vision-Language Model}.
\newblock \emph{arXiv preprint arXiv:2507.20630}, 2025{\natexlab{a}}.

\bibitem[Li et~al.(2024)Li, Ge, Ge, Wang, Wang, Zhang, and Shan]{li2024seed}
Bohao Li, Yuying Ge, Yixiao Ge, Guangzhi Wang, Rui Wang, Ruimao Zhang, and Ying
  Shan.
\newblock {SEED-Bench: Benchmarking Multimodal Large Language Models}.
\newblock In \emph{IEEE/CVF Conference on Computer Vision and Pattern
  Recognition}, pp.\  13299--13308, 2024.

\bibitem[Li et~al.(2026{\natexlab{a}})Li, Chen, Chen, Shan, Zuo, Lyu, An, Li,
  and Yang]{occamtoken2026}
Geng Li, Guohao Chen, Ting Chen, Shilin Shan, Kuangji Zuo, Bofan Lyu, Tuo An,
  Gen Li, and Jianfei Yang.
\newblock {OccamToken: Efficient VLM Inference with Training-Free and
  Budget-Adaptive Token Pruning}.
\newblock \emph{arXiv preprint arXiv:2605.29657}, 2026{\natexlab{a}}.

\bibitem[Li et~al.(2026{\natexlab{b}})Li, Ji, Zhang, and Li]{li2026care}
Jiasheng Li, Zhong Ji, Yan Zhang, and Huihui Li.
\newblock {Calibrate Before Reason: Robust Visual Token Reduction against
  Semantic Drift in VLMs}.
\newblock \emph{arXiv preprint arXiv:2607.27700}, 2026{\natexlab{b}}.

\bibitem[Li et~al.(2026{\natexlab{c}})Li, Hu, Huang, Li, Zhong, and
  Wang]{sinkpruner2026}
Shiyu Li, Zi-Yuan Hu, Shijia Huang, Yanyang Li, Yiwu Zhong, and Liwei Wang.
\newblock {SinkPruner: Sink-Free Visual Token Pruning for Multimodal Large
  Language Models}.
\newblock In \emph{Findings of the Association for Computational Linguistics:
  EMNLP}, 2026{\natexlab{c}}.

\bibitem[Li et~al.(2026{\natexlab{d}})Li, Zheng, Zhao, Liang, Liu, Zhu, Chen,
  Zhou, Zhao, and Guo]{li2026crisp}
Xu~Li, Yi~Zheng, Mengyang Zhao, Yuxuan Liang, Zhe Liu, Rui Zhu, Xiaolei Chen,
  Wei Zhou, Baoquan Zhao, and Juncen Guo.
\newblock {CRISP: Pre-LLM Yet Text-Driven Visual Token Pruning for Efficient
  LVLM Inference}.
\newblock In \emph{IEEE International Conference on Multimedia and Expo},
  2026{\natexlab{d}}.

\bibitem[Li et~al.(2025{\natexlab{b}})Li, Zhan, Chen, Liu, and Lu]{mob2026}
Yangfu Li, Hongjian Zhan, Tianyi Chen, Qi~Liu, and Yue Lu.
\newblock {Why 1 + 1 < 1 in Visual Token Pruning: Beyond Naive Integration via
  Multi-Objective Balanced Covering}.
\newblock In \emph{Advances in Neural Information Processing Systems},
  2025{\natexlab{b}}.

\bibitem[Li et~al.(2023)Li, Du, Zhou, Wang, Zhao, and Wen]{li2023pope}
Yifan Li, Yifan Du, Kun Zhou, Jinpeng Wang, Wayne~Xin Zhao, and Ji-Rong Wen.
\newblock {Evaluating Object Hallucination in Large Vision-Language Models}.
\newblock In \emph{Conference on Empirical Methods in Natural Language
  Processing}, 2023.

\bibitem[Li et~al.(2026{\natexlab{e}})Li, Li, Li, Chen, and
  Zhang]{specflow2026}
Zhaoyang Li, Yanjun Li, Wangkai Li, Yujia Chen, and Tianzhu Zhang.
\newblock {Spectral Heat Flow for Conservative Token Condensation in
  Vision-Language Models}.
\newblock In \emph{International Conference on Machine Learning},
  2026{\natexlab{e}}.

\bibitem[Liu et~al.(2024{\natexlab{a}})Liu, Li, Li, and Lee]{llava2024}
Haotian Liu, Chunyuan Li, Yuheng Li, and Yong~Jae Lee.
\newblock {Improved Baselines with Visual Instruction Tuning}.
\newblock In \emph{IEEE/CVF Conference on Computer Vision and Pattern
  Recognition}, 2024{\natexlab{a}}.

\bibitem[Liu et~al.(2024{\natexlab{b}})Liu, Li, Li, Li, Zhang, Shen, and
  Lee]{llavanext2024}
Haotian Liu, Chunyuan Li, Yuheng Li, Bo~Li, Yuanhan Zhang, Sheng Shen, and
  Yong~Jae Lee.
\newblock {LLaVA-NeXT: Improved Reasoning, OCR, and World Knowledge}.
\newblock LLaVA project blog, 2024{\natexlab{b}}.
\newblock URL \url{https://llava-vl.github.io/blog/2024-01-30-llava-next/}.

\bibitem[Liu et~al.(2026)Liu, Du, Zhu, Lian, Li, Chen, Guan, and
  Wang]{hiprune2026}
Jizhihui Liu, Feiyi Du, Guangdao Zhu, Niu Lian, Jun Li, Bin Chen, Weili Guan,
  and Yaowei Wang.
\newblock {HiPrune: Hierarchical Attention for Efficient Token Pruning in
  Vision-Language Models}.
\newblock In \emph{Findings of the Association for Computational Linguistics:
  ACL}, 2026.

\bibitem[Liu et~al.(2024{\natexlab{c}})Liu, Duan, Zhang, Li, Zhang, Zhao, Yuan,
  Wang, He, Liu, Chen, and Lin]{liu2024mmbench}
Yuan Liu, Haodong Duan, Yuanhan Zhang, Bo~Li, Songyang Zhang, Wangbo Zhao, Yike
  Yuan, Jiaqi Wang, Conghui He, Ziwei Liu, Kai Chen, and Dahua Lin.
\newblock {MMBench: Is Your Multi-modal Model an All-around Player?}
\newblock In \emph{European Conference on Computer Vision}, 2024{\natexlab{c}}.

\bibitem[Lu et~al.(2026{\natexlab{a}})Lu, Zhang, Jin, Hu, Shi, Jiang, Hu, and
  Li]{lrcp2026}
Hongyu Lu, Feng Zhang, Wenwei Jin, Huanling Hu, Tianjun Shi, Shikai Jiang, Yao
  Hu, and Jiawei Li.
\newblock {LRCP: Low-Rank Compressibility Guided Visual Token Pruning for
  Efficient LVLMs}.
\newblock \emph{arXiv preprint arXiv:2605.15621}, 2026{\natexlab{a}}.

\bibitem[Lu et~al.(2026{\natexlab{b}})Lu, Zhang, Jin, Hu, Zhang, Hu, Li, and
  Jiang]{evocut2026}
Hongyu Lu, Feng Zhang, Wenwei Jin, Huanling Hu, Pengfei Zhang, Yao Hu, Jiawei
  Li, and Shikai Jiang.
\newblock {EvoCut: Multi-Layer Evolution-Aware Visual Token Compression for
  Efficient Large Vision-Language Models}.
\newblock \emph{arXiv preprint arXiv:2606.01756}, 2026{\natexlab{b}}.

\bibitem[Lu et~al.(2022)Lu, Mishra, Xia, Qiu, Chang, Zhu, Tafjord, Clark, and
  Kalyan]{lu2022scienceqa}
Pan Lu, Swaroop Mishra, Tony Xia, Liang Qiu, Kai-Wei Chang, Song-Chun Zhu,
  Oyvind Tafjord, Peter Clark, and Ashwin Kalyan.
\newblock {Learn to Explain: Multimodal Reasoning via Thought Chains for
  Science Question Answering}.
\newblock In \emph{Advances in Neural Information Processing Systems}, 2022.

\bibitem[Lu et~al.(2026{\natexlab{c}})Lu, Wu, Xu, Luo, Fan, Chen, Dong, Yang,
  and Guo]{rora2026}
Qiyanhui Lu, Han Wu, Rongjian Xu, Tingzhang Luo, Cheng Fan, Xinghao Chen,
  Minjing Dong, Jufeng Yang, and Jianyuan Guo.
\newblock {RoRA: Role-Oriented Regional Allocation for Visual Token Pruning in
  MLLMs}.
\newblock \emph{arXiv preprint arXiv:2608.07088}, 2026{\natexlab{c}}.

\bibitem[Ma et~al.(2026{\natexlab{a}})Ma, Jiang, Li, Liu, Li, Xu, and
  Zhang]{provip2026}
Chaofang Ma, Lin Jiang, Carol~Jingyi Li, Xingyu Liu, Zeyu Li, Jiang Xu, and Wei
  Zhang.
\newblock {Not All Attention Heads Contribute to Critical Visual Token
  Selection: Head-Aware Pruning Matters More}.
\newblock \emph{arXiv preprint arXiv:2608.25332}, 2026{\natexlab{a}}.

\bibitem[Ma et~al.(2026{\natexlab{b}})Ma, Xiao, Chen, Min, Zhu, Wang, and
  Liao]{ma2026desap}
Kexin Ma, Jing Xiao, Chaofeng Chen, Geyong Min, Guibo Zhu, Jinqiao Wang, and
  Liang Liao.
\newblock {Decoupled Similarity for Task-Aware Token Pruning in Large
  Vision-Language Models}.
\newblock In \emph{ACM International Conference on Multimedia},
  2026{\natexlab{b}}.

\bibitem[Ma et~al.(2026{\natexlab{c}})Ma, Zhang, Wang, Chen, Song, and
  Zheng]{apet2026}
Qiankun Ma, Ziyao Zhang, Haofei Wang, Jie Chen, Zhen Song, and Hairong Zheng.
\newblock {ApET: Approximation-Error Guided Token Compression for Efficient
  VLMs}.
\newblock In \emph{IEEE/CVF Conference on Computer Vision and Pattern
  Recognition}, 2026{\natexlab{c}}.

\bibitem[Oh \& Kim(2026)Oh and Kim]{oh2026anchorprune}
Kyuan Oh and Bumsoo Kim.
\newblock {AnchorPrune: Relevance-Anchored Contextual Expansion for Visual
  Token Pruning}.
\newblock In \emph{European Conference on Computer Vision}, 2026.

\bibitem[Ou et~al.(2026)Ou, Song, Zhou, Sun, Zhang, and Luo]{sieve2026}
Congyang Ou, Ruike Song, Yang Zhou, Libo Sun, Haokui Zhang, and Zhenbo Luo.
\newblock {When Vision Becomes Text: Visual Token Pruning via Cross-Modal
  Residual Guidance in VLMs}.
\newblock \emph{arXiv preprint arXiv:2608.10489}, 2026.

\bibitem[Qian et~al.(2026)Qian, Wang, Shi, Jiang, Gao, and Yu]{e2s2026}
Taoyu Qian, Qi~Wang, Daqian Shi, Yuanhao Jiang, Shang Gao, and Hualong Yu.
\newblock {E2S-Pruner: Progressive Two-Stage Evidence Fusion for Visual Token
  Pruning in Vision-Language Models}.
\newblock \emph{arXiv preprint arXiv:2608.23253}, 2026.

\bibitem[Radford et~al.(2021)Radford, Kim, Hallacy, Ramesh, Goh, Agarwal,
  Sastry, Askell, Mishkin, Clark, Krueger, and Sutskever]{radford2021clip}
Alec Radford, Jong~Wook Kim, Chris Hallacy, Aditya Ramesh, Gabriel Goh,
  Sandhini Agarwal, Girish Sastry, Amanda Askell, Pamela Mishkin, Jack Clark,
  Gretchen Krueger, and Ilya Sutskever.
\newblock {Learning Transferable Visual Models From Natural Language
  Supervision}.
\newblock In \emph{International Conference on Machine Learning}, 2021.

\bibitem[Shang et~al.(2025)Shang, Cai, Xu, Lee, and Yan]{shang2025prumerge}
Yuzhang Shang, Mu~Cai, Bingxin Xu, Yong~Jae Lee, and Yan Yan.
\newblock {LLaVA-PruMerge: Adaptive Token Reduction for Efficient Large
  Multimodal Models}.
\newblock In \emph{Proceedings of the IEEE/CVF International Conference on
  Computer Vision}, pp.\  22857--22867, 2025.

\bibitem[Singh et~al.(2019)Singh, Natarajan, Shah, Jiang, Chen, Batra, Parikh,
  and Rohrbach]{singh2019textvqa}
Amanpreet Singh, Vivek Natarajan, Meet Shah, Yu~Jiang, Xinlei Chen, Dhruv
  Batra, Devi Parikh, and Marcus Rohrbach.
\newblock {Towards VQA Models That Can Read}.
\newblock In \emph{IEEE/CVF Conference on Computer Vision and Pattern
  Recognition}, 2019.

\bibitem[Song et~al.(2025)Song, Wang, Chen, Wang, Guan, and Wang]{trim2025}
Dingjie Song, Wenjun Wang, Shunian Chen, Xidong Wang, Michael~X. Guan, and
  Benyou Wang.
\newblock {Less is More: A Simple yet Effective Token Reduction Method for
  Efficient Multi-modal LLMs}.
\newblock In \emph{Proceedings of the 31st International Conference on
  Computational Linguistics}, pp.\  7614--7623, 2025.

\bibitem[Su et~al.(2021)Su, Lu, Pan, Murtadha, Wen, and Liu]{su2024roformer}
Jianlin Su, Yu~Lu, Shengfeng Pan, Ahmed Murtadha, Bo~Wen, and Yunfeng Liu.
\newblock {RoFormer: Enhanced Transformer with Rotary Position Embedding}.
\newblock \emph{arXiv preprint arXiv:2104.09864}, 2021.

\bibitem[Sun et~al.(2026{\natexlab{a}})Sun, Ji, Jin, Deng, Fu, and
  Wang]{hap2026}
Yuanhao Sun, Huawei Ji, Yuan Jin, Cheng Deng, Luoyi Fu, and Xinbing Wang.
\newblock {HAP: Head-Adaptive Visual Token Pruning via Cross-Modal Alignment}.
\newblock \emph{arXiv preprint arXiv:2608.23921}, 2026{\natexlab{a}}.

\bibitem[Sun et~al.(2026{\natexlab{b}})Sun, Ma, Liu, Chen, Tang, Hu, and
  Xu]{sun2026ivc}
Zhichao Sun, Yidong Ma, Gang Liu, Yibo Chen, Xu~Tang, Yao Hu, and Yongchao Xu.
\newblock {IVC-Prune: Revealing the Implicit Visual Coordinates in LVLMs for
  Vision Token Pruning}.
\newblock In \emph{International Conference on Learning Representations},
  2026{\natexlab{b}}.

\bibitem[Tong et~al.(2026)Tong, Bai, Zhu, Jiang, and Liu]{tong2026fsr}
Enwei Tong, Yuanchao Bai, Yao Zhu, Junjun Jiang, and Xianming Liu.
\newblock {Focus-Scan-Refine: From Human Visual Perception to Efficient Visual
  Token Pruning}.
\newblock \emph{arXiv preprint arXiv:2602.05809}, 2026.

\bibitem[Tong et~al.(2025)Tong, Jin, Qin, Li, Zou, Li, Li, and Li]{flowcut2026}
Jintao Tong, Wenwei Jin, Pengda Qin, Anqi Li, Yixiong Zou, Yuhong Li, Yuhua Li,
  and Ruixuan Li.
\newblock {FlowCut: Rethinking Redundancy via Information Flow for Efficient
  Vision-Language Models}.
\newblock In \emph{Advances in Neural Information Processing Systems}, 2025.

\bibitem[Wang et~al.(2026{\natexlab{a}})Wang, Zhang, Han, and Zhang]{sts2026}
Jiahui Wang, Kai Zhang, Mai Han, and Huanghe Zhang.
\newblock {When Attention Collapses: Stage-Aware Visual Token Pruning from
  Structure to Semantics}.
\newblock \emph{arXiv preprint arXiv:2606.03569}, 2026{\natexlab{a}}.

\bibitem[Wang et~al.(2026{\natexlab{b}})Wang, Guo, Huang, Lu, Zhang, Li, Zhang,
  Cao, Shen, Du, Gan, Wang, Cong, and Zhang]{tops2026}
Tinghao Wang, Yichen Guo, Rui Huang, Zheng Lu, Qizhe Zhang, Chenxi Li, Yuan
  Zhang, Jiajun Cao, Zhirong Shen, Yaosong Du, Guangyan Gan, Wenya Wang,
  Lin~William Cong, and Shanghang Zhang.
\newblock {TOPS: First-Principles Visual Token Pruning via Constructing Token
  Optimal Preservation Sets for Efficient MLLM Inference}.
\newblock \emph{arXiv preprint arXiv:2606.27161}, 2026{\natexlab{b}}.

\bibitem[Wang et~al.(2026{\natexlab{c}})Wang, Yang, and Shen]{wang2026eadp}
Xuehui Wang, Xuankun Yang, and Wei Shen.
\newblock {Combating Textual Noise and Redundancy: Entropy-Aware Dense Visual
  Token Pruning}.
\newblock In \emph{European Conference on Computer Vision}, 2026{\natexlab{c}}.

\bibitem[Wang et~al.(2026{\natexlab{d}})Wang, Wu, Ni, Yang, Liu, Yang, Wen, He,
  Tang, Liu, and Zhou]{wtr2026}
Yahong Wang, Juncheng Wu, Zhangkai Ni, Longzhen Yang, Yihang Liu, Chengmei
  Yang, Ying Wen, Lianghua He, Xianfeng Tang, Hui Liu, and Yuyin Zhou.
\newblock {When Token Pruning is Worse than Random: Understanding Visual Token
  Information in VLLMs}.
\newblock In \emph{IEEE/CVF Conference on Computer Vision and Pattern
  Recognition}, 2026{\natexlab{d}}.

\bibitem[Wang et~al.(2026{\natexlab{e}})Wang, Qiao, Diao, Zhuge, Zhang, Zhang,
  Zhang, and Lu]{wang2026era}
Yuhao Wang, Mu~Qiao, Haiwen Diao, Yunzhi Zhuge, Pingping Zhang, Xindong Zhang,
  Lei Zhang, and Huchuan Lu.
\newblock {ERA: Entropy-Guided Visual Token Pruning with Rectified Attention
  for Efficient MLLMs}.
\newblock \emph{arXiv preprint arXiv:2606.31982}, 2026{\natexlab{e}}.

\bibitem[Wen et~al.(2026)Wen, Lee, and Choi]{cenprune2026}
Shunjie Wen, Jaeyeon Lee, and Dong-Wan Choi.
\newblock {Centering before Pruning: Lightweight Geometry Correction for
  Diversity-Based Visual Token Pruning in LVLMs}.
\newblock \emph{arXiv preprint arXiv:2608.30263}, 2026.

\bibitem[Wen et~al.(2025)Wen, Gao, Wang, Zhang, Zhang, Li, He, and
  Zhang]{dart2025}
Zichen Wen, Yifeng Gao, Shaobo Wang, Junyuan Zhang, Qintong Zhang, Weijia Li,
  Conghui He, and Linfeng Zhang.
\newblock {Stop Looking for Important Tokens in Multimodal Language Models:
  Duplication Matters More}.
\newblock In \emph{Conference on Empirical Methods in Natural Language
  Processing}, 2025.

\bibitem[Wu et~al.(2026{\natexlab{a}})Wu, Fan, Dai, Tong, Ma, and
  Shen]{hidrop2026}
Hao Wu, Yingqi Fan, Jinyang Dai, Junlong Tong, Yunpu Ma, and Xiaoyu Shen.
\newblock {HiDrop: Hierarchical Vision Token Reduction in MLLMs via Late
  Injection, Concave Pyramid Pruning, and Early Exit}.
\newblock In \emph{International Conference on Learning Representations},
  2026{\natexlab{a}}.

\bibitem[Wu et~al.(2026{\natexlab{b}})Wu, Ma, Ni, Zhang, Shu, Jiang, and
  Chen]{vlmpruner2026}
Zhenkai Wu, Xiaowen Ma, Zhenliang Ni, Dengming Zhang, Han Shu, Xin Jiang, and
  Xinghao Chen.
\newblock {VLM-Pruner: Buffering for Spatial Sparsity in an Efficient VLM
  Centrifugal Token Pruning Paradigm}.
\newblock In \emph{IEEE/CVF Conference on Computer Vision and Pattern
  Recognition}, 2026{\natexlab{b}}.

\bibitem[Xing et~al.(2025)Xing, Huang, Dong, Lu, Zhang, Zang, Cao, He, Wang,
  Wu, and Lin]{xing2025pyramiddrop}
Long Xing, Qidong Huang, Xiaoyi Dong, Jiajie Lu, Pan Zhang, Yuhang Zang, Yuhang
  Cao, Conghui He, Jiaqi Wang, Feng Wu, and Dahua Lin.
\newblock {Conical Visual Concentration for Efficient Large Vision-Language
  Models}.
\newblock In \emph{IEEE/CVF Conference on Computer Vision and Pattern
  Recognition}, 2025.

\bibitem[Xu et~al.(2026)Xu, Shi, and Gao]{score2026}
Tong Xu, Hailong Shi, and Xingyu Gao.
\newblock {SCoRe: Salience-Coverage Reduction for Vision Token Pruning in
  Vision-Language Models}.
\newblock In \emph{Proceedings of the IEEE/CVF Conference on Computer Vision
  and Pattern Recognition}, pp.\  24686--24695, 2026.

\bibitem[Yang et~al.(2025)Yang, Chen, Tian, Wang, Li, Yu, and
  Jia]{yang2025visionzip}
Senqiao Yang, Yukang Chen, Zhuotao Tian, Chengyao Wang, Jingyao Li, Bei Yu, and
  Jiaya Jia.
\newblock {VisionZip: Longer is Better but Not Necessary in Vision Language
  Models}.
\newblock In \emph{IEEE/CVF Conference on Computer Vision and Pattern
  Recognition}, 2025.

\bibitem[Ye et~al.(2024)Ye, Wu, Lin, and Zhou]{fitprune2025}
Weihao Ye, Qiong Wu, Wenhao Lin, and Yiyi Zhou.
\newblock {Fit and Prune: Fast and Training-free Visual Token Pruning for
  Multi-modal Large Language Models}.
\newblock \emph{arXiv preprint arXiv:2409.10197}, 2024.

\bibitem[Yu et~al.(2026)Yu, Li, Qu, Wang, Chen, and Zhu]{visiontrim2026}
Hanxun Yu, Wentong Li, Xuan Qu, Song Wang, Junbo Chen, and Jianke Zhu.
\newblock {VisionTrim: Unified Vision Token Compression for Training-Free MLLM
  Acceleration}.
\newblock In \emph{International Conference on Learning Representations}, 2026.

\bibitem[Yue et~al.(2024)Yue, Ni, Zhang, Zheng, Liu, Zhang, Stevens, Jiang,
  Ren, Sun, Wei, Yu, Yuan, Sun, Yin, Zheng, Yang, Liu, Huang, Sun, Su, and
  Chen]{yue2024mmmu}
Xiang Yue, Yuansheng Ni, Kai Zhang, Tianyu Zheng, Ruoqi Liu, Ge~Zhang, Samuel
  Stevens, Dongfu Jiang, Weiming Ren, Yuxuan Sun, Cong Wei, Botao Yu, Ruibin
  Yuan, Renliang Sun, Ming Yin, Boyuan Zheng, Zhenzhu Yang, Yibo Liu, Wenhao
  Huang, Huan Sun, Yu~Su, and Wenhu Chen.
\newblock {MMMU: A Massive Multi-discipline Multimodal Understanding and
  Reasoning Benchmark for Expert AGI}.
\newblock In \emph{IEEE/CVF Conference on Computer Vision and Pattern
  Recognition}, pp.\  9556--9567, 2024.

\bibitem[Zhang et~al.(2026{\natexlab{a}})Zhang, Ma, Fang, Yu, Zhang, Zhang, Mi,
  and Yu]{zhang2026vscan}
Ce~Zhang, Kaixin Ma, Tianqing Fang, Wenhao Yu, Hongming Zhang, Zhisong Zhang,
  Haitao Mi, and Dong Yu.
\newblock {VScan: Rethinking Visual Token Reduction for Efficient Large
  Vision-Language Models}.
\newblock \emph{Transactions on Machine Learning Research}, 2026{\natexlab{a}}.

\bibitem[Zhang et~al.(2026{\natexlab{b}})Zhang, Yu, Wu, Wen, Yan, Ding, Qi, and
  Zhang]{d2pruner2026}
Evelyn Zhang, Fufu Yu, Aoqi Wu, Zichen Wen, Ke~Yan, Shouhong Ding, Biqing Qi,
  and Linfeng Zhang.
\newblock {D2Pruner: Debiased Importance and Structural Diversity for MLLM
  Token Pruning}.
\newblock In \emph{Proceedings of the AAAI Conference on Artificial
  Intelligence}, pp.\  12412--12420, 2026{\natexlab{b}}.

\bibitem[Zhang et~al.(2024)Zhang, Cheng, Lu, Zhuo, Wang, Cao, Guo, She, and
  Zhang]{fastervlm2024}
Qizhe Zhang, Aosong Cheng, Ming Lu, Zhiyong Zhuo, Minqi Wang, Jiajun Cao,
  Shaobo Guo, Qi~She, and Shanghang Zhang.
\newblock {[CLS] Attention is All You Need for Training-Free Visual Token
  Pruning: Make VLM Inference Faster}.
\newblock \emph{arXiv preprint arXiv:2412.01818v1}, 2024.

\bibitem[Zhang et~al.(2025{\natexlab{a}})Zhang, Cheng, Lu, Zhang, Zhuo, Cao,
  Guo, She, and Zhang]{zhang2025vispruner}
Qizhe Zhang, Aosong Cheng, Ming Lu, Renrui Zhang, Zhiyong Zhuo, Jiajun Cao,
  Shaobo Guo, Qi~She, and Shanghang Zhang.
\newblock {Beyond Text-Visual Attention: Exploiting Visual Cues for Effective
  Token Pruning in VLMs}.
\newblock In \emph{Proceedings of the IEEE/CVF International Conference on
  Computer Vision}, pp.\  20857--20867, 2025{\natexlab{a}}.

\bibitem[Zhang et~al.(2025{\natexlab{b}})Zhang, Liu, Li, Lu, Zhang, Pan, She,
  and Zhang]{zheng2025cdpruner}
Qizhe Zhang, Mengzhen Liu, Lichen Li, Ming Lu, Yuan Zhang, Junwen Pan, Qi~She,
  and Shanghang Zhang.
\newblock {Beyond Attention or Similarity: Maximizing Conditional Diversity for
  Token Pruning in MLLMs}.
\newblock In \emph{Advances in Neural Information Processing Systems},
  2025{\natexlab{b}}.

\bibitem[Zhang et~al.(2025{\natexlab{c}})Zhang, Fang, Yang, and
  Feng]{zhang2025llavamini}
Shaolei Zhang, Qingkai Fang, Zhe Yang, and Yang Feng.
\newblock {LLaVA-Mini: Efficient Image and Video Large Multimodal Models with
  One Vision Token}.
\newblock In \emph{International Conference on Learning Representations},
  2025{\natexlab{c}}.

\bibitem[Zhang et~al.(2025{\natexlab{d}})Zhang, Fan, Ma, Zheng, Huang, Cheng,
  Gudovskiy, Okuno, Nakata, Keutzer, and Zhang]{zhang2025sparsevlm}
Yuan Zhang, Chun-Kai Fan, Junpeng Ma, Wenzhao Zheng, Tao Huang, Kuan Cheng,
  Denis Gudovskiy, Tomoyuki Okuno, Yohei Nakata, Kurt Keutzer, and Shanghang
  Zhang.
\newblock {SparseVLM: Visual Token Sparsification for Efficient Vision-Language
  Model Inference}.
\newblock In \emph{International Conference on Machine Learning},
  2025{\natexlab{d}}.

\bibitem[Zhong et~al.(2026)Zhong, An, Wang, Li, Yang, and He]{dive2026}
Chen Zhong, Xiao An, Zijie Wang, Jiepan Li, Guangyi Yang, and Wei He.
\newblock {DIVE: Dynamic Iterative Visual Evidence Construction for Efficient
  Vision-Language Models}.
\newblock \emph{arXiv preprint arXiv:2608.04496}, 2026.

\bibitem[Zhou et~al.(2026)Zhou, Zhang, Yang, Gao, and Wang]{zhou2026calibrated}
Qing Zhou, Hongyuan Zhang, Tao Yang, Junyu Gao, and Qi~Wang.
\newblock {Statistically Calibrated Scaling for Token Merging in Transformers}.
\newblock In \emph{International Conference on Machine Learning}, 2026.

\bibitem[Zhu et~al.(2026)Zhu, You, Jiang, Yang, Liu, and Yuan]{coverpruner2026}
Qingchan Zhu, Weihang You, Hanqi Jiang, Changdi Yang, Tianming Liu, and Geng
  Yuan.
\newblock {Who Speaks for the Pruned? Visual Token Pruning as Coverage
  Optimization}.
\newblock In \emph{Conference on Empirical Methods in Natural Language
  Processing}, 2026.

\bibitem[Zou et~al.(2025)Zou, Lu, Wang, Yan, Lyu, Zheng, Zhang, and
  Hu]{zou2025holov}
Xin Zou, Di~Lu, Yizhou Wang, Yibo Yan, Yuanhuiyi Lyu, Xu~Zheng, Linfeng Zhang,
  and Xuming Hu.
\newblock {Don't Just Chase "Highlighted Tokens" in MLLMs: Revisiting Visual
  Holistic Context Retention}.
\newblock In \emph{Advances in Neural Information Processing Systems}, 2025.

\end{thebibliography}
\bibliographystyle{iclr2027_conference}
\clearpage
\appendix
\etocdepthtag.toc{appendixmatter}
\etocsettagdepth{mainmatter}{none}
\etocsettagdepth{appendixmatter}{subsection}
\etocsettocstyle{\section*{Appendix Contents}}{}
\begingroup\hypersetup{linkcolor=black}\tableofcontents\endgroup
\clearpage
\section{Attention-output analysis and conservation identity}
\label{app:proofs}

\subsection{Setting and statement}
\label{app:setting}

We consider one attention head and a fixed text query that follows the image. Specifically, we partition the uncompressed visual tokens into the groups $C_i$ represented by each representative, with total weight $Z_i=\sum_{j\in C_i}a_j$ and conditional value $\bar v_i=\sum_{j\in C_i}a_jv_j/Z_i$, where $a_j>0$ is the unnormalized attention weight of token $j$. The uncompressed output is then exactly the normalized weighted sum over $(Z_i,\bar v_i)$ and the text tokens, whereas the compressed computation replaces each group by $(\widehat Z_i,\widehat v_i)$.

\begin{proposition}[Local attention approximation]
\label{prop:attention}
Suppose that $\|\bar v_i-\widehat v_i\|\leq\eta$ and $|\log\widehat Z_i-\log Z_i|\leq\delta$ for every visual group, that all compressed representative values and text values have norm at most $V$, and that the query and the text keys and values are unchanged. Then
\begin{equation}
  \norm{y-\widehat y}\leq
  \eta+2V\tanh(\delta/2).
  \label{eq:bound}
\end{equation}
\end{proposition}

The two error terms are additive, and the state controls both of them. Specifically, CSC updates the representative features and contribution weights, SCA maps the weights and original positions to $\widehat Z_i$, and CSI determines the representative set on which both act. The content composition also changes the representative key and hence $\widehat Z_i$, which couples the two error sources. Therefore, the content update in CSC is confidence-gated and uses a small coefficient $\alpha$. Moreover, since $m$ and $P$ are carried explicitly, the same rule controls $\delta$ after every reduction rather than only after the first one. We give the proof below, decompose $\delta$ in Appendix~\ref{app:logit}, and derive an exact identity for the evidence-weighted contribution withheld by the upper bound $m_{\max}$ in Appendix~\ref{app:clipping}.

\subsection{Perturbation of normalized weights}

We first establish the normalization bound used in Proposition~\ref{prop:attention}. It is a statement about positive measures and does not depend on any particular compression algorithm.

\begin{lemma}
\label{lem:tilt}
Let $p$ be a probability vector, let $r_i\in[a,b]$ with $0<a\leq b$, and let $\widehat p_i=p_ir_i/\sum_jp_jr_j$. Then
\begin{equation}
  \norm{p-\widehat p}_1\leq 2\,\frac{\sqrt b-\sqrt a}{\sqrt b+\sqrt a}.
  \label{eq:tilt}
\end{equation}
\end{lemma}
\begin{proof}
If $a=b$, both sides vanish. Otherwise, let $\mu=\sum_ip_ir_i\in[a,b]$. Since $r\mapsto|r-\mu|$ is convex on $[a,b]$, it lies below its chord, \textit{i.e.,}
\[
  |r-\mu|\leq\frac{b-r}{b-a}(\mu-a)+\frac{r-a}{b-a}(b-\mu).
\]
Taking the expectation under $p$ and dividing by $\mu$ gives
\[
  \norm{p-\widehat p}_1=\frac{\E_p|r-\mu|}{\mu}\leq\frac{2(b-\mu)(\mu-a)}{(b-a)\,\mu}.
\]
The right-hand side is maximized at $\mu=\sqrt{ab}$, where it equals $2(\sqrt b-\sqrt a)^2/(b-a)$, which is exactly Equation~\ref{eq:tilt}.
\end{proof}

\subsection{Proof of Proposition~\ref{prop:attention}}

\begin{proof}
We treat each text token as a singleton group whose weight and value are identical in the reference and compressed computations. Let $p_i=Z_i/Z$ and $\widehat p_i=\widehat Z_i/\widehat Z$, where $Z$ and $\widehat Z$ sum over both visual and text groups. Since the grouping into $(Z_i,\bar v_i)$ in Appendix~\ref{app:setting} is exact, we have
\[
  y-\widehat y=\sum_ip_i(\bar v_i-\widehat v_i)+\sum_i(p_i-\widehat p_i)\widehat v_i.
\]
The first sum vanishes on text groups and has norm at most $\eta$ over all visual groups because $\sum_ip_i\leq1$. The second sum has norm at most $V\norm{p-\widehat p}_1$ by the norm bound on values. With $r_i=\widehat Z_i/Z_i$, we have $\widehat p_i=p_ir_i/\sum_jp_jr_j$, where $r_i\in[e^{-\delta},e^{\delta}]$ on visual groups and $r_i=1$ on text groups. Applying Lemma~\ref{lem:tilt} with $a=e^{-\delta}$ and $b=e^{\delta}$ gives $\norm{p-\widehat p}_1\leq2\tanh(\delta/2)$, which completes the proof.
\end{proof}

The proposition is a local statement for one head with a shared query and shared text keys and values. Restricting the query to a text token that follows the image ensures that every member of a visual group is causally visible.

\subsection{Decomposition of the contribution discrepancy}
\label{app:logit}

Write the reference weight of token $j$ in group $C_i$ as $a_j=w_j\exp(s_j)$, where $w_j>0$ and $s_j$ is the query--key logit including the position encoding, and let $M_i=\sum_{j\in C_i}w_j$. Suppose a reference logit $s_i$ satisfies $|s_j-s_i|\leq\xi_i$ for all $j\in C_i$, and the compressed group uses $\widehat Z_i=\mu_i\exp(\widehat s_i)$. Then $Z_i$ lies between $M_i\exp(s_i-\xi_i)$ and $M_i\exp(s_i+\xi_i)$, and the triangle inequality yields
\begin{equation}
  \big|\log\widehat Z_i-\log Z_i\big|\leq\xi_i+|\widehat s_i-s_i|+|\log\mu_i-\log M_i|.
  \label{eq:logit-decomposition}
\end{equation}
The three terms correspond to the within-group logit variation, the perturbation of the representative logit, and the mismatch of the effective contribution, respectively. Note that the weights $w_j$ may themselves depend on the fixed query. Therefore, Equation~\ref{eq:logit-decomposition} does not require a query-independent exact group mass. Let $\mathcal N$ denote the input normalization of the decoder and $W_k$ its key projection, and keep the original RoPE index of the representative. The orthogonality of the rotation gives
\begin{equation}
  |\widehat s_i-s_i|\leq\frac{\norm{q}\norm{W_k}}{\sqrt{d_h}}\,\norm{\mathcal N(h'_i)-\mathcal N(h_i)}.
  \label{eq:key-perturbation}
\end{equation}
Therefore, a content update affects both terms of Proposition~\ref{prop:attention}, even when it is intended to improve the representative values. Renumbering the position would introduce a further rotation term. In \method, the multiplicative factor entering this analysis is $\mu_i=\exp(B_{\ell h}(q,i))$, which combines the recorded weight $m_i$ with the per-head gain and the distance factor. Hence, the bound applies to the complete bias of Equation~\ref{eq:attention}.

\subsection{Conservation under the upper bound}
\label{app:clipping}

The contribution update in Equation~\ref{eq:mass} admits an exact characterization of the contribution withheld by the upper bound. When the scaling denominator is nonzero and the inherited weights lie in $[1,m_{\max}]$, the evidence-weighted contribution withheld by the bound is
\begin{equation}
  \sum_{j\in A\cup D} e_jm_j-\sum_{i\in A}e_im'_i
  =\sum_{i\in A}e_i\,[m_i+\lambda b_i-m_{\max}]_+ .
  \label{eq:clipping}
\end{equation}
Only the representatives that reach the bound contribute to the right-hand side. Therefore, the withheld contribution is confined to the largest groups, and the identity holds unchanged when the input weights are inherited from an earlier stage. To see this, let $Q_D=\sum_{j\in D}e_jm_j$ and $Q_A=\sum_{i\in A}e_im_i$. When the denominator of $\lambda$ exceeds the numerical threshold, we have $\lambda\sum_{i\in A}e_ib_i=Q_D$ and hence $\sum_{i\in A}e_i(m_i+\lambda b_i)=Q_A+Q_D=\sum_{j\in A\cup D}e_jm_j$. Since the evidence and the weights are non-negative and $m_i\geq1$, the lower bound of the projection is inactive and $m'_i=\min(m_i+\lambda b_i,m_{\max})$. Subtracting the bounded sum gives Equation~\ref{eq:clipping} in exact arithmetic. This identity concerns the evidence-weighted contribution, whereas Proposition~\ref{prop:attention} covers the attention output, which additionally depends on the query-dependent logits and on the representative values. The implementation uses a denominator threshold of $10^{-8}$, below which $\lambda=0$. At a later reduction, Equation~\ref{eq:clipping} holds with $m$ the inherited weights and $e$ the evidence of that stage as defined in Section~\ref{sec:update}, so the conservation holds within each stage under its own evidence.

\section{Implementation details}
\label{app:implementation}

\begin{algorithm}[htbp]
\caption{One reduction in \method: selection followed by composition}
\label{alg:state}
\begin{algorithmic}[1]
\Require state $S=(H,P,m,G)$, target count $K$, and the signals of the stage: encoder signals and instruction (initial stage) or instruction states at the current depth (internal stages)
\State \textit{Selection:} obtain the representative set $A$ with $|A|=K$ and the evidence $e$ by CSI (initial stage), by the internal selection (internal stages), or by ordered spatial subsampling with uniform $e$ (final stage); sort $A$ by original position; $D\gets\{1,\ldots,|H|\}\setminus A$
\State \textit{Composition:} assign each $j\in D$ to $a(j)$ by Equation~\ref{eq:correspondence} using the similarities in $G$
\State Compose the contributions $m'_A$ by Equation~\ref{eq:mass} and the content $H'_A$ by Equation~\ref{eq:content}, using $m$, $H$, the positions in $P$, and $e$
\State Restrict the positions, similarities, and anchors to $A$; return $S'=\mathcal C_K(S;A,e)=(H'_A,P_A,m'_A,G_A)$
\State Run the decoder blocks with SCA (Equation~\ref{eq:attention}); at the next reduction, apply the same steps to the current state
\end{algorithmic}
\end{algorithm}

\begin{table}[htbp]
\caption{Constants of the main configuration, shared by both execution strategies and by all four models.}
\label{tab:constants}\centering\small
\setlength{\tabcolsep}{4pt}
\begin{tabular}{@{}l>{\raggedright\arraybackslash}p{0.40\linewidth}>{\raggedright\arraybackslash}p{0.30\linewidth}@{}}\toprule
Symbol & Value & Role\\\midrule
$\beta$ & 1.5 & evidence exponent, Equation~\ref{eq:coverage-weights}\\
$u$ & 0.40 / 0.30 / 0.25 at $K=64$ / 128 / 192, concentration-adaptive, at most 0.60 & uniform component\\
$\lambda_{\max}$, $\sigma_0$ & 0.30, 0.10 & instruction residual bound and selectivity scale\\
$\lambda_{\max}$ (first-layer route) & 0.15, without instruction anchors & used when no CLIP text encoder is available\\
$N_{\mathrm r}$ & 16 regions with a 10\% uniform regional weight & region balancing\\
core fraction & $0.25+0.45\sqrt{c}$ of the budget & protected attention peaks\\
anchors & at most 25\% of the budget and at most 24 anchors & instruction anchors\\
$\rho_{\mathrm r}$ & 5\% & coverage-gap reserve\\
$\gamma$, $m_{\max}$ & 2, 8 & contribution transfer, Equation~\ref{eq:mass}\\
$\alpha$, $\tau$ & 0.15, 0.75 & content composition, Equation~\ref{eq:content}\\
$g_{\min}$, $g_{\max}$ & 0.6, 1.4 & head-gain range, Equation~\ref{eq:head-gain}\\
$u_\ell$ & 1 for \fixed, and for \pro\ 0.8 before the first internal reduction and 1 thereafter & strategy gain\\
stages & layers $\mathrm{round}(14L/32)$ and $\mathrm{round}(24L/32)$ of an $L$-layer decoder, with $K_1=\mathrm{round}(1.5K)$ and $K_3=\mathrm{round}(0.25K)$ & progressive schedule\\
internal quotas & 10\% coordinate, 12.5\% spatial, 5\% reserve & internal selection\\
\bottomrule\end{tabular}
\end{table}

Algorithm~\ref{alg:state} summarizes the shared construction of both execution strategies, and Table~\ref{tab:constants} lists every constant of the main configuration. In the formulas of this appendix, $K$ denotes the number of representatives selected by the stage at hand, which equals the budget under \fixed\ and $K_1=\operatorname{round}(1.5K)$ at the initial stage of \pro. The same values are used by both execution strategies and by all four models, and only the head gains in Equation~\ref{eq:head-gain} are recomputed from the projections of each model. When no CLIP text encoder is available, as in Qwen2.5-VL, the instruction distribution is derived from the first decoder layer. This route uses the fixed bound $\lambda_{\max}=0.15$ and no instruction anchors (Table~\ref{tab:constants}). All similarity features are normalized row-wise with a small positive constant in the denominator. $\MassNorm$ divides a non-negative vector by its sum and returns the uniform vector when the sum is zero, and $\operatorname{MinMax}$ subtracts the minimum and divides by the range, with the same numerical constant.

\subsection{CSI: visual and instruction evidence}

\paragraph{Visual evidence.}
We first aggregate the CLIP class-token attention with entropy-dependent head weights and truncate it at the 99.5th percentile. Then, we smooth it with two steps of diffusion (coefficient $0.35$, similarity temperature $0.1$) over a graph that links each token to its eight nearest neighbors in the CLIP key space, followed by a $3\times3$ spatial kernel. Meanwhile, the feature distinctiveness combines a global term, $1-x_j^\top\operatorname{Normalize}(\sum_ix_i)$ for normalized projected features $x_j$, with a local term, one minus the mean similarity to the four spatially adjacent features. The two min--max-normalized terms are mixed with coefficients $0.45$ and $0.55$, respectively. Let $a$ and $s$ be the min--max-normalized smoothed attention and distinctiveness. The base evidence is
\begin{equation}
  b=\MassNorm\big(\rho a+(1-\rho)s\big),\qquad
  \rho=\frac{\operatorname{Var}(a)}{\operatorname{Var}(a)+\operatorname{Var}(s)},
  \label{eq:base-evidence}
\end{equation}
with a numerically stabilized denominator.

\paragraph{Instruction evidence.}
Following the entropy-filtered instruction scoring of EADP \citep{wang2026eadp}, the frozen CLIP text encoder processes the instruction within its native 77-token limit. The beginning-of-sequence token is excluded from the token-level aggregation and the end-of-sequence token is retained. Let $u_{tj}$ be the cosine similarity between the projected text token $t$ and the image patch $j$. For each text token, $\pi_{tj}=\softmax_j(u_{tj}/0.01)$ has entropy $H_t$. Then, we keep the half of the text tokens with the lowest entropy (rounding upward) and weight them by $\beta_t\propto\exp(-H_t/0.01)$. Specifically, the implementation uses the negative-similarity convention also described by AnchorPrune \citep{oh2026anchorprune}, \textit{i.e.,} $r^{\mathrm{dense}}_j=-\sum_t\beta_tu_{tj}$ and $r^{\mathrm{eos}}_j=-u_{\mathrm{eos},j}$, and mixes the separately min--max-normalized dense and end-of-sequence scores equally before mass normalization. Finally, the resulting distribution $\pi$ enters Equation~\ref{eq:query-evidence} with $\lambda_{\max}=0.30$ and $\sigma_0=0.10$, and a zero residual sets $\lambda_q=0$. The mixture bounds the instruction-dependent adjustment and keeps the evidence normalized.

\subsection{CSI: coverage weights, anchors, and reserve}

The base uniform coefficient $u$ in Equation~\ref{eq:coverage-weights} is $0.40$, $0.30$, and $0.25$ for the budgets 64, 128, and 192, respectively, with linear interpolation between these budgets and $0.25$ above 192. It is further adapted to the attention concentration: with $c=1-H(a_{\mathrm{mass}})/\log N$ for the mass-normalized encoder attention and a running mean $\bar c$ (decay $0.98$, after a 16-sample warm-up), the base coefficient is multiplied by $\clip(1+(\bar c-c)/\max(\bar c,\epsilon),0.5,1.5)$ and bounded by $0.60$. Since this statistic is updated online, exact reproduction uses the same initialization and sample order.

For region balancing, we first concatenate the scaled normalized encoder features, projected features, and centered spatial descriptors with squared scale weights $0.425$, $0.425$, and $0.15$, respectively. Then, we obtain $N_{\mathrm r}=16$ region centers by farthest-point initialization followed by nearest-center allocation. Finally, with $r_g=\sum_{j\in g}w_j$ being the mass of region $g$ and $N_{\mathrm a}$ the number of active regions, each region receives the weight $0.9r_g+0.1/N_{\mathrm a}$, which is distributed within the region proportionally to $w_j$ and renormalized to give $\omega$.

The protected attention core takes a budget fraction of $0.25+0.45\sqrt c$. Following the relevance-anchor construction of \citet{oh2026anchorprune}, instruction anchors extend this core when the instruction selectivity is at least $0.05$. The $\max(2,\lfloor K/12\rfloor)$ most relevant tokens form the initial anchor set, which grows along the relevance order until three candidates whose novelty with respect to the initial set exceeds $0.2$ have been included, subject to at most $25\%$ of the budget and at most 24 anchors. Candidates whose similarity to a retained anchor reaches $0.95$ are discarded. The protected set leaves at least $\max(2,\lfloor K/8\rfloor)$ places for the greedy coverage and the reserve. Then, greedy coverage fills the remaining non-reserved positions by Equation~\ref{eq:coverage}. Finally, the reserve of $\rho_{\mathrm r}=5\%$ repeatedly selects the token with the largest current gap $1-\max_{i\in A}\kappa_{ij}$ and updates the coverage after each choice. All returned indices are sorted in the original spatial order.

\subsection{CSC: confidence}

For a representative $i$ and a removed token $j$, the affinity used for the confidence is
\begin{equation}
  t_{ij}=s_{ij}+0.45\exp\!\Big(-\frac{\norm{p_i^{2D}-p_j^{2D}}^2}{2(0.2)^2}\Big),
  \label{eq:affinity}
\end{equation}
where $p^{2D}$ denotes the normalized image coordinates derived from the original positions in $P$. Let $k_f=\min(K,\max(2,\lceil0.1K\rceil))$ and let $\varpi_j$ be the softmax at temperature $0.07$ of the $k_f$ largest affinities of token $j$. Its confidence is
\begin{equation}
  q_j=\operatorname{sigmoid}\!\Big(\frac{\max_it_{ij}-0.4}{0.1}\Big)\cdot
      \frac{1+\clip\big(1-H(\varpi_j)/\log k_f,0,1\big)}{2}.
  \label{eq:confidence}
\end{equation}
Note that the correspondence $a(j)$ is still chosen by the semantic similarity $s_{ij}$, and Equation~\ref{eq:affinity} controls the confidence only. Equation~\ref{eq:content} uses $\tilde e_j=\operatorname{MinMax}(e)_j+10^{-6}$, $\alpha=0.15$, and $\tau=0.75$, and its centroid uses the current hidden features without retrieving the features removed at an earlier stage. Equation~\ref{eq:mass} uses $\gamma=2$ and $m_{\max}=8$.

\subsection{SCA: head gains and positions}

Let $W_q^{\ell h}$ and $W_k^{\ell h}$ be the query and key projections of one head. For the RoPE pair $r$ in the split-half convention,
\[
  E_{\ell hr}=\sqrt{\norm{W_{q,r}^{\ell h}}_2^2+\norm{W_{q,r+d_h/2}^{\ell h}}_2^2}\;
              \sqrt{\norm{W_{k,r}^{\ell h}}_2^2+\norm{W_{k,r+d_h/2}^{\ell h}}_2^2},
\]
where $W_{\cdot,r}$ denotes the row of the projection that produces coordinate $r$. Equation~\ref{eq:head-gain} maps $\nu_{\ell h}$ to $[g_{\min},g_{\max}]=[0.6,1.4]$ by min--max normalization over all decoder layers and heads with a numerical constant of $10^{-8}$. Therefore, a larger frequency emphasis receives a larger gain. This construction measures the relative frequency emphasis of the projection matrices and requires no attention-head labels.

The visual representatives keep their original indices in the image span, and the subsequent text positions follow that original span. Both RoPE and the distance factor in Equation~\ref{eq:attention} use these persistent indices, and the causal attention mask is preserved. The frequencies $\theta_r$ are those of the host model. For Qwen2.5-VL, the base is $10^6$ and the distance factor is evaluated on each axis of its multidimensional RoPE (Appendix~\ref{app:cross-model-implementation}). During decoding, the visual-key weights remain associated with the cached visual tokens of the corresponding layer.

\subsection{Progressive composition and internal selection}

The metadata $G$ contains the initial similarity matrix restricted to the surviving representatives and the surviving instruction anchors. Meanwhile, the original positions in $P$ provide a shared indexing convention across stages, from which the normalized coordinates used by the spatial affinity of Equation~\ref{eq:affinity} and by the coordinate-based component are derived. At the first internal reduction, the normalized hidden states of the decoder give the query of the last instruction token and the current visual keys. We first apply the distance correction to their RoPE logits with all representative weights set to one, normalize the attention over the visual keys, and aggregate the heads by their departure from uniform visual attention. Then, a second score averages, over the instruction tokens after the image, the softmax-normalized scaled inner products between their value projections and those of the visual representatives. Finally, the internal selection averages the two min--max-normalized scores with equal weight, and this finite score, taken before the protections below, is the evidence $e$ of the stage in Equations~\ref{eq:mass} and~\ref{eq:content}. The surviving anchors remain protected. Moreover, a coordinate-based component following IVC-Prune \citep{sun2026ivc}, stratified spatial selection, and a coverage-gap reserve using the inherited similarity matrix receive quotas of $10\%$, $12.5\%$, and $5\%$, respectively. The final reduction uses ordered spatial subsampling with the same type of reserve and uniform evidence. Each reduction then performs the composition in Algorithm~\ref{alg:state}, which consults only the evidence $e$ of the stage, the inherited similarities in $G$, the inherited weights $m$, the positions $P$ together with the coordinates derived from them, and the hidden states of the representatives at the current depth. The anchors in $G$ serve the internal selection alone, and the instruction states reach the composition only through $A$ and $e$. The model attention always receives the effective weights recorded in the state, and the unit-weight computation is confined to the internal selection.

\section{Schedules and computational accounting}
\label{app:schedule}

\begin{table}[t]
\caption{Token schedules for the 32 decoder layers. The last column gives the exact layer average.}
\label{tab:schedule}\centering
\begin{tabular}{@{}lccccc@{}}\toprule
Method & Target $K$ & Layers 0--13 & Layers 14--23 & Layers 24--31 & Mean\\\midrule
\pro & 192 & 288 & 173 & 48 & 192.0625\\
\pro & 128 & 192 & 115 & 32 & 127.9375\\
\pro & 64 & 96 & 58 & 16 & 64.1250\\
\fixed & $K$ & $K$ & $K$ & $K$ & $K$\\
\bottomrule\end{tabular}
\end{table}

For a target average count $K$, the progressive schedule sets $K_1=\operatorname{round}(1.5K)$, $K_3=\operatorname{round}(0.25K)$, and $K_2=\operatorname{round}\big((32K-14K_1-8K_3)/10\big)$, which gives the counts in Table~\ref{tab:schedule}. On the 40 decoder layers of LLaVA-1.5-13B, the two reductions take place at layers 18 and 30, and $K_2=\operatorname{round}\big((40K-18K_1-10K_3)/12\big)$ gives 168, 112, and 56 tokens, and the layer averages equal 192, 128, and 64 exactly. LLaVA-NeXT-7B also has 32 decoder layers and gives 960/576/160, 480/288/80, and 240/144/40 content tokens at 640, 320, and 160, respectively. On the 28 layers of Qwen2.5-VL-7B, the reductions take place at layers 12 and 21, which gives 768/469/128, 384/235/64, and 192/117/32 tokens at 512, 256, and 128, respectively, with layer averages of 511.9, 256.1, and 127.9. Note that these counts are determined before inference, whereas the identities of the retained representatives depend on the intermediate computation.

The mean count $\overline K=L^{-1}\sum_\ell K_\ell$ measures the token exposure through the decoder, and the cost of a schedule additionally depends on the distribution of the counts over the layers. With text length $T$, hidden size $d$, feed-forward size $d_{\mathrm{ff}}$, and $n_\ell=T+K_\ell$, a dense matrix-multiplication approximation for a LLaMA-style decoder is
\begin{equation}
  \operatorname{FLOPs}_{\mathrm{decoder}}\approx\sum_{\ell=0}^{L-1}\big(8n_\ell d^2+4n_\ell^2d+6n_\ell d\,d_{\mathrm{ff}}\big),
  \label{eq:flops}
\end{equation}
counting a multiply and an addition as two operations. Equal $\overline K$ approximately aligns the terms linear in the sequence length, but not $\sum_\ell K_\ell^2$. For the three reported schedules, the visual-only quadratic term of \pro\ is about $1.25$ times that of \fixed. The visual encoder, instruction encoder, initial pairwise similarities, internal selection, and bias construction also contribute to the end-to-end cost, and the initial pairwise kernel requires quadratic storage in the original patch count. These operations are included in the measured generation times of Table~\ref{tab:efficiency}, whereas the decoder estimates count the decoder matrix multiplications only.

\section{Evaluation protocol and additional results}
\label{app:evaluation}

\subsection{Tasks, metrics, and retention}

\begin{table}[htbp]
\caption{Per-task performance of \method\ on LLaVA-1.5-7B. Retention $R_{\mathrm{PC}}$ (\%) averages six task-specific ratios to the uncompressed reference; MME-PC sums perception and cognition, POPE uses F1, and TextVQA uses the OCR-augmented prompt. $K$ is exact at every layer for \fixed\ and denotes the target layer average for \pro\ (actual means $192.06$, $127.94$, and $64.13$). Aggregates are computed before rounding.}
\label{tab:main}\centering
\setlength{\tabcolsep}{3.2pt}
\begin{tabular}{@{}lcccccccc@{}}\toprule
Method & $K$ & GQA & MMB & MME-PC & POPE & SQA & VQA\tsup{T} & $R_{\mathrm{PC}}$ (\%)\\\midrule
Vanilla & 576 & 62.0 & 64.0 & 1867 & 85.8 & 69.5 & 58.2 & 100.0\\\midrule
\fixed & 192 & 61.3 & 64.3 & 1901 & 86.2 & 69.1 & 57.6 & 100.0\\
\fixed & 128 & 61.3 & 64.1 & 1867 & 86.3 & 69.0 & 57.1 & 99.5\\
\fixed & 64 & 60.2 & 63.1 & 1819 & 85.8 & 68.5 & 56.4 & 98.1\\\midrule
\pro & 192 & 61.6 & 64.5 & 1855 & 86.2 & 69.3 & 57.9 & 99.8\\
\pro & 128 & 61.3 & 64.7 & 1890 & 86.1 & 69.1 & 57.7 & 100.0\\
\pro & 64 & 60.9 & 64.4 & 1863 & 86.0 & 68.4 & 57.0 & 99.2\\
\bottomrule\end{tabular}
\end{table}

The per-task scores of both execution strategies on LLaVA-1.5-7B at all three budgets are shown in Table~\ref{tab:main}. Specifically, GQA uses the exact-match accuracy, and MMBench uses the English development split and the answer-selection accuracy of the evaluator. MME reports the sum of the perception and cognition scores, or the perception score alone for $R_{\mathrm P}$. POPE reports the F1 over the combined predictions of the three splits, which is distinct from the mean of the split-wise F1 scores. ScienceQA uses the image-containing subset, and TextVQA uses the OCR-augmented prompt and its standard metric. The retention averages the six ratios with equal weight. LLaVA-1.5-13B uses the same tasks, protocol, and constants, with the head gains recomputed from its projections.

\subsection{Cross-model implementation}
\label{app:cross-model-implementation}

\begin{table}[p]\centering
\caption{Cross-model performance of \fixed. Each model has its own uncompressed reference. $R_{\mathrm{PC}}$ uses the same six tasks as Table~\ref{tab:main}, including OCR-augmented TextVQA. For LLaVA-NeXT, $K$ counts content tokens and $N$ includes 48 retained newline tokens. For the other models, $K=N$. Aggregates are computed before rounding.}\label{tab:transfer}
\fontsize{8.3}{9.5}\selectfont
\setlength{\tabcolsep}{3pt}\renewcommand{\arraystretch}{1.0}
\begin{tabular*}{\linewidth}{@{\extracolsep{\fill}}lcccccccc@{}}\toprule
Setting & $N$ & GQA & MMB & MME-PC & POPE & SQA & VQA\tsup{T} & $R_{\mathrm{PC}}$ (\%)\\\midrule
\multicolumn{9}{l}{\textbf{LLaVA-1.5-13B}}\\
Vanilla & 576 & 63.3 & 68.7 & 1824 & 86.0 & 72.8 & 61.3 & 100.0\\
$K=192$ ($\downarrow 66.7\%$) & 192 & 62.9 & 68.5 & 1821 & 87.0 & 73.1 & 60.6 & 99.9\\
$K=128$ ($\downarrow 77.8\%$) & 128 & 62.7 & 68.0 & 1811 & 87.1 & 73.7 & 60.1 & 99.7\\
$K=64$ ($\downarrow 88.9\%$) & 64 & 61.4 & 67.1 & 1775 & 86.3 & 73.0 & 59.5 & 98.3\\
\midrule
\multicolumn{9}{l}{\textbf{LLaVA-NeXT-7B}}\\
Vanilla & 2928 & 64.2 & 64.3 & 1817 & 86.9 & 68.0 & 61.1 & 100.0\\
$K=640$ ($\downarrow 77.8\%$) & 688 & 63.5 & 63.6 & 1865 & 87.4 & 67.9 & 60.3 & 99.9\\
$K=320$ ($\downarrow 88.9\%$) & 368 & 62.6 & 62.5 & 1812 & 87.3 & 67.8 & 58.6 & 98.5\\
$K=160$ ($\downarrow 94.4\%$) & 208 & 61.1 & 59.5 & 1734 & 86.7 & 66.7 & 55.8 & 95.4\\
\midrule
\multicolumn{9}{l}{\textbf{Qwen2.5-VL-7B}}\\
Vanilla & 1296 & 59.8 & 83.4 & 2320 & 86.6 & 88.1 & 76.7 & 100.0\\
$K=512$ ($\downarrow 60.5\%$) & 512 & 59.2 & 83.0 & 2308 & 85.5 & 87.6 & 76.7 & 99.3\\
$K=256$ ($\downarrow 80.2\%$) & 256 & 58.9 & 82.2 & 2308 & 84.6 & 86.3 & 75.7 & 98.5\\
$K=128$ ($\downarrow 90.1\%$) & 128 & 57.9 & 80.3 & 2287 & 82.2 & 84.1 & 71.6 & 95.9\\
\bottomrule\end{tabular*}
\medskip
\parbox{\linewidth}{\footnotesize LLaVA-1.5 uses $336\times336$ inputs. LLaVA-NeXT uses $672\times672$ inputs with 2,880 content tokens before compression. Qwen2.5-VL uses $1008\times1008$ inputs with 1,296 tokens after its native merger.}
\label{tab:transfer-end}
\end{table}

Table~\ref{tab:transfer} reports the per-task scores of \fixed\ on the three additional models, and Tables~\ref{tab:cross-13b}, \ref{tab:cross-qwen}, and~\ref{tab:cross-next} include those of \pro. The architecture-specific interfaces are described below.

\paragraph{LLaVA-NeXT.}
We evaluate the Vicuna-7B checkpoint with its native image preparation and conversation template at a fixed $672\times672$ input. The base view and the high-resolution grid contribute 576 and $48\times48$ tokens, respectively, followed by the native insertion of 48 newline markers. Compression takes place after the image-feature packing and before the language decoder. The CLIP evidence and descriptors are aligned with the packed content grid, with the spatial neighborhoods kept separately for the base and high-resolution views. The newline markers keep their original order and positions, are excluded from the representative selection and content composition, and receive no additional bias. Thus, the budget $K$ counts the content representatives, and the decoder receives $K+48$ visual tokens under \fixed\ and $K_1+48$, $K_2+48$, and $K_3+48$ over the three stages of \pro.

\paragraph{Qwen2.5-VL.}
We evaluate Qwen2.5-VL-7B-Instruct at a fixed $1008\times1008$ input, which yields 1,296 tokens after the native visual merger. Its encoder has neither a class token nor a CLIP text space. Therefore, CSI uses the mean incoming attention of the last global visual-attention block as the visual evidence, and computes the instruction relevance from the frozen input normalization and the query and key projections of the first language block. The visual descriptors are restored from the window order and averaged within the native merger groups before selection. The state preserves the original temporal, height, and width position components, and SCA evaluates each rotary frequency on the axis assigned by the multidimensional RoPE configuration of the model, with the native generation position offsets retained. These interfaces supply architecture-specific evidence and geometry to the shared coverage, contribution-transfer, and confidence-gated composition rules.

\paragraph{Configuration and records.}
Both high-resolution models use batch size one, scaled dot-product attention, and lmms-eval version \texttt{v0.6-80-g2c08ee55}, and their full and compressed configurations use the same fixed input resolution. Table~\ref{tab:transfer-extra} reports the TextVQA prompt without OCR tokens, the image subset of SEED-Bench \citep{li2024seed}, and MMMU validation \citep{yue2024mmmu}.

\begin{table}[p]\centering
\caption{Additional task scores for the cross-model evaluation. TextVQA here omits OCR tokens from the prompt. SEED-I denotes the image subset. These scores are reported separately from the six-task aggregate in Table~\ref{tab:transfer}.}\label{tab:transfer-extra}
\fontsize{8.5}{10}\selectfont
\setlength{\tabcolsep}{3pt}\renewcommand{\arraystretch}{1.0}
\begin{tabular*}{\linewidth}{@{\extracolsep{\fill}}lcccc@{}}\toprule
Setting & MME-P & VQA\tsup{T} (no OCR) & SEED-I & MMMU\\\midrule
\multicolumn{5}{l}{\textbf{LLaVA-NeXT-7B}}\\
Vanilla & 1496 & 64.6 & 69.9 & 36.3\\
$K=640$ ($\downarrow 77.8\%$) & 1531 & 63.1 & 69.0 & 35.2\\
$K=320$ ($\downarrow 88.9\%$) & 1482 & 59.4 & 67.7 & 35.3\\
$K=160$ ($\downarrow 94.4\%$) & 1449 & 53.0 & 65.2 & 34.8\\
\midrule
\multicolumn{5}{l}{\textbf{Qwen2.5-VL-7B}}\\
Vanilla & 1671 & 82.4 & 77.8 & 48.9\\
$K=512$ ($\downarrow 60.5\%$) & 1671 & 81.5 & 76.7 & 48.4\\
$K=256$ ($\downarrow 80.2\%$) & 1687 & 80.5 & 75.6 & 48.8\\
$K=128$ ($\downarrow 90.1\%$) & 1655 & 74.7 & 73.1 & 47.2\\
\bottomrule\end{tabular*}
\label{tab:transfer-extra-end}
\end{table}

\subsection{Controlled interventions}
\label{app:ablation}

\begin{table}[htbp]
\caption{Controlled MME interventions at fixed representative sets. \emph{Content} enables Equation~\ref{eq:content}; \emph{attention} enables the complete visual attention correction. Each budget has its own paired reference run. Components and totals are rounded independently from the saved scores.}
\label{tab:ablation}\centering
\begin{tabular}{@{}cccccc@{}}\toprule
$K$ & Content & Attention & Perception & Cognition & Total\\\midrule
192 & Yes & Yes & 1537 & 364 & 1901\\
 & No & Yes & 1534 & 364 & 1898\\
 & Yes & No & 1485 & 349 & 1834\\
 & No & No & 1483 & 351 & 1834\\
\midrule
288 & Yes & Yes & 1497 & 351 & 1848\\
 & No & Yes & 1496 & 351 & 1847\\
 & Yes & No & 1502 & 342 & 1844\\
 & No & No & 1503 & 341 & 1844\\
\bottomrule\end{tabular}
\end{table}

To isolate the content composition and the attention bias, Table~\ref{tab:ablation} uses a paired intervention protocol on MME under \fixed. Specifically, the four conditions share the selected indices, token order, instruction anchors, prompts, and persistent positions. Disabling the content composition leaves the effective weights unchanged. Meanwhile, disabling the attention bias removes its additional attention contribution while preserving the content update when enabled.

\subsection{Cumulative study and position span}
\label{app:components}

Rows (a)--(e) of Table~\ref{tab:components} place the representatives on a span of $\min(576,\operatorname{round}(4.5K))$ positions, which equals the full span at $K=192$ and $K=128$ and compresses the grid to 288 positions at $K=64$. Row (f) removes this rule and is the main configuration. Then, Table~\ref{tab:span} varies the span alone around the main configuration with all other constants fixed. It can be observed that the full span is the best at 192 and 128 tokens and matches the best setting at 64 tokens within a tenth of a point. Moreover, the compressed spans lose most on MME perception, ScienceQA, and TextVQA. These results confirm that the decoder relies on the original grid geometry.

\begin{table}[htbp]
\caption{Position span around the main configuration (retention, \%). Representatives are placed on a proportionally compressed copy of the original grid spanning 288, 432, or 576 positions. The full span of 576 preserves the original grid geometry and is the main configuration.}
\label{tab:span}\centering\small
\begin{tabular}{@{}lccc@{}}\toprule
Span & $K=192$ & $K=128$ & $K=64$\\\midrule
288 & 98.49 & 98.06 & 97.59\\
432 & 99.11 & 99.24 & 98.17\\
576 (main) & 99.96 & 99.49 & 98.10\\
\bottomrule\end{tabular}
\end{table}

\subsection{State inheritance under \pro}
\label{app:inheritance}

Rows (g) and (h) of Table~\ref{tab:components} share the first stage and the schedule of \pro\ and differ only at the two internal reductions. In row (g), each internal reduction resets the contribution weights of the representatives to one and composes neither contributions nor content. It also renumbers their positions contiguously within the image span and recomputes the similarities from the current hidden states without the inherited anchors. Therefore, it selects from a reduced set of ordinary tokens. In contrast, row (h) is the main configuration, whose reductions compose the inherited state by Algorithm~\ref{alg:state}. At 64 tokens, the 16 representatives of the final stage carry an average contribution weight of $6.86$ in row (h), whereas row (g) attends to them with unit weights.

\subsection{Sensitivity}
\label{app:sensitivity}

\begin{table}[htbp]
\caption{\fixed\ retention (\%) under one-at-a-time changes. Each row uses one configuration across all three budgets. The reference is shared with Table~\ref{tab:main}.}
\label{tab:sensitivity}\centering
\begin{tabular}{@{}lccc@{}}\toprule
Configuration & $K=192$ & $K=128$ & $K=64$\\\midrule
Reference & 100.0 & 99.5 & 98.1\\
Head gains $[0.5,1.5]$ & 99.9 & 99.5 & 97.8\\
Head gains $[0.7,1.3]$ & 99.8 & 99.6 & 98.3\\
Coverage reserve $3\%$ & 99.9 & 99.5 & 98.3\\
Coverage reserve $8\%$ & 99.9 & 99.5 & 97.9\\
Evidence exponent $1.0$ & 99.6 & 99.6 & 98.1\\
Evidence exponent $2.0$ & 99.8 & 99.3 & 98.1\\
Uniform coefficient multiplier $0.75$ & 100.0 & 99.3 & 97.9\\
Uniform coefficient multiplier $1.25$ & 99.8 & 99.5 & 97.9\\
\bottomrule\end{tabular}
\end{table}

\begin{table}[htbp]
\caption{\pro\ retention (\%) for nearby first-stage attention gains. Gains from layer 14 onward remain one. The main configuration uses $0.80$ at all budgets.}
\label{tab:pro-gain}\centering
\begin{tabular}{@{}lccc@{}}\toprule
First-stage gain & $K=192$ & $K=128$ & $K=64$\\\midrule
$0.75$ & 99.9 & 99.9 & 99.2\\
$0.80$ & 99.8 & 100.0 & 99.2\\
$0.85$ & 99.7 & 99.9 & 99.3\\
\bottomrule\end{tabular}
\end{table}

Table~\ref{tab:sensitivity} records one-at-a-time variations of the \fixed\ configuration, and Table~\ref{tab:pro-gain} records nearby first-stage gains of \pro. It can be observed that every variation stays within $0.4$ points of the main configuration at every budget, and the first-stage gains of $0.75$ and $0.85$ stay within $0.15$ points of it. Therefore, the reported results do not depend on budget-specific tuning and one shared configuration is used throughout.

\section{External comparisons}
\label{app:external}

\subsection{Sources, coverage, and normalization}

Table~\ref{tab:source-key} identifies every numerical source of the LLaVA-1.5-7B comparisons, and Table~\ref{tab:external-bases} lists the corresponding full-model references. ToMe, VisionZip, DivPrune, VisPruner \citep{zhang2025vispruner}, HoloV \citep{zou2025holov}, and the RESTORE combinations use the evaluations reported by RESTORE. The second VisionZip report and the SparseVLM, PyramidDrop, and VScan \citep{zhang2026vscan} rows come from VisionTrim. FasterVLM \citep{fastervlm2024}, the earlier version preceding VisPruner, uses the report in PruneSID, FastV uses the report in DART, and TRIM \citep{trim2025} and PruMerge+ \citep{shang2025prumerge} use the reports in CDPruner. The remaining configurations use their own papers, and different reports of the same method are listed separately.

For these comparisons, the retention averages the six task ratios of Table~\ref{tab:main}, and published averages over other task suites are not used. MME-PC and MME-P define separate aggregates. Note that the source-specific references matter. For example, SCoRe reports a ScienceQA reference of $66.8$, whereas most other sources use $69.5$, and the reproduced full-model row of HiDrop uses a POPE reference of $86.8$. Therefore, the tables compare reported results under the protocols of their sources.

\subsection{Complete LLaVA-1.5-7B comparisons}

Tables~\ref{tab:external-fixed-full} and~\ref{tab:external-pro-full} give all six task scores at every budget for every configuration, including ToMe, the earlier FasterVLM version, the duplicate VisionZip report, the perception-only convention, the post-pruning counts, and the sample-average budgets. The venue labels refer to the named method rather than the paper supplying a reproduced result. For a combined method, the label refers to the added compression or correction method, and records without a confirmed archival venue are labeled arXiv.

\begingroup\fontsize{8.5}{10}\selectfont
\setlength{\tabcolsep}{3pt}\renewcommand{\arraystretch}{0.95}
\setlength{\LTcapwidth}{\linewidth}
\begin{longtable}{@{}w{l}{\budgetlabw}@{\hspace{3pt}}lccccccc@{}}
\caption{Pre-decoder compression on LLaVA-1.5-7B. Budget labels give the fixed visual count $K$ entering the decoder. MME-PC includes perception and cognition; MME-P uses perception only. $R$ averages the six ratios to each source's full-model reference. Within each budget and metric block, methods are ordered by increasing unrounded $R$. Bold and underline indicate the two highest distinct displayed values in each performance column; ties share a mark and full-model rows are excluded. VisionZip$^{\mathrm R}$ and VisionZip$^{\mathrm V}$ use RESTORE and VisionTrim reports. Appendix~\ref{app:external} specifies sources and evaluation protocols.}\label{tab:external-fixed-full}\\
\toprule
 & Method & GQA & MMB & MME & POPE & SQA & VQA\tsup{T} & $R$ (\%)\\
\midrule\endfirsthead
\multicolumn{9}{l}{\textit{Table \thetable\ (continued)}}\\
\toprule
 & Method & GQA & MMB & MME & POPE & SQA & VQA\tsup{T} & $R$ (\%)\\
\midrule\endhead
\midrule\multicolumn{9}{r}{\textit{Continued on next page}}\\\endfoot
\bottomrule\endlastfoot
\multicolumn{9}{l}{\textbf{(a) Fixed counts: MME-PC and $R_{\mathrm{PC}}$}}\\*
 & Vanilla, $K=576$ & 62.0 & 64.0 & 1867 & 85.8 & 69.5 & 58.2 & 100.0\\*
\midrule
 & \methodvenue{HoloV}{NeurIPS 2025} & 58.6 & 62.6 & 1779 & 85.0 & 67.3 & 55.8 & 96.5\\*
 & \methodvenue{VisionZip\textsuperscript{R}}{CVPR 2025} & 59.2 & 62.5 & 1749 & 85.2 & 68.7 & 55.8 & 96.7\\*
 & \methodvenue{DivPrune}{CVPR 2025} & 58.9 & 63.1 & 1723 & 86.5 & 69.0 & 55.7 & 96.8\\*
 & \methodvenue{ToMe}{ICLR 2023} & 59.5 & 62.6 & 1727 & 86.9 & 69.0 & 55.8 & 97.0\\*
 & \methodvenue{VisionZip\textsuperscript{V}}{CVPR 2025} & 59.3 & 63.0 & 1783 & 85.3 & 68.9 & 57.3 & 97.6\\*
 & \methodvenue{VisPruner}{ICCV 2025} & 59.4 & 62.5 & 1784 & 86.0 & 68.3 & 57.7 & 97.7\\*
 & \methodvenue{FasterVLM}{arXiv 2024} & 59.3 & 63.5 & 1780 & 85.3 & \underline{70.0} & 57.3 & 98.0\\*
 & \methodvenue{PruneSID}{ICLR 2026} & 60.1 & 63.7 & 1791 & 86.9 & 68.5 & 56.7 & 98.1\\*
 & \methodvenue{MMTok}{ICLR 2026} & 60.1 & 63.4 & 1774 & 86.4 & 68.8 & 57.7 & 98.2\\*
 & \methodvenue{ZOO-Prune}{CVPR 2026} & 60.0 & 62.9 & 1782 & 87.2 & 69.2 & 57.3 & 98.2\\*
 & \methodvenue{SCOPE}{NeurIPS 2025} & 60.1 & 63.6 & 1804 & 86.4 & 68.8 & 57.7 & 98.5\\*
 & \methodvenue{EvoCut}{arXiv 2026} & 60.2 & 64.2 & 1794 & 86.5 & 68.5 & 57.6 & 98.5\\*
 & \methodvenue{FlowCut}{NeurIPS 2025} & 60.1 & 63.2 & 1836 & 86.1 & 68.6 & 57.5 & 98.5\\*
 & \longmethodvenue{RESTORE+VisPruner}{ICML 2026} & 60.9 & 63.3 & 1816 & 86.1 & 69.5 & 57.0 & 98.7\\*
 & \longmethodvenue{RESTORE+HoloV}{ICML 2026} & \underline{61.0} & 63.7 & 1793 & 86.6 & 69.6 & 57.2 & 98.8\\*
 & \methodvenue{SpecFlow}{ICML 2026} & 58.3 & \textbf{65.8} & 1827 & 85.8 & 69.7 & 57.9 & 98.9\\*
 & \methodvenue{DeSAP}{ACM MM 2026} & 60.3 & 63.5 & \underline{1848} & 87.1 & 69.9 & 57.2 & 99.2\\*
 & \methodvenue{CRISP}{ICME 2026} & 60.3 & 63.8 & 1797 & \underline{87.7} & 69.4 & \textbf{59.7} & 99.5\\*
 & \methodvenue{CaRe}{arXiv 2026} & 60.6 & 63.4 & 1809 & \textbf{88.9} & \textbf{70.2} & \underline{58.0} & \underline{99.6}\\*
\covistrow\cellcolor{white}\budgetlabel[8.4pt]{20}{$K{=}192$ ($\downarrow 66.7\%$)} & \textbf{\fixed} & \textbf{61.3} & \underline{64.3} & \textbf{1901} & 86.2 & 69.1 & 57.6 & \textbf{100.0}\\
\midrule
 & \methodvenue{ToMe}{ICLR 2023} & 58.7 & 60.7 & 1668 & 86.5 & 68.4 & 54.8 & 95.3\\*
 & \methodvenue{VisionZip\textsuperscript{R}}{CVPR 2025} & 58.5 & 61.4 & 1705 & 83.2 & 68.8 & 55.6 & 95.4\\*
 & \methodvenue{HoloV}{NeurIPS 2025} & 57.5 & 62.5 & 1761 & 82.2 & 69.0 & 55.6 & 95.8\\*
 & \methodvenue{VisionZip\textsuperscript{V}}{CVPR 2025} & 57.6 & 62.0 & 1762 & 83.2 & 68.9 & 56.8 & 96.2\\*
 & \methodvenue{PruneSID}{ICLR 2026} & 58.8 & 62.1 & 1749 & 86.5 & 68.3 & 54.7 & 96.3\\*
 & \methodvenue{FasterVLM}{arXiv 2024} & 57.8 & 62.5 & 1762 & 82.8 & \textbf{70.0} & 56.3 & 96.4\\*
 & \methodvenue{DivPrune}{CVPR 2025} & 58.6 & 63.7 & 1702 & 86.5 & 68.9 & 55.2 & 96.6\\*
 & \methodvenue{VisPruner}{ICCV 2025} & 58.1 & 61.9 & 1778 & 84.5 & 68.8 & 56.9 & 96.7\\*
 & \methodvenue{FlowCut}{NeurIPS 2025} & 58.5 & 62.1 & 1792 & 85.2 & 68.6 & 57.3 & 97.2\\*
 & \methodvenue{EvoCut}{arXiv 2026} & 59.2 & 62.3 & 1790 & 85.7 & 68.7 & 57.1 & 97.5\\*
 & \methodvenue{MMTok}{ICLR 2026} & 59.3 & 62.3 & 1779 & 86.3 & 68.8 & 57.0 & 97.5\\*
 & \methodvenue{SCOPE}{NeurIPS 2025} & 59.7 & 62.5 & 1776 & 86.1 & 68.4 & 57.2 & 97.6\\*
 & \methodvenue{SpecFlow}{ICML 2026} & 57.6 & \textbf{64.3} & 1794 & 84.9 & \underline{69.9} & 56.8 & 97.6\\*
 & \methodvenue{ZOO-Prune}{CVPR 2026} & 59.5 & 61.9 & 1752 & \underline{87.1} & 68.9 & 57.9 & 97.6\\*
 & \methodvenue{CaRe}{arXiv 2026} & 60.2 & 62.1 & 1731 & \textbf{87.8} & \underline{69.9} & \underline{58.0} & 98.1\\*
 & \longmethodvenue{RESTORE+HoloV}{ICML 2026} & 60.8 & 63.0 & 1807 & 86.0 & 68.7 & 56.6 & 98.2\\*
 & \longmethodvenue{RESTORE+VisPruner}{ICML 2026} & \underline{60.9} & 63.3 & 1813 & 86.2 & 68.7 & 56.4 & 98.3\\*
 & \methodvenue{DeSAP}{ACM MM 2026} & 59.5 & 62.8 & \underline{1815} & 87.0 & \underline{69.9} & 57.0 & 98.4\\*
 & \methodvenue{CRISP}{ICME 2026} & 59.7 & 62.8 & 1794 & \textbf{87.8} & 69.8 & \textbf{58.5} & \underline{98.8}\\*
\covistrow\cellcolor{white}\budgetlabel[8.4pt]{20}{$K{=}128$ ($\downarrow 77.8\%$)} & \textbf{\fixed} & \textbf{61.3} & \underline{64.1} & \textbf{1867} & 86.3 & 69.0 & 57.1 & \textbf{99.5}\\
\midrule
 & \methodvenue{ToMe}{ICLR 2023} & 56.0 & 57.9 & 1588 & 84.3 & 68.0 & 52.6 & 92.0\\*
 & \methodvenue{HoloV}{NeurIPS 2025} & 55.1 & 60.0 & 1699 & 76.8 & 68.6 & 54.9 & 92.6\\*
 & \methodvenue{VisionZip\textsuperscript{V}}{CVPR 2025} & 55.1 & 60.1 & 1690 & 77.0 & 69.0 & 55.5 & 92.8\\*
 & \methodvenue{FasterVLM}{arXiv 2024} & 55.0 & 60.6 & 1667 & 76.6 & \textbf{70.2} & 55.3 & 92.9\\*
 & \methodvenue{VisionZip\textsuperscript{R}}{CVPR 2025} & 56.0 & 60.5 & 1690 & 78.2 & 69.5 & 53.8 & 93.1\\*
 & \methodvenue{VisPruner}{ICCV 2025} & 55.8 & 60.0 & 1670 & 80.5 & 68.5 & 55.6 & 93.4\\*
 & \methodvenue{PruneSID}{ICLR 2026} & 57.1 & 58.8 & 1733 & 83.8 & 67.8 & 54.2 & 94.1\\*
 & \methodvenue{DivPrune}{CVPR 2025} & 57.1 & 60.2 & 1653 & 85.3 & 68.3 & 54.5 & 94.2\\*
 & \methodvenue{FlowCut}{NeurIPS 2025} & 55.6 & 60.8 & 1744 & 80.2 & 69.1 & 55.6 & 94.3\\*
 & \methodvenue{SpecFlow}{ICML 2026} & 55.3 & \textbf{63.7} & 1713 & 80.5 & 69.7 & 54.9 & 94.7\\*
 & \methodvenue{EvoCut}{arXiv 2026} & 56.6 & 61.8 & 1692 & 83.9 & 68.8 & 55.7 & 95.0\\*
 & \methodvenue{ZOO-Prune}{CVPR 2026} & 58.5 & 60.2 & 1676 & 85.9 & 68.3 & 55.4 & 95.1\\*
 & \methodvenue{SCOPE}{NeurIPS 2025} & 58.3 & 61.7 & 1698 & 83.9 & 68.6 & \underline{56.6} & 95.7\\*
 & \longmethodvenue{RESTORE+VisPruner}{ICML 2026} & \underline{59.2} & 61.6 & 1722 & 85.1 & 68.0 & 55.4 & 95.9\\*
 & \methodvenue{MMTok}{ICLR 2026} & 58.3 & 61.2 & 1715 & 85.8 & 69.2 & 56.0 & 96.1\\*
 & \methodvenue{CaRe}{arXiv 2026} & 59.1 & 60.6 & 1665 & \underline{87.0} & \underline{69.8} & 56.3 & 96.2\\*
 & \longmethodvenue{RESTORE+HoloV}{ICML 2026} & 59.0 & 61.9 & \underline{1787} & 84.9 & 68.0 & 55.4 & 96.5\\*
 & \methodvenue{DeSAP}{ACM MM 2026} & 58.3 & 62.3 & 1748 & 85.8 & 69.6 & 55.7 & 96.7\\*
 & \methodvenue{CRISP}{ICME 2026} & 58.3 & 60.7 & 1721 & \textbf{87.2} & 68.7 & \textbf{58.3} & \underline{96.8}\\*
\covistrow\cellcolor{white}\budgetlabel[8.4pt]{20}{$K{=}64$ ($\downarrow 88.9\%$)} & \textbf{\fixed} & \textbf{60.2} & \underline{63.1} & \textbf{1819} & 85.8 & 68.5 & 56.4 & \textbf{98.1}\\
\midrule
\multicolumn{9}{l}{\textbf{(b) Fixed counts: MME-P and $R_{\mathrm{P}}$}}\\*
 & Vanilla, $K=576$ & 62.0 & 64.0 & 1511 & 85.8 & 69.5 & 58.2 & 100.0\\*
\midrule
 & \methodvenue{PruMerge+}{ICCV 2025} & 58.2 & 61.8 & 1408 & 83.1 & 69.1 & 54.0 & 95.3\\*
 & \methodvenue{TRIM}{COLING 2025} & 58.4 & 63.0 & 1413 & 85.3 & 68.6 & 52.2 & 95.5\\*
 & \longmethodvenue{Cen-Prune+DivPrune}{arXiv 2026} & 59.3 & 62.5 & 1416 & 86.8 & 68.6 & \textbf{57.2} & 97.3\\*
 & \methodvenue{SPARE}{arXiv 2026} & 60.0 & 62.5 & 1416 & \underline{87.5} & 68.6 & 56.6 & 97.6\\*
 & \longmethodvenue{Cen-Prune+ZOO-Prune}{arXiv 2026} & 59.7 & 62.5 & 1445 & 86.5 & 68.9 & 57.0 & 97.7\\*
 & \methodvenue{EADP}{ECCV 2026} & 60.0 & 62.7 & 1439 & 87.2 & 69.0 & 56.5 & 97.8\\*
 & \longmethodvenue{Cen-Prune+CDPruner}{arXiv 2026} & \underline{60.2} & 63.1 & 1442 & 86.7 & 68.2 & 57.0 & 97.8\\*
 & \methodvenue{CDPruner}{NeurIPS 2025} & 59.9 & 63.1 & 1431 & \textbf{87.7} & 69.0 & 56.2 & 97.9\\*
 & \methodvenue{AnchorPrune}{ECCV 2026} & 59.8 & \underline{63.3} & 1454 & 86.8 & 68.9 & 57.0 & 98.3\\*
 & \methodvenue{CoverPruner}{EMNLP 2026} & 59.1 & 62.9 & \underline{1465} & \underline{87.5} & \textbf{70.4} & 56.9 & 98.5\\*
 & \methodvenue{SCoRe}{CVPR 2026} & 59.3 & 62.6 & 1451 & 87.3 & \underline{70.1} & \textbf{57.2} & \underline{99.0}\\*
\covistrow\cellcolor{white}\budgetlabel[12.6pt]{12}{$K{=}128$ ($\downarrow 77.8\%$)} & \textbf{\fixed} & \textbf{61.3} & \textbf{64.1} & \textbf{1524} & 86.3 & 69.0 & \underline{57.1} & \textbf{99.6}\\
\midrule
 & \methodvenue{PruMerge+}{ICCV 2025} & 55.4 & 59.6 & 1317 & 75.7 & \underline{69.5} & 52.0 & 91.1\\*
 & \methodvenue{TRIM}{COLING 2025} & 56.6 & 60.9 & 1351 & 85.9 & 69.0 & 49.7 & 93.3\\*
 & \longmethodvenue{Cen-Prune+ZOO-Prune}{arXiv 2026} & 58.2 & 60.7 & 1390 & 85.3 & 67.8 & 56.2 & 95.5\\*
 & \longmethodvenue{Cen-Prune+DivPrune}{arXiv 2026} & 58.4 & 60.5 & 1370 & 85.3 & 68.5 & \textbf{56.6} & 95.6\\*
 & \methodvenue{SCoRe}{CVPR 2026} & 56.9 & 60.2 & 1414 & 83.1 & \textbf{69.9} & \underline{56.5} & 96.2\\*
 & \methodvenue{CDPruner}{NeurIPS 2025} & 58.6 & 61.1 & 1415 & \textbf{87.5} & 68.1 & 55.3 & 96.3\\*
 & \methodvenue{EADP}{ECCV 2026} & \underline{59.4} & \underline{61.9} & 1404 & 86.5 & 68.9 & 55.0 & 96.4\\*
 & \methodvenue{CoverPruner}{EMNLP 2026} & 58.4 & 60.8 & \underline{1432} & 86.1 & 69.3 & 55.6 & 96.5\\*
 & \methodvenue{SPARE}{arXiv 2026} & 59.1 & 61.5 & 1400 & \underline{87.3} & 68.1 & 55.8 & 96.5\\*
 & \methodvenue{AnchorPrune}{ECCV 2026} & 58.6 & 61.3 & 1420 & 85.8 & 68.8 & 56.1 & 96.6\\*
 & \longmethodvenue{Cen-Prune+CDPruner}{arXiv 2026} & 59.0 & \underline{61.9} & 1414 & 86.1 & 68.8 & 56.3 & \underline{96.7}\\*
\covistrow\cellcolor{white}\budgetlabel[12.6pt]{12}{$K{=}64$ ($\downarrow 88.9\%$)} & \textbf{\fixed} & \textbf{60.2} & \textbf{63.1} & \textbf{1495} & 85.8 & 68.5 & 56.4 & \textbf{98.3}\\
\end{longtable}\endgroup

\begingroup\fontsize{8.5}{10}\selectfont
\setlength{\tabcolsep}{3pt}\renewcommand{\arraystretch}{0.95}
\setlength{\LTcapwidth}{\linewidth}
\begin{longtable}{@{}lccccccc@{}}
\caption{Compression during decoder prefill on LLaVA-1.5-7B. Panels (a,b,d) use MME-PC and $R_{\mathrm{PC}}$; (c) uses MME-P and $R_{\mathrm P}$. In (a,c), budget headings give reported layer averages $K$; $^\ddagger$ marks a budget as labeled by its source. Panel (b) uses post-pruning counts except SIEVE$^\ddagger$; (d) uses sample-average final counts. $^{\mathrm T}$ denotes training. Sorting and bold/underline marks follow Table~\ref{tab:external-fixed-full}; the single-method settings in (d) are unranked.}\label{tab:external-pro-full}\\
\toprule
Method & GQA & MMB & MME & POPE & SQA & VQA\tsup{T} & $R$ (\%)\\
\midrule\endfirsthead
\multicolumn{8}{l}{\textit{Table \thetable\ (continued)}}\\
\toprule
Method & GQA & MMB & MME & POPE & SQA & VQA\tsup{T} & $R$ (\%)\\
\midrule\endhead
\midrule\multicolumn{8}{r}{\textit{Continued on next page}}\\\endfoot
\bottomrule\endlastfoot
\multicolumn{8}{l}{\textbf{(a) Layer schedules: MME-PC and $R_{\mathrm{PC}}$}}\\*
Vanilla, $K=576$ & 62.0 & 64.0 & 1867 & 85.8 & 69.5 & 58.2 & 100.0\\*\midrule
\multicolumn{8}{c}{\textit{$K=192$ ($\downarrow 66.7\%$)}}\\*\midrule
\methodvenue{SparseVLM}{ICML 2025} & 57.6 & 62.5 & 1721 & 83.6 & 69.1 & 56.1 & 95.9\\*
\methodvenue{PyramidDrop}{CVPR 2025} & 57.3 & 63.3 & 1797 & 84.8 & 69.2 & 56.5 & 97.0\\*
\methodvenue{ApET$^{\ddagger}$}{CVPR 2026} & 60.2 & 63.4 & 1808 & 86.3 & 68.5 & 54.4 & 97.5\\*
\methodvenue{V2Drop}{CVPR 2026} & 58.5 & 63.7 & \underline{1826} & 85.1 & 69.3 & 55.6 & 97.6\\*
\longmethodvenue{DivPrune+Random$^{\ddagger}$}{CVPR 2026} & 60.1 & 63.2 & 1777 & \textbf{87.1} & 69.3 & 55.9 & 97.9\\*
\methodvenue{E2S-Pruner}{arXiv 2026} & 58.6 & 64.4 & 1814 & 83.6 & \underline{69.8} & 57.2 & 98.0\\*
\methodvenue{DART+Random$^{\ddagger}$}{CVPR 2026} & 59.6 & 63.9 & 1816 & 84.4 & 69.2 & 56.8 & 98.0\\*
\methodvenue{VScan}{TMLR 2026} & 60.6 & 63.9 & 1806 & 86.2 & 68.6 & 57.7 & 98.6\\*
\methodvenue{LRCP$^{\ddagger}$}{arXiv 2026} & 60.3 & 63.7 & 1823 & 86.4 & 68.8 & 57.4 & 98.7\\*
\methodvenue{ProViP$^{\ddagger}$}{arXiv 2026} & 60.9 & \textbf{64.8} & 1824 & 85.5 & 68.7 & 57.4 & 98.9\\*
\methodvenue{VisionTrim$^{\ddagger}$}{ICLR 2026} & \underline{61.0} & 64.4 & 1798 & \underline{86.8} & \textbf{70.8} & \textbf{58.4} & \underline{99.7}\\*
\covistrow \textbf{\pro} & \textbf{61.6} & \underline{64.5} & \textbf{1855} & 86.2 & 69.3 & \underline{57.9} & \textbf{99.8}\\
\midrule
\multicolumn{8}{c}{\textit{$K=128$ ($\downarrow 77.8\%$)}}\\*\midrule
\methodvenue{SparseVLM}{ICML 2025} & 56.0 & 60.0 & 1696 & 80.5 & 67.1 & 54.9 & 93.1\\*
\methodvenue{V2Drop}{CVPR 2026} & 56.3 & 61.8 & 1712 & 80.9 & 68.8 & 53.8 & 94.0\\*
\methodvenue{PyramidDrop}{CVPR 2025} & 57.1 & 61.6 & 1761 & 82.6 & 68.4 & 56.6 & 95.6\\*
\methodvenue{ApET$^{\ddagger}$}{CVPR 2026} & 58.9 & 62.3 & 1801 & 86.1 & 68.7 & 53.9 & 96.6\\*
\methodvenue{E2S-Pruner}{arXiv 2026} & 57.8 & 63.5 & 1779 & 83.1 & \textbf{69.9} & 56.0 & 96.8\\*
\longmethodvenue{SinkPruner$^{\ddagger}$}{EMNLP Findings 2026} & 59.3 & 63.6 & 1806 & 85.3 & \underline{69.8} & 55.8 & 97.8\\*
\methodvenue{LRCP$^{\ddagger}$}{arXiv 2026} & 58.9 & 63.1 & \underline{1807} & \underline{87.0} & 68.5 & 56.8 & 97.9\\*
\methodvenue{VScan}{TMLR 2026} & 59.8 & 63.0 & 1792 & 86.1 & 68.9 & 57.3 & 98.0\\*
\methodvenue{ProViP$^{\ddagger}$}{arXiv 2026} & 60.2 & 63.9 & 1785 & 86.3 & 69.2 & 56.7 & 98.2\\*
\methodvenue{STS}{arXiv 2026} & 60.1 & 63.5 & 1803 & \textbf{87.2} & 68.8 & 57.6 & 98.6\\*
\methodvenue{VisionTrim$^{\ddagger}$}{ICLR 2026} & \underline{60.3} & \underline{64.0} & 1788 & 86.6 & 69.7 & \textbf{58.2} & \underline{98.9}\\*
\covistrow \textbf{\pro} & \textbf{61.3} & \textbf{64.7} & \textbf{1890} & 86.1 & 69.1 & \underline{57.7} & \textbf{100.0}\\
\midrule
\multicolumn{8}{c}{\textit{$K=64$ ($\downarrow 88.9\%$)}}\\*\midrule
\methodvenue{SparseVLM}{ICML 2025} & 52.7 & 56.2 & 1505 & 75.1 & 62.2 & 51.8 & 86.5\\*
\methodvenue{V2Drop}{CVPR 2026} & 50.5 & 55.2 & 1470 & 75.1 & 68.9 & 51.8 & 86.9\\*
\methodvenue{PyramidDrop}{CVPR 2025} & 47.5 & 58.8 & 1561 & 76.2 & 69.0 & 50.6 & 87.7\\*
\methodvenue{E2S-Pruner}{arXiv 2026} & 55.5 & 57.3 & 1656 & 78.5 & 69.4 & 49.5 & 90.6\\*
\methodvenue{DART+Random$^{\ddagger}$}{CVPR 2026} & 55.1 & 60.4 & 1670 & 74.1 & 68.4 & 53.4 & 91.4\\*
\methodvenue{ApET$^{\ddagger}$}{CVPR 2026} & 56.9 & 61.2 & 1714 & 84.4 & 68.9 & 53.0 & 94.5\\*
\longmethodvenue{DivPrune+Random$^{\ddagger}$}{CVPR 2026} & 57.8 & 61.3 & 1693 & \underline{86.8} & 68.3 & 54.5 & 95.3\\*
\methodvenue{LRCP$^{\ddagger}$}{arXiv 2026} & 57.1 & 60.9 & 1721 & 85.3 & 68.5 & 55.6 & 95.4\\*
\methodvenue{VScan}{TMLR 2026} & 58.3 & 62.1 & 1698 & 85.0 & 69.1 & 55.6 & 95.9\\*
\longmethodvenue{SinkPruner$^{\ddagger}$}{EMNLP Findings 2026} & 57.4 & 62.8 & 1754 & 83.8 & \underline{70.0} & 55.5 & 96.3\\*
\methodvenue{ProViP$^{\ddagger}$}{arXiv 2026} & 58.7 & 62.3 & 1723 & 85.8 & 69.4 & 54.9 & 96.3\\*
\methodvenue{STS}{arXiv 2026} & \underline{59.0} & 61.6 & 1718 & \textbf{87.0} & 69.2 & \underline{56.8} & 96.9\\*
\methodvenue{VisionTrim$^{\ddagger}$}{ICLR 2026} & 58.8 & \underline{63.0} & \underline{1780} & 86.2 & \textbf{71.0} & \underline{56.8} & \underline{98.0}\\*
\covistrow \textbf{\pro} & \textbf{60.9} & \textbf{64.4} & \textbf{1863} & 86.0 & 68.4 & \textbf{57.0} & \textbf{99.2}\\
\midrule
\multicolumn{8}{l}{\textbf{(b) Post-pruning counts: MME-PC and $R_{\mathrm{PC}}$}}\\*
\multicolumn{8}{c}{\textit{$K=192$ ($\downarrow 66.7\%$)}}\\*\midrule
\methodvenue{FastV}{ECCV 2024} & 52.7 & 61.2 & 1612 & 64.8 & 67.3 & 52.5 & 88.1\\*
\methodvenue{PriorTR}{ECCV 2026} & 60.4 & 63.6 & 1845 & 83.8 & 68.6 & 57.5 & 98.3\\*
\methodvenue{DART}{EMNLP 2025} & 60.0 & 63.6 & 1856 & 82.8 & 69.8 & 57.4 & 98.4\\*
\methodvenue{SIEVE$^{\ddagger}$}{arXiv 2026} & 60.6 & \textbf{64.6} & 1820 & 85.5 & 68.9 & 57.8 & 98.9\\*
\methodvenue{VLM-Pruner}{CVPR 2026} & \underline{61.3} & 63.5 & 1788 & \textbf{86.0} & \underline{70.0} & 58.0 & 98.9\\*
\methodvenue{CLSE}{ECCV 2026} & \textbf{61.5} & 63.6 & 1817 & 85.1 & \textbf{70.1} & 57.5 & 99.0\\*
\methodvenue{MoB}{NeurIPS 2025} & 61.2 & 63.8 & 1858 & 84.5 & \underline{70.0} & \underline{58.2} & 99.4\\*
\methodvenue{RoRA}{arXiv 2026} & 61.2 & \underline{64.5} & \underline{1861} & 85.5 & 69.1 & \textbf{58.7} & \underline{99.7}\\*
\methodvenue{D$^2$Pruner}{AAAI 2026} & 60.8 & 64.3 & \textbf{1879} & \underline{85.6} & \underline{70.0} & \underline{58.2} & \textbf{99.8}\\
\midrule
\multicolumn{8}{c}{\textit{$K=128$ ($\downarrow 77.8\%$)}}\\*\midrule
\methodvenue{FastV}{ECCV 2024} & 49.6 & 56.1 & 1490 & 59.6 & 60.2 & 50.6 & 81.6\\*
\methodvenue{DART}{EMNLP 2025} & 58.7 & 63.2 & 1840 & 80.1 & 69.1 & 56.4 & 96.8\\*
\methodvenue{PriorTR}{ECCV 2026} & 59.1 & 62.4 & 1820 & 81.7 & 68.6 & 56.9 & 96.9\\*
\methodvenue{VLM-Pruner}{CVPR 2026} & \textbf{61.0} & 62.5 & 1768 & 85.4 & \underline{69.3} & 57.3 & 98.0\\*
\methodvenue{MoB}{NeurIPS 2025} & 60.7 & 63.2 & \underline{1842} & 81.7 & \underline{69.3} & 57.5 & 98.0\\*
\methodvenue{CLSE}{ECCV 2026} & \underline{60.9} & 63.0 & 1816 & 84.4 & \textbf{70.1} & 56.6 & 98.3\\*
\methodvenue{SIEVE$^{\ddagger}$}{arXiv 2026} & 59.8 & \textbf{64.0} & 1815 & \underline{85.8} & 68.8 & 57.6 & 98.5\\*
\methodvenue{D$^2$Pruner}{AAAI 2026} & 60.3 & \underline{63.7} & \textbf{1850} & 85.1 & \underline{69.3} & \underline{58.0} & \underline{98.9}\\*
\methodvenue{RoRA}{arXiv 2026} & 60.0 & \underline{63.7} & \textbf{1850} & \textbf{86.1} & 69.2 & \textbf{58.4} & \textbf{99.1}\\
\midrule
\multicolumn{8}{c}{\textit{$K=64$ ($\downarrow 88.9\%$)}}\\*\midrule
\methodvenue{FastV}{ECCV 2024} & 46.1 & 48.0 & 1256 & 48.0 & 51.1 & 47.8 & 71.3\\*
\methodvenue{DART}{EMNLP 2025} & 55.9 & 60.6 & 1765 & 73.9 & \underline{69.8} & 54.4 & 93.1\\*
\methodvenue{PriorTR}{ECCV 2026} & 56.6 & 61.2 & 1745 & 76.3 & 68.8 & 54.9 & 93.6\\*
\methodvenue{CLSE}{ECCV 2026} & 58.3 & 61.9 & \underline{1809} & 78.3 & 69.6 & 55.5 & 95.6\\*
\methodvenue{VLM-Pruner}{CVPR 2026} & \textbf{59.2} & 61.3 & 1752 & 82.2 & 68.4 & 56.1 & 95.8\\*
\methodvenue{MoB}{NeurIPS 2025} & \underline{59.0} & 61.7 & 1806 & 77.2 & 69.7 & \textbf{57.0} & 96.0\\*
\methodvenue{SIEVE$^{\ddagger}$}{arXiv 2026} & 58.0 & \textbf{62.2} & 1740 & \underline{83.5} & 69.0 & \underline{56.5} & 96.1\\*
\methodvenue{RoRA}{arXiv 2026} & 57.5 & \underline{62.0} & 1773 & \textbf{84.2} & 69.0 & 56.3 & \underline{96.3}\\*
\methodvenue{D$^2$Pruner}{AAAI 2026} & 57.9 & 61.9 & \textbf{1823} & 82.4 & \textbf{70.0} & 56.1 & \textbf{96.7}\\
\midrule
\multicolumn{8}{l}{\textbf{(c) Layer schedules: MME-P and $R_{\mathrm P}$}}\\*
Vanilla, $K=576$ & 62.0 & 64.0 & 1511 & 85.8 & 69.5 & 58.2 & 100.0\\*\midrule
\multicolumn{8}{c}{\textit{$K=128$ ($\downarrow 77.8\%$)}}\\*\midrule
\methodvenue{TOPS$^{\ddagger}$}{arXiv 2026} & 60.5 & \underline{62.5} & \underline{1483} & \underline{86.8} & 68.2 & 57.0 & 98.3\\*
\methodvenue{STAR-Pro}{arXiv 2026} & \underline{60.8} & 62.4 & 1456 & \textbf{87.3} & \underline{68.9} & \underline{57.2} & \underline{98.4}\\*
\covistrow \textbf{\pro} & \textbf{61.3} & \textbf{64.7} & \textbf{1525} & 86.1 & \textbf{69.1} & \textbf{57.7} & \textbf{100.0}\\
\midrule
\multicolumn{8}{c}{\textit{$K=64$ ($\downarrow 88.9\%$)}}\\*\midrule
\methodvenue{TOPS$^{\ddagger}$}{arXiv 2026} & 58.7 & 60.9 & 1443 & \underline{86.5} & \underline{68.6} & 56.2 & 96.8\\*
\methodvenue{STAR-Pro}{arXiv 2026} & 59.4 & 61.0 & 1421 & \textbf{87.5} & \underline{68.6} & \underline{56.3} & 97.0\\*
\methodvenue{HiDrop$^{\mathrm{T}}$}{ICLR 2026} & \underline{60.5} & \underline{63.2} & \underline{1473} & 86.4 & \textbf{68.9} & 55.2 & \underline{97.8}\\*
\covistrow \textbf{\pro} & \textbf{60.9} & \textbf{64.4} & \textbf{1502} & 86.0 & 68.4 & \textbf{57.0} & \textbf{99.2}\\
\midrule
\multicolumn{8}{l}{\textbf{(d) Sample-average final counts: MME-PC and $R_{\mathrm{PC}}$}}\\*
\multicolumn{8}{c}{\textit{$K=128$ ($\downarrow 77.8\%$)}}\\*\midrule
\methodvenue{OccamToken$^{S}$}{arXiv 2026} & 60.9 & 63.9 & 1825 & 86.3 & 69.1 & 58.0 & 99.1\\
\midrule
\multicolumn{8}{c}{\textit{$K=64$ ($\downarrow 88.9\%$)}}\\*\midrule
\methodvenue{OccamToken$^{S}$}{arXiv 2026} & 59.3 & 63.2 & 1801 & 86.2 & 69.0 & 57.4 & 98.1\\
\end{longtable}\endgroup

\subsection{Compression location and budget semantics}

\paragraph{Fixed counts before prefill.}
The fixed comparison concerns the number of visual tokens that enter and traverse the decoder. Encoder-internal reduction is compatible with this definition. For example, FlowCut completes the LLaVA-1.5 compression inside the vision encoder. Consequently, its LLaVA-1.5 row belongs to the fixed group, although its implementations for other backbones also prune inside the LLM. Cen-Prune is evaluated with three separate base selectors.

\paragraph{Layer schedules and post-pruning counts.}
For a sequence of layer counts $K_\ell$, the mean is $L^{-1}\sum_\ell K_\ell$ and includes every layer before the first pruning event. For example, a method that processes 576 tokens for two layers and $K$ for the remaining 30 layers has a mean of $(2\cdot576+30K)/32$, \textit{i.e.,} 216, 156, and 96 for the post-pruning counts 192, 128, and 64, respectively. These means differ from the target means of \pro. E2S-Pruner explicitly accounts for the initial full-token layers: its four stage lengths are $3/14/5/10$ with counts $576/194/148/96$, $576/122/48/42$, and $576/20/8/0$, which yield exact means of 192, 128, and 64, respectively. STS uses $1.5K$ tokens for the first half of the decoder and $0.5K$ for the second half. STAR-Pro uses $12/8/12$ layers with counts $256/74/36$ or $128/37/18$, which yield exact means of 128 and 64, respectively. The reported schedules of SinkPruner give approximate means of $130.19$ and $64.88$ for the nominal 128 and 64 settings, and TOPS uses $2K\rightarrow K\rightarrow K/4$ over $12/12/8$ layers, giving a mean of $1.1875K$. VisionTrim keeps half of the 576 tokens for the first two decoder layers and the reported count thereafter, which gives the layer averages of 198, 138, and 78 in Table~\ref{tab:external}. Table~\ref{tab:external-pro-full} keeps its source labels and marks these rows. ProViP, ApET, LRCP, SIEVE, and the WTR combinations likewise retain their source-defined budgets.

\paragraph{Sample averages and training.}
PruneSID-Dyn and StepPrune vary the pre-decoder count across samples, and OccamToken reports the sample-average final counts after staged compression. Note that a sample average of final counts does not determine the mean across decoder layers. HiDrop includes training and uses its own reproduced full-model reference. Tables~\ref{tab:other-pc} and~\ref{tab:other-p} preserve these results under their original settings and identify the distinct training protocol.

\paragraph{Coverage beyond the six-task suite.}
DIVE \citep{dive2026} reports five of the six tasks. Therefore, its published aggregate is not inserted into $R_{\mathrm{PC}}$. HAP, TransPrune, and FitPrune \citep{hap2026,transprune2026,fitprune2025} provide results under FLOPs-based settings, which are outside the present numerical comparison. S$^2$Prune \citep{s2prune2026} and ET-Prune \citep{etprune2026} evaluate other backbone suites and are therefore not included.

\begingroup\fontsize{8.1}{9}\selectfont\setlength{\tabcolsep}{3pt}
\begin{longtable}{@{}lcccccccc@{}}
\caption{Additional reported settings with MME perception plus cognition. The budget column reproduces each source's setting. $^{\mathrm T}$ denotes training. Source codes are defined in Table~\ref{tab:source-key}.}\label{tab:other-pc}\\
\toprule Method & Budget & GQA & MMB & MME & POPE & SQA & VQA\tsup{T} & $R$ (\%)\\\midrule\endfirsthead
\multicolumn{9}{l}{\textit{Additional settings (continued)}}\\\toprule
Method & Budget & GQA & MMB & MME & POPE & SQA & VQA\tsup{T} & $R$ (\%)\\\midrule\endhead
\bottomrule\endfoot
\methodvenue{PruneSID-Dyn$^{S}$}{ICLR 2026}\textsuperscript{PS} & 192 & 60.2 & 63.8 & 1797 & 87.1 & 69.1 & 56.9 & 98.5\\
\methodvenue{PruneSID-Dyn$^{S}$}{ICLR 2026}\textsuperscript{PS} & 128 & 58.9 & 62.6 & 1760 & 86.9 & 68.8 & 55.1 & 96.9\\
\methodvenue{PruneSID-Dyn$^{S}$}{ICLR 2026}\textsuperscript{PS} & 64 & 57.2 & 59.7 & 1734 & 84.1 & 68.1 & 54.2 & 94.5\\
\methodvenue{OccamToken$^{S}$}{arXiv 2026}\textsuperscript{OC} & 128 & 60.9 & 63.9 & 1825 & 86.3 & 69.1 & 58.0 & 99.1\\
\methodvenue{OccamToken$^{S}$}{arXiv 2026}\textsuperscript{OC} & 64 & 59.3 & 63.2 & 1801 & 86.2 & 69.0 & 57.4 & 98.1\\
\methodvenue{OccamToken$^{S}$}{arXiv 2026}\textsuperscript{OC} & 32 & 59.1 & 62.9 & 1780 & 86.2 & 69.1 & 56.3 & 97.5\\
\methodvenue{StepPrune$^{S}$$^{\mathrm{T}}$}{arXiv 2026}\textsuperscript{SE} & 192 & 59.5 & 63.5 & 1773 & 85.7 & 69.9 & 56.5 & 97.8\\
\methodvenue{StepPrune$^{S}$$^{\mathrm{T}}$}{arXiv 2026}\textsuperscript{SE} & 128 & 57.3 & 63.2 & 1749 & 85.2 & 69.8 & 55.3 & 96.5\\
\methodvenue{StepPrune$^{S}$$^{\mathrm{T}}$}{arXiv 2026}\textsuperscript{SE} & 64 & 56.4 & 61.8 & 1698 & 83.1 & 69.3 & 54.7 & 94.7\\
\end{longtable}\endgroup

\begingroup\fontsize{8.1}{9}\selectfont\setlength{\tabcolsep}{3pt}
\begin{longtable}{@{}lcccccccc@{}}
\caption{Additional reported settings with MME perception only. The budget column reproduces each source's setting, and $^{\ddagger}$ marks a budget as labeled by its source. $^{\mathrm T}$ denotes training. Source codes are defined in Table~\ref{tab:source-key}.}\label{tab:other-p}\\
\toprule Method & Budget & GQA & MMB & MME & POPE & SQA & VQA\tsup{T} & $R$ (\%)\\\midrule\endfirsthead
\multicolumn{9}{l}{\textit{Additional settings (continued)}}\\\toprule
Method & Budget & GQA & MMB & MME & POPE & SQA & VQA\tsup{T} & $R$ (\%)\\\midrule\endhead
\bottomrule\endfoot
\methodvenue{TOPS$^{\ddagger}$}{arXiv 2026}\textsuperscript{TP} & 128 & 60.5 & 62.5 & 1483 & 86.8 & 68.2 & 57.0 & 98.3\\
\methodvenue{TOPS$^{\ddagger}$}{arXiv 2026}\textsuperscript{TP} & 64 & 58.7 & 60.9 & 1443 & 86.5 & 68.6 & 56.2 & 96.8\\
\methodvenue{TOPS$^{\ddagger}$}{arXiv 2026}\textsuperscript{TP} & 32 & 56.7 & 59.5 & 1385 & 83.5 & 68.8 & 54.9 & 94.3\\
\methodvenue{STAR-Pro}{arXiv 2026}\textsuperscript{SR} & 128 & 60.8 & 62.4 & 1456 & 87.3 & 68.9 & 57.2 & 98.4\\
\methodvenue{STAR-Pro}{arXiv 2026}\textsuperscript{SR} & 64 & 59.4 & 61.0 & 1421 & 87.5 & 68.6 & 56.3 & 97.0\\
\methodvenue{STAR-Pro}{arXiv 2026}\textsuperscript{SR} & 32 & 57.2 & 60.1 & 1358 & 86.6 & 68.7 & 54.6 & 94.8\\
\methodvenue{HiDrop$^{\mathrm{T}}$}{ICLR 2026}\textsuperscript{HI} & 80 & 61.3 & 63.7 & 1467 & 86.6 & 67.5 & 54.9 & 97.7\\
\methodvenue{HiDrop$^{\mathrm{T}}$}{ICLR 2026}\textsuperscript{HI} & 64 & 60.5 & 63.2 & 1473 & 86.4 & 68.9 & 55.2 & 97.8\\
\methodvenue{HiDrop$^{\mathrm{T}}$}{ICLR 2026}\textsuperscript{HI} & 48 & 59.8 & 63.7 & 1446 & 85.8 & 67.7 & 54.4 & 96.8\\
\midrule
Vanilla & 576 & 62.0 & 64.0 & 1511 & 85.8 & 69.5 & 58.2 & 100.0\\
\covistrow \pro & 192 & 61.6 & 64.5 & 1494 & 86.2 & 69.3 & 57.9 & 99.7\\
\covistrow \pro & 128 & 61.3 & 64.7 & 1525 & 86.1 & 69.1 & 57.7 & 100.0\\
\covistrow \pro & 64 & 60.9 & 64.4 & 1502 & 86.0 & 68.4 & 57.0 & 99.2\\
\end{longtable}\endgroup

\begin{longtable}{@{}lp{2.7in}cc@{}}
\caption{Sources for external task scores. Codes identify a specific report, including reproduced baseline rows. B1--B12 identify the complete six-task denominators in Table~\ref{tab:external-bases}.}\label{tab:source-key}\\
\toprule Code & Reporting paper & Table & Base\\\midrule\endfirsthead
\multicolumn{4}{l}{\textit{Sources (continued)}}\\\toprule
Code & Reporting paper & Table & Base\\\midrule\endhead
\bottomrule\endfoot
R & RESTORE \citep{cho2026restore} & 1 & B1\\
D & DeSAP \citep{ma2026desap} & 1 & B2\\
C & CRISP \citep{li2026crisp} & 1 & B2\\
V & VisionTrim \citep{visiontrim2026} & 1 & B2\\
T & CDPruner \citep{zheng2025cdpruner} & 1 & B3\\
S & SCoRe \citep{score2026} & 1 & B4\\
E & EADP \citep{wang2026eadp} & 1 & B5\\
A & AnchorPrune \citep{oh2026anchorprune} & 1 & B6\\
SC & SCOPE \citep{scope2025} & 1 & B2\\
SF & SpecFlow \citep{specflow2026} & 1 & B2\\
EV & EvoCut \citep{evocut2026} & 1 & B2\\
CA & CaRe \citep{li2026care} & 1 & B2\\
MT & MMTok \citep{mmtok2026} & 1 & B2\\
ZO & ZOO-Prune \citep{zooprune2026} & 1 & B2\\
FC & FlowCut \citep{flowcut2026} & 1 & B2\\
PS & PruneSID \citep{prunesid2026} & 1 & B2\\
SP & SPARE \citep{spare2026} & 1 & B3\\
CP & CoverPruner \citep{coverpruner2026} & 1 & B3\\
CN & Cen-Prune + CDPruner \citep{cenprune2026} & 2 & B7\\
E2 & E2S-Pruner \citep{e2s2026} & 1 & B1\\
V2 & V2Drop \citep{v2drop2026} & 1 & B1\\
ST & STS \citep{sts2026} & 1 & B2\\
PV & ProViP \citep{provip2026} & 1 & B2\\
AP & ApET \citep{apet2026} & 1 & B2\\
LR & LRCP \citep{lrcp2026} & 1 & B2\\
SK & SinkPruner \citep{sinkpruner2026} & 11 & B2\\
DA & DART \citep{dart2025} & 1 & B2\\
PT & PriorTR \citep{priortr2026} & 1 & B8\\
VP & VLM-Pruner \citep{vlmpruner2026} & 1 & B9\\
D2 & D$^2$Pruner \citep{d2pruner2026} & 1 & B2\\
MO & MoB \citep{mob2026} & 2 & B2\\
SI & SIEVE \citep{sieve2026} & 1 & B2\\
CL & CLSE \citep{clse2026} & 1 & B2\\
RR & RoRA \citep{rora2026} & 1 & B2\\
OC & OccamToken \citep{occamtoken2026} & 1 & B2\\
TP & TOPS \citep{tops2026} & 11 & B3\\
SR & STAR-Pro \citep{starpro2026} & 6 & B10\\
HI & HiDrop \citep{hidrop2026} & 1 & B11\\
SE & StepPrune \citep{stepprune2026} & 1 & B2\\
WR & WTR \citep{wtr2026} & 2 & B12\\
\end{longtable}
\begin{table}[t]\centering
\caption{Source-specific full-model scores. PC denotes perception plus cognition and P denotes perception only. Displayed denominators are rounded; retention uses the original reported precision.}
\label{tab:external-bases}\setlength{\tabcolsep}{4pt}
\begin{tabular}{@{}lccccccc@{}}\toprule
Base & MME type & GQA & MMB & MME & POPE & SQA & VQA\tsup{T}\\\midrule
B1 & PC & 61.9 & 64.6 & 1862 & 85.9 & 69.5 & 58.2\\
B2 & PC & 61.9 & 64.7 & 1862 & 85.9 & 69.5 & 58.2\\
B3 & P & 61.9 & 64.7 & 1507 & 85.9 & 69.5 & 58.2\\
B4 & P & 62.0 & 64.3 & 1511 & 85.9 & 66.8 & 58.2\\
B5 & P & 61.9 & 64.7 & 1513 & 85.9 & 69.6 & 58.2\\
B6 & P & 61.9 & 64.0 & 1509 & 85.8 & 69.6 & 58.3\\
B7 & P & 61.9 & 64.7 & 1511 & 85.9 & 69.6 & 58.2\\
B8 & PC & 61.9 & 64.6 & 1864 & 85.9 & 69.5 & 58.3\\
B9 & PC & 62.0 & 64.7 & 1862 & 85.9 & 69.5 & 58.2\\
B10 & P & 61.9 & 64.7 & 1508 & 85.9 & 69.5 & 58.2\\
B11 & P & 61.9 & 64.7 & 1507 & 86.8 & 69.5 & 58.2\\
B12 & PC & 62.0 & 64.6 & 1862 & 85.9 & 69.5 & 58.3\\
\bottomrule\end{tabular}\end{table}

\subsection{Cross-model sources and per-task comparisons}
\label{app:cross-model-sources}

\paragraph{LLaVA-1.5-13B.}
Table~\ref{tab:cross-13b} of the main text uses the results from FSR Table~5 \citep{tong2026fsr}, SCOPE Table~6 \citep{scope2025}, MMTok Table~17 \citep{mmtok2026}, ZOO-Prune Table~C \citep{zooprune2026}, PruneSID Table~9 \citep{prunesid2026}, and VisionZip Table~10 \citep{yang2025visionzip}, where SCOPE does not report VQAv2. DivPrune uses the reproduction in ZOO-Prune, and CDPruner uses the reproduction in FSR. The PruneSID record at 192 tokens duplicates the VisionZip scores of the same table and is therefore omitted. As in the other comparisons, each external row is normalized by the full model of its own source. Specifically, FSR reports GQA 63.3, MMBench 68.5, MME 1828, POPE 86.1, ScienceQA 72.8, and TextVQA 61.2, and the other five sources report 63.2, 67.7, 1818, 85.9, 72.8, and 61.3, respectively. The sources report VQAv2 on the test-dev split with a full-model score of $80.0$, whereas we evaluate the validation split with a full-model score of $78.3$. Our rows are therefore normalized by our own full model (Table~\ref{tab:transfer}). For \fixed, the six tasks keep the runs used throughout the paper and VQAv2 comes from a later run of the same configuration.

\paragraph{LLaVA-NeXT-7B.}
Table~\ref{tab:cross-next} of the main text uses sources that evaluate this backbone under the perception-plus-cognition convention and the OCR-augmented TextVQA prompt. The full-model TextVQA reference of $61.3$ in these sources matches our $61.1$ rather than our plain-prompt $64.6$. The sources are VisionTrim Table~2 (PyramidDrop, VisionZip, and VisionTrim), ZOO-Prune Table~3 (DivPrune, VisionZip, and ZOO-Prune), OccamToken Table~2 (DivPrune), ApET Table~2, FlowCut Table~2, and CaRe Table~3. The budgets of FlowCut and VisionTrim give the layer averages implied by their settings. FlowCut prunes to twice the reported count before the language model and to the reported count at its second layer. Meanwhile, VisionTrim, whose hyperparameters are shared across models, keeps half of the 2,880 tokens for the first two decoder layers and the reported count thereafter. When a method is reported by several sources, its own paper takes precedence. Otherwise, the third-party report with the highest retention is kept, which favors the compared method. Each row is normalized by the full model of its own source. The remaining panels of this appendix are organized by source instead. The MME-PC panel uses ZOO-Prune Table~3, including the VisionZip and DivPrune evaluations of that paper, and the MME-P panel uses AnchorPrune Table~2, including PruMerge+, VisionZip, DivPrune, and CDPruner. Each panel includes the source paper's own method, normalizes each external row by the full model of that source (Table~\ref{tab:cross-model-bases}), and normalizes \fixed\ by its own full model. The labels 640, 320, and 160 preserve the budgets reported by the sources; for our implementation they are content-token counts with 48 additional newline tokens.

\paragraph{Qwen2.5-VL-7B.}
Table~\ref{tab:cross-qwen} of the main text collects the sources that report MME, MMBench, POPE, and ScienceQA on this backbone, which is the entry requirement of that table. The sources are TOPS Table~2 (CDPruner, DivPrune, and TOPS), PruneSID Table~12 (PruneSID and VisionZip), and CaRe Table~S4 (CaRe and ZOO-Prune). MMBench-CN is reported in addition to the seven tasks of Table~\ref{tab:external}, and among these sources only PruneSID Table~12 reports VQAv2, with a full-model score of $82.9$ against our $81.6$ on the validation split. TextVQA follows the OCR-augmented prompt. The reason is that the only source reporting it gives a full-model reference of $77.7$, which is close to our $76.7$ under that prompt and far from our plain-prompt $82.4$. At the 512-token budget, PruneSID Table~12 is the only source that meets this requirement, with $33.3\%$ of the tokens kept. Table~\ref{tab:external-qwen} draws CDPruner and DivPrune from CDPruner Table~4, and HiPrune \citep{hiprune2026} and EADP from EADP Table~3, all of which report 512, 256, and 128 retained tokens. Then, it averages the three metrics jointly available in that panel, \textit{i.e.,} MMBench, MME-PC, and TextVQA, into $R_3$. Table~\ref{tab:external-qwen-subsets}(a) uses the GQA, MMBench, POPE, and ScienceQA results of CoIn Table~3 \citep{du2026coin}, and panel (b) uses the MMBench, MME-P, TextVQA, and MMMU results of AnchorPrune Table~3. Each panel defines its own four-task aggregate $R_4$ at the two overlapping budgets of 256 and 128. The published TextVQA values remain as reported. Our TextVQA results omit the OCR tokens in the three-task table and include them in the AnchorPrune panel, with both variants reported in Tables~\ref{tab:transfer} and~\ref{tab:transfer-extra}.

\paragraph{Eligibility and ordering.}
The panels of this appendix admit only pre-decoder fixed-count configurations. As a result, methods that reduce tokens inside the language model under a layer-average budget, such as PyramidDrop, ApET, and VisionTrim, appear in the main tables instead, where the Budget column or the table note states the semantics of each budget and, where it can be derived, the layer-average count it corresponds to. Methods that keep the full visual sequence in the first decoder layers and count only the tokens kept afterwards, such as FastV and DART, are not compared on these backbones. Moreover, the per-tile upper bounds of MMTok on NeXT are not treated as exact whole-image counts, the token keep ratios under dynamic-resolution Qwen inputs are not converted into the fixed counts of our $1008\times1008$ evaluation, and missing task scores are not imputed. Within each panel and budget, methods are ordered by the unrounded aggregate from the lowest to the highest. Bold and underline denote the highest and second-highest distinct displayed scores including ties, and the uncompressed references are excluded from this highlighting.

\begin{table}[htbp]
\centering
\caption{External full-model references for cross-model retention. Each row lists scores in the order stated by its panel header. The CoViST references appear in Tables~\ref{tab:transfer} and~\ref{tab:transfer-extra}. Source-specific ratios use the full available precision before these references are rounded for display.}
\label{tab:cross-model-bases}
\small
\setlength{\tabcolsep}{4pt}
\begin{tabular*}{\linewidth}{@{\extracolsep{\fill}}lcccccc@{}}
\toprule
\multicolumn{7}{l}{\textbf{LLaVA-NeXT-7B: six tasks}}\\
Source & GQA & MMB & MME & POPE & SQA & VQA\tsup{T}\\\midrule
ZOO-Prune (MME-PC) & 64.2 & 67.9 & 1842 & 86.4 & 70.2 & 61.3\\
AnchorPrune (MME-P) & 64.2 & 67.2 & 1529 & 86.4 & 70.2 & 61.3\\
\bottomrule
\end{tabular*}
\par\medskip
\begin{tabular*}{\linewidth}{@{\extracolsep{\fill}}lccc@{}}
\toprule
\multicolumn{4}{l}{\textbf{Qwen2.5-VL-7B: three-task comparison}}\\
Source & MMB & MME-PC & VQA\tsup{T}\\\midrule
CDPruner (C) & 82.8 & 2304 & 84.8\\
EADP (E) & 83.9 & 2303 & 84.5\\
\bottomrule
\end{tabular*}
\par\medskip
\begin{tabular*}{\linewidth}{@{\extracolsep{\fill}}lcccc@{}}
\toprule
\multicolumn{5}{l}{\textbf{Qwen2.5-VL-7B: four-task comparisons}}\\
Source & GQA & MMB & POPE & SQA\\\midrule
CoIn & 60.8 & 83.8 & 86.3 & 88.2\\\midrule
Source & MMB & MME-P & VQA\tsup{T} & MMMU\\\midrule
AnchorPrune & 83.1 & 1688 & 77.3 & 50.6\\
\bottomrule
\end{tabular*}
\label{tab:cross-model-bases-end}
\end{table}

\begin{table}[p]\centering
\caption{Pre-decoder comparisons on LLaVA-NeXT-7B. Panel (a) uses the MME-PC results reported by \citet{zooprune2026}. Panel (b) uses the MME-P results reported by \citet{oh2026anchorprune}, including AnchorPrune itself. Each panel averages its six displayed tasks using the full-model references of the corresponding source. CoViST uses its own references and OCR-augmented TextVQA. Published rows retain their source protocols. CoViST keeps $K+48$ visual tokens including newline markers. Ranking and highlighting follow Table~\ref{tab:external-fixed-full} and are independent within each panel and budget.}\label{tab:external-next}
\fontsize{7.9}{8.8}\selectfont\setlength{\budgetlabw}{15pt}
\setlength{\tabcolsep}{3pt}\renewcommand{\arraystretch}{1.0}
\begin{tabular*}{\linewidth}{@{}w{l}{\budgetlabw}@{\hspace{3pt}}l@{\extracolsep{\fill}\hspace{2\tabcolsep}}ccccccc@{}}
\toprule
\multicolumn{9}{l}{\textbf{(a) MME-PC and six-task retention}}\\
 & Method & GQA & MMB & MME-PC & POPE & SQA & VQA\tsup{T} & $R$ (\%)\\\midrule
 & Vanilla & 64.2 & 64.3 & 1817 & 86.9 & 68.0 & 61.1 & 100.0\\
 & Full (source) & 64.2 & 67.9 & 1842 & 86.4 & 70.2 & 61.3 & 100.0\\
\midrule
 & \methodvenue{DivPrune}{CVPR 2025} & 61.6 & \underline{65.4} & 1773 & 85.5 & 67.8 & 55.4 & 95.7\\
 & \methodvenue{ZOO-Prune}{CVPR 2026} & \underline{62.2} & 65.2 & \underline{1816} & \underline{86.8} & \underline{68.0} & 58.0 & 97.2\\
 & \methodvenue{VisionZip}{CVPR 2025} & 61.3 & \textbf{66.3} & 1787 & 86.3 & \textbf{68.1} & \underline{60.2} & \underline{97.5}\\
\covistrowskip\budgetlabelii{4}{$K{=}640$}{($\downarrow 77.8\%$)} & \textbf{\fixed} & \textbf{63.5} & 63.6 & \textbf{1865} & \textbf{87.4} & 67.9 & \textbf{60.3} & \textbf{99.9}\\
\midrule
 & \methodvenue{DivPrune}{CVPR 2025} & 59.6 & \underline{63.7} & 1731 & 83.5 & \textbf{67.8} & 53.8 & 93.6\\
 & \methodvenue{VisionZip}{CVPR 2025} & 59.3 & 63.1 & 1702 & 82.1 & \underline{67.3} & \textbf{58.9} & 94.1\\
 & \methodvenue{ZOO-Prune}{CVPR 2026} & \underline{61.0} & \textbf{64.9} & \underline{1788} & \underline{85.5} & \textbf{67.8} & 57.3 & \underline{96.1}\\
\covistrowskip\budgetlabelii{4}{$K{=}320$}{($\downarrow 88.9\%$)} & \textbf{\fixed} & \textbf{62.6} & 62.5 & \textbf{1812} & \textbf{87.3} & \textbf{67.8} & \underline{58.6} & \textbf{98.5}\\
\midrule
 & \methodvenue{VisionZip}{CVPR 2025} & 55.5 & 60.1 & 1630 & 74.8 & \underline{68.3} & \textbf{56.2} & 89.8\\
 & \methodvenue{DivPrune}{CVPR 2025} & 57.8 & \underline{62.3} & 1658 & 79.4 & 68.0 & 52.4 & 91.0\\
 & \methodvenue{ZOO-Prune}{CVPR 2026} & \underline{59.9} & \textbf{64.2} & \textbf{1739} & \underline{83.1} & \textbf{68.4} & 55.4 & \underline{94.4}\\
\covistrowskip\budgetlabelii{4}{$K{=}160$}{($\downarrow 94.4\%$)} & \textbf{\fixed} & \textbf{61.1} & 59.5 & \underline{1734} & \textbf{86.7} & 66.7 & \underline{55.8} & \textbf{95.4}\\
\bottomrule\end{tabular*}
\par\medskip
\begin{tabular*}{\linewidth}{@{}w{l}{\budgetlabw}@{\hspace{3pt}}l@{\extracolsep{\fill}\hspace{2\tabcolsep}}ccccccc@{}}
\toprule
\multicolumn{9}{l}{\textbf{(b) MME-P and six-task retention}}\\
 & Method & GQA & MMB & MME-P & POPE & SQA & VQA\tsup{T} & $R$ (\%)\\\midrule
 & Vanilla & 64.2 & 64.3 & 1496 & 86.9 & 68.0 & 61.1 & 100.0\\
 & Full (source) & 64.2 & 67.2 & 1529 & 86.4 & 70.2 & 61.3 & 100.0\\
\midrule
 & \methodvenue{PruMerge+}{ICCV 2025} & 55.4 & 62.8 & 1295 & 67.1 & \underline{70.0} & 51.0 & 87.5\\
 & \methodvenue{DivPrune}{CVPR 2025} & 58.1 & 62.8 & 1315 & 79.6 & 67.4 & 52.6 & 90.6\\
 & \methodvenue{VisionZip}{CVPR 2025} & 61.3 & 64.6 & 1462 & 86.0 & 68.2 & 60.1 & 97.0\\
 & \methodvenue{CDPruner}{NeurIPS 2025} & 62.5 & \underline{65.9} & 1491 & \textbf{87.4} & 68.2 & 58.8 & 97.9\\
\covistrowskip & \textbf{\fixed} & \underline{63.5} & 63.6 & \textbf{1531} & \textbf{87.4} & 67.9 & \underline{60.3} & \underline{99.9}\\
\budgetlabelii{6}{$K{=}640$}{($\downarrow 77.8\%$)} & \methodvenue{AnchorPrune}{ECCV 2026} & \textbf{64.3} & \textbf{67.1} & \underline{1529} & \underline{86.4} & \textbf{70.3} & \textbf{61.3} & \textbf{100.0}\\
\midrule
 & \methodvenue{PruMerge+}{ICCV 2025} & 53.7 & 61.6 & 1221 & 61.2 & \textbf{69.6} & 49.6 & 84.3\\
 & \methodvenue{DivPrune}{CVPR 2025} & 56.1 & 60.4 & 1278 & 74.4 & 67.3 & 50.8 & 87.6\\
 & \methodvenue{VisionZip}{CVPR 2025} & 58.9 & 63.1 & 1410 & 82.2 & 67.6 & \underline{58.9} & 94.2\\
 & \methodvenue{CDPruner}{NeurIPS 2025} & 61.3 & \underline{64.6} & \underline{1483} & \textbf{87.6} & 67.2 & 57.3 & 96.5\\
\covistrowskip & \textbf{\fixed} & \underline{62.6} & 62.5 & 1482 & \underline{87.3} & 67.8 & 58.6 & \underline{98.3}\\
\budgetlabelii{6}{$K{=}320$}{($\downarrow 88.9\%$)} & \methodvenue{AnchorPrune}{ECCV 2026} & \textbf{62.9} & \textbf{65.3} & \textbf{1508} & 87.2 & \underline{69.0} & \textbf{59.6} & \textbf{98.4}\\
\midrule
 & \methodvenue{PruMerge+}{ICCV 2025} & 51.8 & 58.0 & 1153 & 56.8 & \textbf{70.2} & 47.7 & 81.0\\
 & \methodvenue{DivPrune}{CVPR 2025} & 53.2 & 57.4 & 1212 & 67.9 & 67.3 & 48.1 & 83.4\\
 & \methodvenue{VisionZip}{CVPR 2025} & 55.3 & 58.8 & 1326 & 74.8 & 67.8 & \underline{56.0} & 89.2\\
 & \methodvenue{CDPruner}{NeurIPS 2025} & 60.8 & \underline{64.2} & 1431 & \underline{86.7} & 67.4 & 55.7 & 95.2\\
\covistrowskip & \textbf{\fixed} & \underline{61.1} & 59.5 & \underline{1449} & \underline{86.7} & 66.7 & 55.8 & \underline{95.7}\\
\budgetlabelii{6}{$K{=}160$}{($\downarrow 94.4\%$)} & \methodvenue{AnchorPrune}{ECCV 2026} & \textbf{62.6} & \textbf{65.5} & \textbf{1481} & \textbf{87.5} & \underline{68.7} & \textbf{60.1} & \textbf{98.2}\\
\bottomrule\end{tabular*}
\label{tab:external-next-end}
\end{table}

\begin{table}[!htbp]\centering
\caption{Fixed-budget comparisons on Qwen2.5-VL-7B at $512/256/128$ tokens. All three task scores are shown, and $R_3$ averages their source-normalized ratios. Superscripts C and E identify Tables 4 and 3 of \citet{zheng2025cdpruner} and \citet{wang2026eadp}, respectively. CoViST TextVQA omits OCR tokens from the prompt. This three-task aggregate is distinct from the six-task result in Table~\ref{tab:transfer}. The corresponding source references are reported in Appendix~\ref{app:cross-model-sources}. Ranking and highlighting follow Table~\ref{tab:external-fixed-full}.}\label{tab:external-qwen}
\fontsize{8.5}{10}\selectfont\setlength{\budgetlabw}{15pt}
\setlength{\tabcolsep}{3pt}\renewcommand{\arraystretch}{1.0}
\begin{tabular*}{\linewidth}{@{}w{l}{\budgetlabw}@{\hspace{3pt}}l@{\extracolsep{\fill}\hspace{2\tabcolsep}}cccc@{}}
\toprule
\multicolumn{6}{l}{\textbf{Qwen2.5-VL-7B: MMBench, MME-PC, and TextVQA}}\\
 & Method & MMB & MME-PC & VQA\tsup{T} & $R_3$ (\%)\\\midrule
 & Vanilla & 83.4 & 2320 & 82.4 & 100.0\\
\midrule
 & \methodvenue{HiPrune\textsuperscript{E}}{ACL Findings 2026} & 80.3 & 2177 & 75.8 & 93.3\\
 & \methodvenue{EADP\textsuperscript{E}}{ECCV 2026} & 81.6 & 2213 & 78.6 & 95.5\\
 & \methodvenue{DivPrune\textsuperscript{C}}{CVPR 2025} & 81.6 & 2279 & \underline{81.8} & 98.0\\
\covistrowskip & \textbf{\fixed} & \textbf{83.0} & \underline{2308} & 81.5 & \underline{99.3}\\
\budgetlabelii{5}{$K{=}512$}{($\downarrow 60.5\%$)} & \methodvenue{CDPruner\textsuperscript{C}}{NeurIPS 2025} & \underline{82.2} & \textbf{2327} & \textbf{84.2} & \textbf{99.9}\\
\midrule
 & \methodvenue{HiPrune\textsuperscript{E}}{ACL Findings 2026} & 78.4 & 2153 & 64.2 & 87.6\\
 & \methodvenue{EADP\textsuperscript{E}}{ECCV 2026} & 80.2 & 2202 & 73.8 & 92.8\\
 & \methodvenue{DivPrune\textsuperscript{C}}{CVPR 2025} & 80.0 & 2184 & 76.0 & 93.7\\
 & \methodvenue{CDPruner\textsuperscript{C}}{NeurIPS 2025} & \underline{80.9} & \underline{2245} & \textbf{82.4} & \underline{97.4}\\
\covistrowskip\budgetlabelii{5}{$K{=}256$}{($\downarrow 80.2\%$)} & \textbf{\fixed} & \textbf{82.2} & \textbf{2308} & \underline{80.5} & \textbf{98.6}\\
\midrule
 & \methodvenue{HiPrune\textsuperscript{E}}{ACL Findings 2026} & 75.0 & 2027 & 51.1 & 79.3\\
 & \methodvenue{EADP\textsuperscript{E}}{ECCV 2026} & \underline{78.4} & \underline{2138} & 65.8 & 88.0\\
 & \methodvenue{DivPrune\textsuperscript{C}}{CVPR 2025} & 77.8 & 2108 & 67.0 & 88.2\\
 & \methodvenue{CDPruner\textsuperscript{C}}{NeurIPS 2025} & 76.2 & 2127 & \textbf{77.8} & \underline{92.0}\\
\covistrowskip\budgetlabelii{5}{$K{=}128$}{($\downarrow 90.1\%$)} & \textbf{\fixed} & \textbf{80.3} & \textbf{2287} & \underline{74.7} & \textbf{95.2}\\
\bottomrule\end{tabular*}
\label{tab:external-qwen-end}
\end{table}

\begin{table}[p]\centering
\caption{Additional fixed-budget comparisons on Qwen2.5-VL-7B over two independently defined four-task subsets. Panel (a) uses Table 3 of \citet{du2026coin}. VisionZip is their reimplementation. Panel (b) uses Table 3 of \citet{oh2026anchorprune} and OCR-augmented TextVQA for CoViST. Every aggregate uses the four displayed task scores and source-specific full-model references. The two $R_4$ columns have different task sets and are ranked separately.}\label{tab:external-qwen-subsets}
\fontsize{8.5}{10}\selectfont
\setlength{\tabcolsep}{3pt}\renewcommand{\arraystretch}{1.0}
\begin{tabular*}{\linewidth}{@{\extracolsep{\fill}}lccccc@{}}
\toprule
\multicolumn{6}{l}{\textbf{(a) General visual reasoning}}\\
Method & GQA & MMB & POPE & SQA & $R_4$ (\%)\\\midrule
Vanilla & 59.8 & 83.4 & 86.6 & 88.1 & 100.0\\
Full (source) & 60.8 & 83.8 & 86.3 & 88.2 & 100.0\\
\midrule
\multicolumn{6}{c}{\textit{$K=256$ ($\downarrow 80.2\%$)}}\\\midrule
\methodvenue{VisionZip}{CVPR 2025} & 57.0 & 78.6 & 83.2 & 84.5 & 94.9\\
\methodvenue{CoIn}{CVPR 2026} & \underline{58.7} & \underline{79.4} & \textbf{85.1} & \underline{84.6} & \underline{96.5}\\
\covistrowfull \textbf{\fixed} & \textbf{58.9} & \textbf{82.2} & \underline{84.6} & \textbf{86.3} & \textbf{98.2}\\
\midrule
\multicolumn{6}{c}{\textit{$K=128$ ($\downarrow 90.1\%$)}}\\\midrule
\methodvenue{VisionZip}{CVPR 2025} & 52.3 & 74.5 & 78.6 & 82.7 & 89.9\\
\methodvenue{CoIn}{CVPR 2026} & \underline{56.8} & \underline{76.0} & \textbf{83.1} & \underline{83.0} & \underline{93.6}\\
\covistrowfull \textbf{\fixed} & \textbf{57.9} & \textbf{80.3} & \underline{82.2} & \textbf{84.1} & \textbf{95.9}\\
\bottomrule\end{tabular*}
\par\bigskip
\begin{tabular*}{\linewidth}{@{\extracolsep{\fill}}lccccc@{}}
\toprule
\multicolumn{6}{l}{\textbf{(b) Perception, text, and multidisciplinary reasoning}}\\
Method & MMB & MME-P & VQA\tsup{T} & MMMU & $R_4$ (\%)\\\midrule
Vanilla & 83.4 & 1671 & 76.7 & 48.9 & 100.0\\
Full (source) & 83.1 & 1688 & 77.3 & 50.6 & 100.0\\
\midrule
\multicolumn{6}{c}{\textit{$K=256$ ($\downarrow 80.2\%$)}}\\\midrule
\methodvenue{CDPruner}{NeurIPS 2025} & 78.0 & 1622 & 63.2 & 47.7 & 91.5\\
\methodvenue{DivPrune}{CVPR 2025} & 80.0 & \underline{1672} & 70.0 & \underline{48.6} & 95.5\\
\methodvenue{AnchorPrune}{ECCV 2026} & \underline{80.8} & 1658 & \underline{72.6} & \underline{48.6} & \underline{96.3}\\
\covistrowfull \textbf{\fixed} & \textbf{82.2} & \textbf{1687} & \textbf{75.7} & \textbf{48.8} & \textbf{99.5}\\
\midrule
\multicolumn{6}{c}{\textit{$K=128$ ($\downarrow 90.1\%$)}}\\\midrule
\methodvenue{CDPruner}{NeurIPS 2025} & 77.9 & 1572 & 57.7 & 45.9 & 88.1\\
\methodvenue{DivPrune}{CVPR 2025} & 76.9 & 1642 & 66.3 & \textbf{47.3} & 92.3\\
\methodvenue{AnchorPrune}{ECCV 2026} & \underline{79.6} & \textbf{1664} & \underline{68.4} & 46.8 & \underline{93.8}\\
\covistrowfull \textbf{\fixed} & \textbf{80.3} & \underline{1655} & \textbf{71.6} & \underline{47.2} & \textbf{96.3}\\
\bottomrule\end{tabular*}
\label{tab:external-qwen-subsets-end}
\end{table}

\section{Design comparison with related corrections}
\label{app:design}

\begin{table}[htbp]
\caption{Design comparison of attention-level corrections. Weight: the quantity attached to a representative in attention. Bound: the upper bound on that weight. Head: whether the correction differs across heads. Position: treatment of representative positions. Inherit: whether weights are carried from one reduction into the next. Content: treatment of representative features. Entries summarize the cited papers, and ``n.r.'' marks aspects those papers do not report.}
\label{tab:design}\centering\footnotesize
\setlength{\tabcolsep}{2pt}\renewcommand{\arraystretch}{1.1}
\begin{tabular}{@{}>{\raggedright\arraybackslash}p{1.5cm}>{\raggedright\arraybackslash}p{2.5cm}>{\raggedright\arraybackslash}p{1.0cm}>{\raggedright\arraybackslash}p{1.8cm}>{\raggedright\arraybackslash}p{2.4cm}>{\raggedright\arraybackslash}p{1.3cm}>{\raggedright\arraybackslash}p{2.6cm}@{}}\toprule
Method & Weight & Bound & Head & Position & Inherit & Content\\\midrule
ToMe & patch count & none & shared & n.a.\ (ViT) & within ViT & average\\
RESTORE & merge count & none & shared & original, distance factor & n.r. & base method\\
ERA & saliency-weighted count & none & shared & n.r. & no & recycled into anchors\\
CaRe & none & -- & -- & n.r. & no & confidence-gated\\
HiDrop & none & -- & -- & persistent IDs & no & none (trained)\\
\covistrow \method & evidence-conserved transfer & $m_{\max}$ & closed-form gains & original, distance factor & yes & confidence-gated, norm-preserving\\
\bottomrule\end{tabular}
\end{table}

Table~\ref{tab:design} summarizes how ToMe \citep{bolya2023tome}, RESTORE \citep{cho2026restore}, ERA \citep{wang2026era}, CaRe \citep{li2026care}, HiDrop \citep{hidrop2026}, and \method\ treat the quantities that a composable state records. The entries are taken from the cited papers and cover only the methods that modify attention or representative features.

\section{Runtime measurement protocol}
\label{app:runtime}

\begin{table}[t]
\caption{Inference efficiency on LLaVA-NeXT-7B with POPE and natural EOS stopping, measured on an A800 GPU with FP16 and SDPA. $K$ counts content tokens, exactly for \fixed\ and as a layer average for \pro. Mean visual tokens additionally include all 48 newline tokens. Latencies are means over 100 prompts with three repetitions, and peak allocated memory is the maximum over requests. Generation speedup is relative to Vanilla. FLOPs estimate decoder prefill matrix multiplications only.}
\label{tab:efficiency-next}
\centering
\begingroup
\fontsize{9}{10.5}\selectfont
\setlength{\tabcolsep}{3.5pt}
\renewcommand{\arraystretch}{1.08}
\begin{tabular*}{\linewidth}{@{\extracolsep{\fill}}lccccccc@{}}
\toprule
\multicolumn{8}{c}{\textbf{(a) Measured latency and output rate}}\\[2pt]
Method & $K$ & \shortstack[c]{Mean visual\\tokens} & \shortstack[c]{TTFT\\(ms)} & \shortstack[c]{LLM prefill\\(ms)} & \shortstack[c]{Generation\\(ms)} & \shortstack[c]{Generation\\speedup} & \shortstack[c]{Output\\(token/s)}\\
\midrule
Vanilla & All & 2928 & 263.94 & 245.66 & 289.64 & $1.00\times$ & 6.91\\
\midrule
\fixed & 640 & 688 & 178.96 & 150.34 & 207.66 & $1.39\times$ & 9.63\\
\fixed & 320 & 368 & 125.85 & 97.71 & 153.32 & $1.89\times$ & 13.04\\
\fixed & 160 & 208 & 102.72 & 74.41 & 130.25 & $2.22\times$ & 15.36\\
\midrule
\pro & 640 & 688 & 225.41 & 197.16 & 253.89 & $1.14\times$ & 7.88\\
\pro & 320 & 368 & 153.56 & 124.79 & 183.57 & $1.58\times$ & 10.90\\
\pro & 160 & 208 & 122.12 & 93.42 & 150.84 & $1.92\times$ & 13.26\\
\bottomrule
\end{tabular*}

\vspace{7pt}
\begin{tabular*}{\linewidth}{@{\extracolsep{\fill}}lcccc@{}}
\toprule
\multicolumn{5}{c}{\textbf{(b) Estimated computation and measured memory}}\\[2pt]
Method & $K$ & \shortstack[c]{Decoder prefill\\FLOPs (T)} & \shortstack[c]{Peak allocated\\memory (GiB)} & \shortstack[c]{Prefill KV\\cache (MiB)}\\
\midrule
Vanilla & All & 43.408 & 14.908 & 1494.82\\
\midrule
\fixed & 640 & 10.004 & 14.140 & 374.81\\
\fixed & 320 & 5.661 & 13.940 & 214.81\\
\fixed & 160 & 3.530 & 13.861 & 134.81\\
\midrule
\pro & 640 & 10.058 & 14.149 & 374.81\\
\pro & 320 & 5.675 & 13.950 & 214.81\\
\pro & 160 & 3.534 & 13.872 & 134.81\\
\bottomrule
\end{tabular*}
\endgroup
\label{tab:efficiency-next-end}
\end{table}

\paragraph{Workload and budget accounting.}
Table~\ref{tab:efficiency-next} reports the complete measurements of Section~\ref{sec:efficiency}. The uncompressed, fixed, and progressive configurations share the same inputs, software environment, attention implementation, and generation settings. The input resolution is $672\times672$, which yields 2,880 content tokens and 48 newline tokens before compression, and every newline token is retained. The reported visual count is averaged over decoder layers and requests and equals $K+48$ for the fixed counts and the matched progressive layer averages. For \pro, it is the layer average rather than the final-stage count.

\paragraph{Hardware and repetitions.}
All configurations run on one NVIDIA A800-SXM4-80GB GPU with FP16 parameters and SDPA. We first sample 100 POPE prompts with seed 20260916 and use 10 disjoint prompts for warm-up. Then, each configuration processes the same 100 evaluation prompts three times, which gives 300 timed requests over the same 100 examples. Generation is greedy with one beam, KV caching enabled, natural EOS stopping, and a maximum of 128 new tokens. Each repetition starts from the same running statistics and random seed, and the three repetitions of every prompt produce identical outputs.

\paragraph{Timing boundaries.}
The inputs are on the GPU before timing begins. The time to first token (TTFT) runs from this point until the first output token is produced. The LLM prefill runs from the entry into the language decoder until the first model forward pass completes, including the LM head, and the initial compression of LLaVA-NeXT takes place inside this interval. Both durations are measured with CUDA events. The total generation time is the synchronized wall-clock duration of the generation call, including the visual encoding, auxiliary instruction encoding, selection, state composition, attention modulation, and autoregressive generation, and excluding the model loading, image preprocessing, host-to-device transfer, text detokenization, and task scoring.

\paragraph{Aggregation, memory, and FLOPs.}
The latency values are means over the measured requests, and the generation speedup is the uncompressed mean generation time divided by the compressed mean. The output rate (token/s) divides the total number of generated tokens by the total generation time including prefill. The peak allocated memory is the maximum accelerator allocation recorded over all requests rather than the reserved allocator memory, and the prefill KV size counts the keys and values stored across decoder layers including the newline tokens. The prefill FLOPs are estimated from the actual sequence length of each layer, including the text and newline tokens, and count the query, key, value, and output projections, the three SwiGLU projections, and the attention matrix multiplications with two operations per multiply--add. They exclude the visual encoder, the multimodal projector, the LM head, softmax, normalization, and the additional compression operations.

\end{document}